\documentclass{article} % For LaTeX2e
\usepackage{iclr2021_conference,times}

\usepackage{amsmath,amsfonts,bm}

\def\eqref#1{equation~\ref{#1}}
\def\1{\bm{1}}

\DeclareMathAlphabet{\mathsfit}{\encodingdefault}{\sfdefault}{m}{sl}
\SetMathAlphabet{\mathsfit}{bold}{\encodingdefault}{\sfdefault}{bx}{n}

\newcommand{\E}{\mathbb{E}}

\newcommand{\Var}{\mathrm{Var}}

\usepackage{algorithm}
\usepackage{algorithmic}
\usepackage{hyperref}
\usepackage{url}
\usepackage{multicol}
\usepackage{multirow}
\usepackage{wrapfig}
\usepackage{graphicx}
\usepackage{subfigure}

\newcommand{\bx}{{\boldsymbol x}}
\newcommand{\bbeta}{{\boldsymbol \beta}}
\newcommand{\bxi}{{\boldsymbol \xi}}
\newcommand{\bW}{{\boldsymbol W}}

\def\cE{{\cal E}}
\usepackage{apptools}

\usepackage{xr}
\def\m{\mu}
\def\n{\nu}

\def\Si{\Sigma}

\def\cE{{\cal E}}

\def\cL{{\cal L}}

\def\hR{\mathbb{R}}

\def\cd{\cdot}

\def\bx{{\bf x}}

\usepackage{amssymb}

\newtheorem{theorem}{Theorem}

\newtheorem{lemma}{Lemma}
\newtheorem{remark}{Remark}

\newtheorem{assumption}{Assumption}

\edef\oldassumption{\the\numexpr\value{assumption}+1}

\usepackage{textcase}

\iclrfinalcopy

\title{Accelerating Convergence of Replica Exchange Stochastic Gradient MCMC via Variance Reduction}

\author{Wei Deng \thanks{Equal contribution}\\
Department of Mathematics\\
Purdue University\\
West Lafayette, IN, USA \\
\texttt{weideng056@gmail.com} \\

\And
Qi Feng \footnote[1]{Equal contribution}\\
Department of Mathematics \\
University of Southern California \\
Los Angeles, CA, USA \\
\texttt{qif@usc.edu} \\
\AND
Georgios Karagiannis \\
Department of Mathematical Sciences \\
Durham University \\
Durham, UK \\
\texttt{georgios.karagiannis@durham.ac.uk} \\
\And
Guang Lin \\
Departments of Mathematics \& \\
School of Mechanical Engineering \\
Purdue University \\
West Lafayette, IN, USA \\
\texttt{guanglin@purdue.edu} \\
\AND
Faming Liang \\
Departments of Statistics \\
Purdue University \\
West Lafayette, IN, USA \\
\texttt{fmliang@purdue.edu} \\
}

\begin{document}

\maketitle

\begin{abstract}
Replica exchange stochastic gradient Langevin dynamics (reSGLD) has shown promise in accelerating the convergence in non-convex learning; however, an excessively large correction for avoiding biases from noisy energy estimators has limited the potential of the acceleration. To address this issue, we study the variance reduction for noisy energy estimators, which promotes much more effective swaps. Theoretically, we provide a non-asymptotic analysis on the exponential convergence for the underlying continuous-time Markov jump process; moreover, we consider a generalized Girsanov theorem which includes the change of Poisson measure to overcome the crude discretization based on the Gr\"{o}nwall's inequality and yields a much tighter error in the 2-Wasserstein ($\mathcal{W}_2$) distance. Numerically, we conduct extensive experiments and obtain state-of-the-art results in optimization and uncertainty estimates for synthetic experiments and image data.
\end{abstract}

\section{Introduction}

Stochastic gradient Monte Carlo methods \citep{Welling11, Chen14, Li16} are the golden standard for Bayesian inference in deep learning due to their theoretical guarantees in uncertainty quantification \citep{VollmerZW2016, Chen15} and non-convex optimization \citep{Yuchen17}. However, despite their scalability with respect to the data size, their mixing rates are often extremely slow for complex deep neural networks with rugged energy landscapes \citep{landscape18}. To speed up the convergence, several techniques have been proposed in the literature in order to accelerate their exploration of multiple modes on the energy landscape, for example, dynamic temperatures \citep{nanxiang17} and cyclic learning rates \citep{ruqi2020}, % and dynamic importance sampling \citep{CSGLD}, 
to name a few. However, such strategies only explore contiguously a limited region around a few informative modes. Inspired by the successes of replica exchange, also known as parallel tempering, in traditional Monte Carlo methods \citep{PhysRevLett86, parallel_tempering05}, reSGLD \citep{deng2020} uses multiple processes based on stochastic gradient Langevin dynamics (SGLD) where interactions between different SGLD chains are conducted in a manner that encourages large jumps. In addition to the ideal utilization of parallel computation, the resulting process is able to jump to more informative modes for more robust uncertainty quantification. However, the noisy energy estimators in mini-batch settings lead to a large bias in the na\"{i}ve swaps, and a large correction is required to reduce the bias, which yields few effective swaps and insignificant accelerations. Therefore, how to reduce the variance of noisy energy estimators becomes essential in speeding up the convergence.

A long standing technique for variance reduction is the control variates method. The key to reducing the variance is to properly design correlated control variates so as to counteract some noise. Towards this direction, \citet{Dubey16, Xu18} proposed to update the control variate periodically for the stochastic gradient estimators and \citet{baker17} studied the construction of control variates using local modes. Despite the advantages in near-convex problems, a natural discrepancy between theory \citep{Niladri18, Xu18, SVRG_HMC_Zou} and practice \citep{kaiming15, bert} is \emph{whether we should avoid the gradient noise in non-convex problems}. To fill in the gap, we only focus on the variance reduction of noisy energy estimators to exploit the theoretical accelerations but no longer consider the variance reduction of the noisy gradients so that the empirical experience from stochastic gradient descents with momentum (M-SGD) can be naturally imported.

In this paper we propose the variance-reduced replica exchange stochastic gradient Langevin dynamics (VR-reSGLD) algorithm to accelerate convergence by reducing the variance of the noisy energy estimators. This algorithm not only \emph{shows the potential of exponential acceleration} via much more effective swaps in the non-asymptotic analysis but also \emph{demonstrates remarkable performance in practical tasks} where a limited time is required; while others \citep{Xu18, SAGA_LD_Zou} may only work well when the dynamics is sufficiently mixed and the discretization error becomes a major component. Moreover, the existing discretization error of the Langevin-based Markov jump processes \citep{chen2018accelerating, deng2020, Futoshi2020} is exponentially dependent on time due to the limitation of Gr\"{o}nwall's inequality. To avoid such a crude estimate, we consider the generalized Girsanov theorem and a change of Poisson measure. As a result, we obtain a much \emph{tighter discretization error only polynomially dependent on time}. Empirically, we test the algorithm through extensive experiments and achieve state-of-the-art performance in both optimization and uncertainty estimates.

\begin{figure*}[!ht]
  \centering
  \vskip -0.1in
  \subfigure[Gibbs measures at three temperatures $\tau$.]{\includegraphics[scale=0.17]{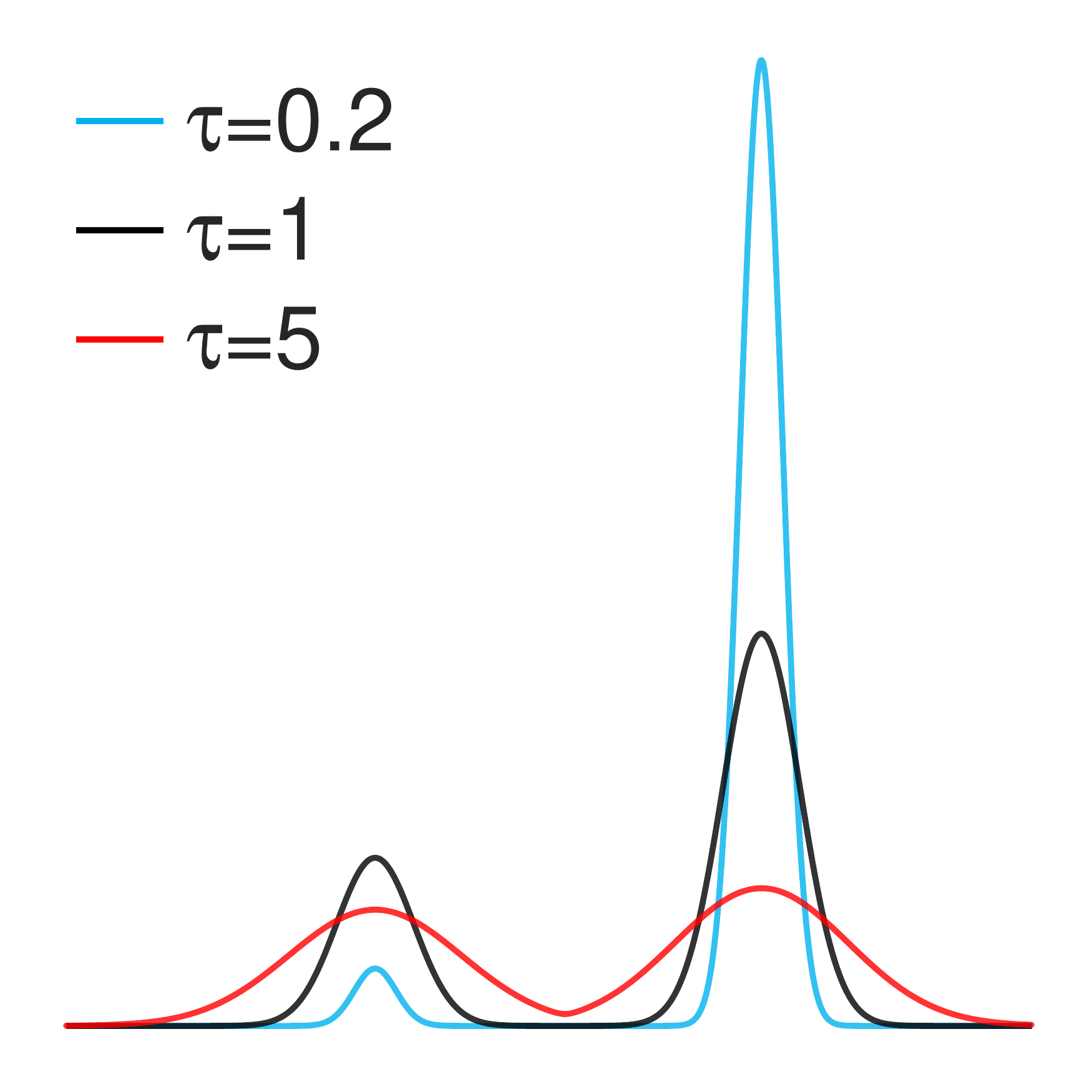}}\label{fig: 3a}\quad\quad
  \hspace{0.3cm}
  \subfigure[Sample trajectories on a energy landscape.]{\includegraphics[scale=0.12]{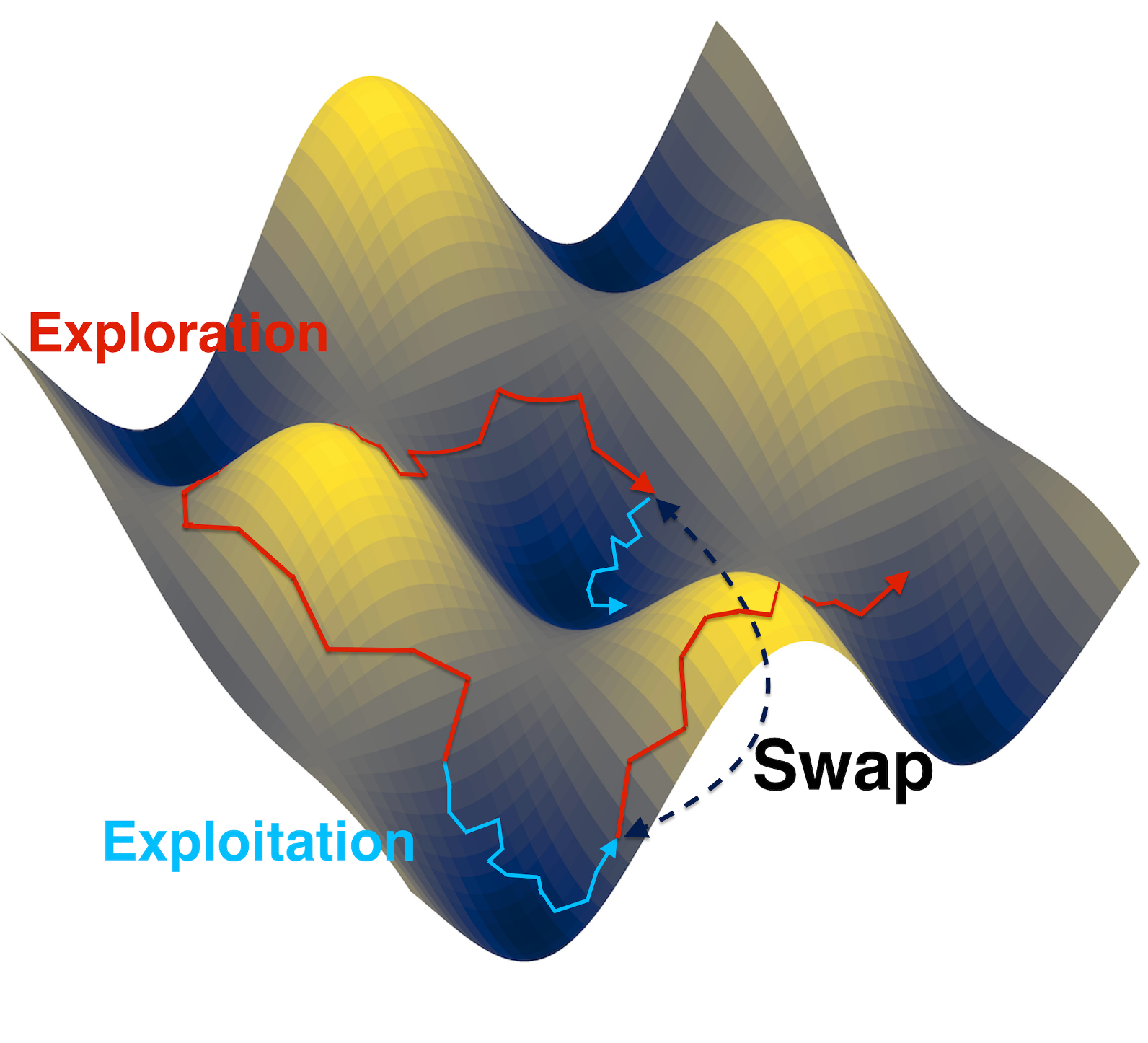}}\label{fig: 3b}\quad\quad
%   \hspace{0.3cm}
  \subfigure[Faster exponential convergence in $\mathcal{W}_2$]{\includegraphics[scale=0.17]{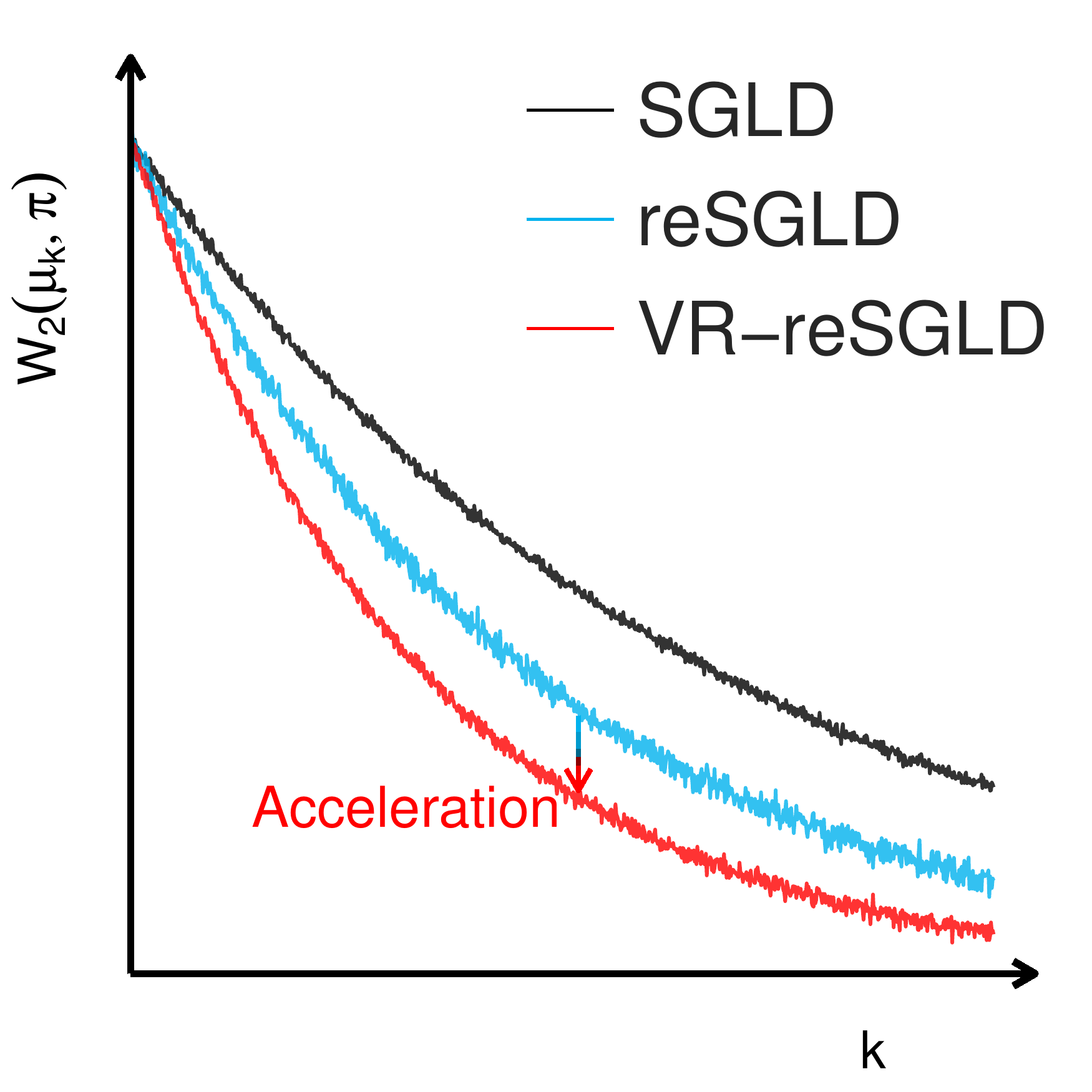}}
  \vspace{-0.5em}
  \caption{An illustration of replica exchange Monte Carlo algorithms for non-convex learning.}
  \label{demo_reld}
  \vspace{-1em}
%   \vskip -0.05in
\end{figure*}

\section{Preliminaries}
\label{reSGLD_section}

A common problem, in Bayesian inference,  is the simulation from a posterior $\mathrm{P}(\bbeta|\bm{X})\propto \mathrm{P}(\bbeta)\prod_{i=1}^N \mathrm{P}(\bx_i|\bbeta)$, where $\mathrm{P}(\bbeta)$ is a proper prior, $\prod_{i=1}^N \mathrm{P}(\bx_i|\bbeta)$ is the likelihood function and $N$ is the number of data points. %Let $\bm{X}=\{\bx_i\}_{i=1}^{N}$ be the entire data points and $\bbeta_k\in \mathbb{R}^d$ be the model parameter at iteration $k$.
When $N$ is large, the standard Langevin dynamics is too costly in evaluating the gradients. To tackle this issue, stochastic gradient Langevin dynamics (SGLD) \citep{Welling11} was proposed to make the algorithm scalable by approximating the gradient through a mini-batch data $B$ of size $n$ such that
\begin{equation}
     \bbeta_{k}=  \bbeta_{k-1}- \eta_k \frac{N}{n}\sum_{i\in B_{k}}\nabla  L( \bx_i|\bbeta_{k-1})+\sqrt{2\eta_k\tau} \bxi_{k},
\end{equation}
where $\bbeta_{k}\in\mathbb{R}^d$, $\tau$ denotes the temperature, $\eta_k$ is the learning rate at iteration $k$, $\bxi_{k}$ is a standard Gaussian vector, and $L(\cdot):=-\log \mathrm{P}(\bbeta|\bm{X})$ is the energy function. SGLD is known to converge weakly to a stationary Gibbs measure $\pi_{\tau}(\bbeta)\propto \exp\left(-L(\bbeta)/\tau\right)$ as $\eta_k$ decays to $0$ \citep{Teh16}. 

% Plots (a: Gibbs; b. paths; c. metastability; d; convergence in W2)

The temperature $\tau$ is the key to accelerating the computations in multi-modal distributions. On the one hand, a high temperature flattens the Gibbs distribution $\exp\left(-{L(\bbeta)}/{\tau}\right)$ (see the red curve in Fig.\ref{demo_reld}(a)) and accelerates mixing by facilitating exploration of the whole domain, but the resulting distribution becomes much less concentrated around the global optima. On the other hand, a low temperature exploits the local region rapidly; however, it may cause the particles to stick in a local region for an exponentially long time, as shown in the blue curve in Fig.\ref{demo_reld}(a,b). %Therefore, SGLD may be quite inefficient in sampling and optimization when the target distribution is highly multi-modal. 
To  bridge the gap between global exploration and local exploitation, \citet{deng2020} proposed the replica exchange SGLD algorithm (reSGLD), which consists of a low-temperature SGLD to encourage exploitation and a high-temperature SGLD to support exploration 
\begin{equation*}
\begin{split}
    \bbeta^{(1)}_{k} &=  \bbeta^{(1)}_{k-1}- \eta_k \frac{N}{n}\sum_{i\in B_{k}}\nabla  L( \bx_i|\bbeta^{(1)}_{k-1})+\sqrt{2\eta_k\tau^{(1)}} \bxi_{k}^{(1)} \\
    \bbeta^{(2)}_{k} &=  \bbeta^{(2)}_{k-1}- \eta_k \frac{N}{n}\sum_{i\in B_{k}}\nabla  L( \bx_i|\bbeta^{(2)}_{k-1})+\sqrt{2\eta_k\tau^{(2)}} \bxi_{k}^{(2)},
\end{split}
\end{equation*}
where the invariant measure is known to be $\pi(\bbeta^{(1)}, \bbeta^{(2)})\propto\exp\left(-\frac{L(\bbeta^{(1)})}{\tau^{(1)}}-\frac{L(\bbeta^{(2)})}{\tau^{(2)}}\right)$ as $\eta_k\rightarrow 0$ and $\tau^{(1)}< \tau^{(2)}$. Moreover, the two processes may swap the positions to allow tunneling between different modes. To avoid inducing a large bias in mini-batch settings, a corrected swapping rate $\widehat S$ is developed such that
\begin{equation*}
\label{vanilla_S}
\small
\widehat S=\exp\Big\{ \left(\frac{1}{\tau^{(1)}}-\frac{1}{\tau^{(2)}}\right)\Big(  \frac{N}{n}\sum_{i\in B_k} L(\bx_i|\bbeta^{(1)}_k)-  \frac{N}{n}\sum_{i\in B_k} L(\bx_i|\bbeta^{(2)}_k)-\frac{\left(\frac{1}{\tau^{(1)}}-\frac{1}{\tau^{(2)}}\right)\widehat \sigma^2}{F}\Big)\Big\},
\end{equation*}
where $\widehat \sigma^2$ is an estimator of the variance of $\frac{N}{n}\sum_{i\in B_k} L(\bx_i|\bbeta^{(1)}_k)-  \frac{N}{n}\sum_{i\in B_k} L(\bx_i|\bbeta^{(2)}_k)$ and $F$ is the correction factor to balance between acceleration and bias. In other words, the parameters switch the positions from $(\bbeta_k^{(1)}, \bbeta_k^{(2)})$ to $(\bbeta_k^{(2)}, \bbeta_k^{(1)})$ with a probability $r(1\wedge \widehat S)\eta_k$, where the constant $r$ is the swapping intensity and can set to $\frac{1}{\eta_k}$ for simplicity.

From a probabilistic point of view, reSGLD is a discretization scheme of replica exchange Langevin diffusion (reLD) in mini-batch settings. %\textcolor{red}{Given a test function $f(\cdot, \cdot)$} and the swapping rate $S(\cdot, \cdot)$,
Given a smooth test function $f$ and a swapping-rate function $S$, the infinitesimal generator $\cL_{S}$ associated with the continuous-time reLD follows
\begin{equation*}
\small
\begin{split}
    &\cL_{S}f(\bbeta^{(1)}, \bbeta^{(2)})=-\langle\nabla_{\bbeta^{(1)}}f(\bbeta^{(1)},\bbeta^{(2)}),\nabla L(\bbeta^{(1)})\rangle-\langle \nabla_{\bbeta^{(2)}}f(\bbeta^{(1)},\bbeta^{(2)}),\nabla L(\bbeta^{(2)})\rangle\\
    &\ \ \ +\tau^{(1)}\Delta_{\bbeta^{(1)}}f(\bbeta^{(1)},\bbeta^{(2)})+\tau^{(2)}\Delta_{\bbeta^{(2)}}f(\bbeta^{(1)},\bbeta^{(2)})+ rS(\bbeta^{(1)},\bbeta^{(2)})\cd (f(\bbeta^{(2)},\bbeta^{(1)})-f(\bbeta^{(1)},\bbeta^{(2)})),
\end{split}
\end{equation*}
where the last term arises from swaps and $\Delta_{\bbeta^{(\cdot)}}$ is the the Laplace operator with respect to $\bbeta^{(\cdot)}$. Note that the infinitesimal generator is closely related to Dirichlet forms in characterizing the evolution of a stochastic process. By standard calculations in Markov semigroups \citep{chen2018accelerating}, the Dirichlet form $\cE_{S}$ associated with the infinitesimal generator $\cL_{S}$ follows
\begin{equation}
\label{dirichlet_forms_main}
\small
\begin{split}
    \cE_{S}(f)=&\underbrace{\int \Big(\tau^{(1)}\|\nabla_{\bbeta^{(1)}}f(\bbeta^{(1)}, \bbeta^{(2)})\|^2+\tau^{(2)}\|\nabla_{\bbeta^{(2)}}f(\bbeta^{(1)}, \bbeta^{(2)})\|^2 \Big)d\pi(\bbeta^{(1)},\bbeta^{(2)})}_{\text{vanilla term } \cE(f)}\\
    &\ +\underbrace{\frac{r}{2}\int S(\bbeta^{(1)},\bbeta^{(2)})\cd (f(\bbeta^{(2)},\bbeta^{(1)})-f(\bbeta^{(1)},\bbeta^{(2)}))^2d\pi(\bbeta^{(1)},\bbeta^{(2)})}_{\text{acceleration term}},
\end{split}
\end{equation}
which leads to a strictly positive acceleration under mild conditions and is crucial for the exponentially accelerated convergence in the $\mathcal{W}_2$ distance (see Fig.\ref{demo_reld}(c)). However, the acceleration depends on the swapping-rate function $S$ and becomes much smaller given a noisy estimate of $\frac{N}{n}\sum_{i\in B} L(\bx_i|\bbeta)$ due to the demand of large corrections to reduce the bias.
 
\section{Variance Reduction in Replica Exchange Stochastic Gradient Langevin Dynamics}
\label{vrresgld}

The desire to obtain more effective swaps and larger accelerations drives us to design more efficient energy estimators. A na\"{i}ve idea would be to apply a large batch size $n$, which reduces the variance of the noisy energy estimator proportionally. However, this comes with a significantly increased memory overhead and computations and therefore is inappropriate for big data problems.

A natural idea to propose more effective swaps is to reduce the variance of the noisy energy estimator $L(B|\bbeta^{(h)})=\frac{N}{n}\sum_{i\in B}L(\bx_i|\bbeta^{(h)})$ for $h\in\{1,2\}$. Considering an unbiased estimator $L(B|\widehat\bbeta^{(h)})$ for $\sum_{i=1}^N L(\bx_i|\widehat\bbeta^{(h)})$ and a constant $c$, we see that a new estimator $\widetilde L(B| \bbeta^{(h)})$, which follows
\begin{equation}
    \widetilde L(B|\bbeta^{(h)})= L(B|\bbeta^{(h)}) +c\left( L(B|\widehat\bbeta^{(h)}) -\sum_{i=1}^N L (\bx_i| \widehat \bbeta^{(h)})\right),
\end{equation}
is still the unbiased estimator for $\sum_{i=1}^N L(\bx_i| \bbeta^{(h)})$. By decomposing the variance, we have
\begin{equation*}
\small
\begin{split}
    % \footnotesize
    &\Var(\widetilde L(B|\bbeta^{(h)}))=\Var\left( L(B|\bbeta^{(h)})\right)+c^2 \Var\left( L(B|\widehat\bbeta^{(h)})\right)+2c\text{Cov}\left( L(B|\bbeta^{(h)}),  L(B|\widehat\bbeta^{(h)})\right).
\end{split}
\end{equation*}
In such a case, $\Var(\widetilde L(B|\bbeta^{(h)}))$ achieves the minimum variance $(1-\rho^2)\Var(L(B|\bbeta^{(h)}))$ given $c^{\star}:=-\frac{\text{Cov}( L(B|\bbeta^{(h)}),  L(B|\widehat\bbeta^{(h)}))}{ \Var(L(B|\widehat \bbeta^{(h)}))}$, where $\text{Cov}(\cdot, \cdot)$ denotes the covariance and $\rho$ is the correlation coefficient of $ L(B|\bbeta^{(h)})$ and $ L(B|\widehat\bbeta^{(h)})$. To propose a correlated control variate, we follow \citet{SVRG} and update $\widehat \bbeta^{(h)}=\bbeta^{(h)}_{m\lfloor \frac{k}{m}\rfloor}$ every $m$ iterations. %Apparently, a smaller $m$ leads to a larger $|\bbeta|$. 
Moreover, the optimal $c^{\star}$ is often unknown in practice. To handle this issue, a well-known solution \citep{SVRG} is to fix $c=-1$ given a high correlation $|\rho|$ of the estimators and then we can present the VR-reSGLD algorithm in Algorithm \ref{alg}. Since the exact variance for correcting the stochastic swapping rate is unknown and even time-varying, we follow \citet{deng2020} and propose to use stochastic approximation \citep{Robbins51} to adaptively update the unknown variance.

%%%%% Conside version

\begin{algorithm}[tb]
  \caption{Variance-reduced replica exchange stochastic gradient Langevin dynamics (VR-reSGLD). The learning rate and temperature can be set to dynamic to speed up the computations. A larger smoothing factor $\gamma$ captures the trend better but becomes less robust. $\mathbb{T}$ is the thinning factor to avoid a cumbersome system.}
  \label{alg}
\begin{algorithmic}
% \footnotesize
% \STATE{\textbf{Input } Control variate $ \bbeta^{(h)}$, learning rate $\eta^{(h)}$ and temperature $\tau^{(h)}$ for $h\in\{1,2\}$, correction factor $F$.}
\STATE{\textbf{Input } The initial parameters $\bbeta_0^{(1)}$ and $\bbeta_0^{(2)}$, learning rate $\eta$, temperatures $\tau^{(1)}$ and $\tau^{(2)}$, correction factor $F$ and smoothing factor $\gamma$. }
\REPEAT
  \STATE{\textbf{Parallel sampling} \text{Randomly pick a mini-batch set $B_{k}$ of size $n$.}}
  \begin{equation}
      \bbeta^{(h)}_{k} =  \bbeta^{(h)}_{k-1}- \eta \frac{N}{n}\sum_{i\in B_{k}}\nabla  L( \bx_i|\bbeta^{(h)}_{k-1})+\sqrt{2\eta\tau^{(h)}} \bxi_{k}^{(h)}, \text{ for } h\in\{1,2\}.
  \end{equation}
%   \vskip -0.5in
  \STATE{\textbf{Variance-reduced energy estimators} \small{Update $\tiny{\widehat L^{(h)}=\sum_{i=1}^N L\left(\bx_i\Big| \bbeta^{(h)}_{m\lfloor \frac{k}{m}\rfloor}\right)}$ every $m$ iterations.}}
  \begin{equation}
  \label{__vr_loss}
      \widetilde L(B_{k}|\bbeta_{k}^{(h)})=\frac{N}{n}\sum_{i\in B_{k}}\left[ L(\bx_i| \bbeta_{k}^{(h)}) - L\left(\bx_i\Big| \bbeta^{(h)}_{m\lfloor \frac{k}{m}\rfloor}\right) \right]+\widehat L^{(h)} , \text{ for } h\in\{1,2\}. 
  \end{equation}
  \IF{$k\ \text{mod}\ m=0$} 
%   \STATE{Compute sample variance $\tilde \sigma_k^2$ based on $\widetilde L(\widetilde B_{i}|\bbeta_{k}^{(1)})$ and $\widetilde L(\widetilde B_{i}|\bbeta_{k}^{(2)})$, where $i\in\{1,\cdots, \tilde m\}$.}
%   \STATE{Obtain an estimator $\tilde \sigma_k^2$ for $\Var(\widetilde L( B_k|\bbeta_{k}^{(1)})-\widetilde L( B_{k}|\bbeta_{k}^{(2)}))$.}
  \STATE{\footnotesize{Update $\widetilde \sigma^2_{k} = (1-\gamma)\widetilde \sigma^2_{k-m}+\gamma \sigma^2_{k}$, where $\sigma_k^2$ is an estimate for $\Var\left(\widetilde L( B_k|\bbeta_{k}^{(1)})-\widetilde L( B_{k}|\bbeta_{k}^{(2)})\right)$.}}
  \ENDIF
   
%   \STATE{Obtain an unbiased estimate $\tilde \sigma_{m+1}^2$ for $\sigma^2$.}
   
  \STATE{\textbf{Bias-reduced swaps} \small{Swap $ \bbeta_{k+1}^{(1)}$ and $ \bbeta_{k+1}^{(2)}$ if $u<\widetilde S_{\eta,m,n}$, where $u\sim \text{Unif }[0,1]$, and $\widetilde S_{\eta,m,n}$ follows}}
  \begin{equation}
  \label{stochastic_jump_rate}
  \small
      \textstyle \widetilde S_{\eta,m,n}=\exp\left\{ \left(\frac{1}{\tau^{(1)}}-\frac{1}{\tau^{(2)}}\right)\left(  \widetilde L( B_{k+1}|\bbeta_{k+1}^{(1)})-  \widetilde L( B_{k+1}|\bbeta_{k+1}^{(2)})-\frac{1}{F}\left(\frac{1}{\tau^{(1)}}-\frac{1}{\tau^{(2)}}\right)\widetilde \sigma^2_{m\lfloor \frac{k}{m}\rfloor}\right)\right\}.
  \end{equation}
%   \IF{$u<\widehat S$}
%   \STATE Swap $ \bbeta_{k+1}^{(1)}$ and $ \bbeta_{k+1}^{(2)}$.
%   \ENDIF
  \UNTIL{$k=k_{\max}$.}
\STATE{\textbf{Output:}  The low-temperature process $\{\bbeta_{i\mathbb{T}}^{(1)}\}_{i=1}^{\lfloor k_{\max}/\mathbb{T}\rfloor}$, where $\mathbb{T}$ is the thinning factor.}
% \vskip -.1 inch
% \vspace{-0.3em}
\end{algorithmic}
% \vspace{-0.3em}
\end{algorithm}

\paragraph{Variants of VR-reSGLD} The number of iterations $m$ to update the control variate $\widehat \bbeta^{(h)}$ gives rise to a trade-off in computations and variance reduction. A small $m$ introduces a highly correlated control variate at the cost of expensive computations; a large $m$, however, may yield a less correlated control variate and setting $c=-1$ fails to reduce the variance. In spirit of the adaptive variance in \citet{deng2020} to estimate the unknown variance, we explore the idea of the adaptive coefficient $\widetilde {c}_k=(1-\gamma_k)\widetilde {c}_{k-m}+\gamma_k c_k$ such that the unknown optimal $c^{\star}$ is well approximated. We present the adaptive VR-reSGLD in Algorithm \ref{adaptive_alg} in Appendix \ref{adaptive_c} and show empirically later that the adaptive VR-reSGLD leads to a significant improvement over VR-reSGLD for the less correlated estimators.

A parallel line of research is to exploit the SAGA algorithm \citep{SAGA} in the study of variance reduction. Despite the most effective performance in variance reduction \citep{Niladri18}, the SAGA type of sampling algorithms require an excessively memory storage of $\mathcal{O}(Nd)$, which is too costly for big data problems. Therefore, we leave the study of the lightweight SAGA algorithm inspired by \citet{pracSVRG, Dongruo} for future works.

\paragraph{Related work} Although our VR-reSGLD is, in spirit, similar to VR-SGLD \citep{Dubey16, Xu18}, it differs from VR-SGLD in two aspects: First, VR-SGLD conducts variance reduction on the gradient and only shows promises in the nearly log-concave distributions or when the Markov process is sufficiently converged; however, our VR-reSGLD solely focuses on the variance reduction of the energy estimator to propose more effective swaps, and therefore we can import the empirical experience in hyper-parameter tuning from M-SGD to our proposed algorithm. Second, VR-SGLD doesn't accelerate the continuous-time Markov process but only focuses on reducing the discretization error; VR-reSGLD possesses a larger acceleration term in the Dirichlet form (\ref{dirichlet_forms_main}) and shows a potential in exponentially speeding up the convergence of the continuous-time process in the early stage, in addition to the improvement on the discretization error. In other words, our algorithm is not only theoretically sound but also more empirically appealing for a wide variety of problems in non-convex learning.

\section{Theoretical properties}

The large variance of noisy energy estimators directly limits the potential of the acceleration and significantly slows down the convergence compared to the replica exchange Langevin dynamics. As a result, VR-reSGLD may lead to a more efficient energy estimator with a much smaller variance.

\begin{lemma}[Variance-reduced energy estimator]
\label{vr-estimator_main}
Under the smoothness
and dissipativity assumptions \ref{assump: lip and alpha beta} and \ref{assump: dissipitive} in Appendix \ref{prelim}, the variance of the variance-reduced energy estimator $\widetilde L(B|\bbeta^{(h)})$, where $h\in\{1,2\}$, is upper bounded by
\begin{equation*}
    \Var\left(\widetilde L(B|\bbeta^{(h)})\right)\leq \min\Big\{ \mathcal{O}\left(\frac{m^2 \eta}{n}\right), \Var\Big(\frac{N}{n}\sum_{i\in B} L(\bx_i| \bbeta^{(h)})\Big)+\Var\Big(\frac{N}{n}\sum_{i\in B} L(\bx_i| \widehat\bbeta^{(h)})\Big)\Big\},
\end{equation*}
where the detailed $\mathcal{O}(\cdot)$ constants is shown in Lemma \ref{vr-estimator} in the appendix.
\end{lemma}

The analysis shows the variance-reduced estimator $\widetilde L(B|\bbeta^{(h)})$ yields a much-reduced variance given a smaller learning rate $\eta$ and a smaller $m$ for updating control variates based on the batch size $n$. Although the truncated swapping rate $S_{\eta, m, n}=\min\{1, \widetilde S_{\eta, m, n}\}$ still satisfies the ``stochastic'' detailed balance given an unbiased swapping-rate estimator $\widetilde S_{\eta, m, n}$ \citep{deng2020} \footnote[2]{\citet{Pseudo-Marginal-method, Matias19} achieve a similar result based on the unbiased likelihood estimator for the Metropolis-hasting algorithm. See section 3.1 \citep{Matias19} for details.}, it doesn't mean the efficiency of the swaps is not affected. By contrast, we can show that the number of swaps may become \emph{exponentially smaller on average}.

\begin{lemma}[Variance reduction for larger swapping rates] \label{exp_S_main_body} Given a large enough batch size $n$, the variance-reduced energy estimator $\widetilde L(B_{k}|\bbeta_{k}^{(h)})$ yields a truncated swapping rate that satisfies
\begin{equation}
\label{larger_SSS}
     \E[S_{\eta, m, n}]\approx \min\Big\{1, S(\bbeta^{(1)}, \bbeta^{(2)})\Big(\mathcal{O}\Big(\frac{1}{n^2}\Big)+e^{-\mathcal{O}\left(\frac{m^2\eta}{n}+\frac{1}{n^2}\right)}\Big)\Big\},
\end{equation}
\end{lemma}
where $S(\bbeta^{(1)}, \bbeta^{(2)})$ is the deterministic swapping rate defined in Appendix \ref{exp_acc}. The proof is shown in Lemma.\ref{exp_S} in Appendix \ref{exp_acc}. Note that the above lemma doesn't require the normality assumption. As $n$ goes to infinity, where the asymptotic normality holds, the RHS of (\ref{larger_SSS}) changes to $\min\Big\{1, S(\bbeta^{(1)}, \bbeta^{(2)})e^{-\mathcal{O}\left(\frac{m^2\eta}{n}\right)}\Big\}$, which becomes exponentially larger as we use a smaller update frequency $m$ and learning rate $\eta$. Since the continuous-time reLD induces a jump operator in the infinitesimal generator, the resulting Dirichlet form potentially leads to a 
much larger acceleration term which linearly depends on the swapping rate $S_{\eta, m, n}$ and yields a faster exponential convergence. Now we are ready to present the first main result.

\begin{theorem}[Exponential convergence]\label{exponential decay_main_body}
Under the smoothness
and dissipativity assumptions \ref{assump: lip and alpha beta} and \ref{assump: dissipitive}, the probability measure associated with reLD at time $t$, denoted as $\nu_t$, converges exponentially fast to the invariant measure $\pi$:
\begin{equation}
    \mathcal{W}_2(\nu_t,\pi) \leq  D_0 \exp\left\{-t\left(1+\delta_{ S_{\eta, m, n}}\right)/c_{\text{LS}}\right\},
\end{equation}
where $D_0$ is a constant depending on the initialization, $\delta_{ S_{\eta, m, n}}:=\inf_{t>0}\frac{\cE_{ S_{\eta, m, n}}(\sqrt{\frac{d\n_t}{d\pi}})}{\cE(\sqrt{\frac{d\n_t}{d\pi}})}-1\geq 0$ depends on $S_{\eta, m, n}$, $\cE_{ S_{\eta, m, n}}$ and $\cE$ are the Dirichlet forms based on the swapping rate $S_{\eta, m, n}$ and are defined in (\ref{dirichlet_forms_main}), $c_{\text{LS}}$ is the constant of the log-Sobolev inequality for reLD without swaps.
\end{theorem}{}

We detail the proof in Theorem.\ref{exponential decay} in Appendix \ref{exp_acc}. Note that $S_{\eta,m,n}=0$ leads to the same performance as the standard Langevin diffusion and $\delta_{ S_{\eta, m, n}}$ is strictly positive when $\frac{d\n_t}{d\pi}$ is asymmetric \citep{chen2018accelerating}; given a smaller $\eta$ and $m$ or a large $n$, the variance becomes much reduced according to Lemma \ref{vr-estimator_main}, yielding a much larger truncated swapping rate by Lemma \ref{exp_S_main_body} and a faster exponential convergence to the invariant measure $\pi$ compared to reSGLD.

Next, we estimate the upper bound of the 2-Wasserstein distance $\mathcal{W}(\mu_{k},\nu_{k\eta})$, where $\mu_{k}$ denotes the probability measure associated with VR-reSGLD at iteration $k$. 
We first bypass the Gr\"{o}nwall inequality and conduct the change of measure to upper bound the relative entropy $D_{KL}(\mu_{k}|\nu_{k\eta})$ following \citep{Maxim17}. In addition to the approximation in the standard Langevin diffusion \citet{Maxim17}, we also consider the change of Poisson measure following \citet{yin_zhu_10, Gikhman} to handle the error from the stochastic swapping rate. We then extend the distance of relative entropy $D_{KL}(\mu_{k}|\nu_{k\eta})$ to the Wasserstein distance $\mathcal{W}_2(\mu_{k},\nu_{k\eta})$ via a weighted transportation-cost inequality of \citet{bolley05}.

\begin{theorem}[Diffusion approximation]
\label{numerical_error_W2_main}
Assume the smoothness, the dissipativity and the gradient assumptions \ref{assump: lip and alpha beta}, \ref{assump: dissipitive} and \ref{assump: stochastic_noise} hold. Given a large enough batch size $n$, a small enough $m$ and $\eta$, we have
\begin{equation}
\label{second_last_W2_main}
\begin{split}
    \mathcal{W}_2(\mu_{k},\nu_{k\eta})&\leq  \mathcal{O}\Big(d k^{3/2}\eta \Big(\eta^{1/4}+\delta^{1/4}+\Big(\frac{m^2}{n}\eta\Big)^{1/8}\Big)\Big),\\
\end{split}
\end{equation}
\end{theorem}
where $\delta$ is a constant that characterizes the scale of noise caused in mini-batch settings and the detail is given in Theorem \ref{numerical_error_W2} in Appendix \ref{discre_error} . Here the last term $\mathcal{O}\big(\big(\frac{m^2}{n}\eta\big)^{1/8}\big)$ comes from the error induced by the stochastic swapping rate, which disappears given a large enough batch size $n$ or a small enough update frequency $m$ and learning rate $\eta$. Note that our upper bound is linearly dependent on time approximately, which is much tighter than the exponential dependence using the Gr\"{o}nwall inequality. Admittedly, the result without swaps is slightly weaker than the diffusion approximation (3.1) in \citet{Maxim17} and we refer readers to Remark \ref{argument_raginsky} in Appendix \ref{discre_error}.

% since the Wiener measure under a new probability measure is not treated as a Brownian motion freely.

Applying the triangle inequality for $\mathcal{W}_2(\mu_{k},\nu_{k\eta})$ and $\mathcal{W}_2(\nu_{k\eta},\pi)$ leads to the final result
\begin{theorem} Assume the smoothness, the dissipativity and the gradient assumptions \ref{assump: lip and alpha beta}, \ref{assump: dissipitive} and \ref{assump: stochastic_noise} hold. Given a small enough learning rate $\eta$, update frequency $m$ and a large enough batch size $n$, we have 
\begin{equation*}
\label{last_W2}
\begin{split}
    \mathcal{W}_2(\mu_{k},\pi)&\leq  \mathcal{O}\Big(d k^{3/2}\eta \Big(\eta^{1/4}+\delta^{1/4}+\Big(\frac{m^2}{n}\eta\Big)^{1/8}\Big)\Big)+\mathcal{O}\Big(e^{\frac{-k\eta(1+\delta_{ S_{\eta, m, n}})}{c_{\text{LS}}}}\Big).\\
\end{split}
\end{equation*}
\end{theorem}
This theorem implies that increasing the batch size $n$ or decreasing the update frequency $m$ not only reduces the numerical error but also potentially leads to a faster exponential convergence of the continuous-time dynamics via a much larger swapping rate $S_{\eta, m, n}$.

\section{Experiments}

% \newpage

\subsection{Simulations of Gaussian Mixture Distributions \label{subsec:Toy-multi-mo}}

We first study the proposed variance-reduced replica exchange stochastic gradient Langevin dynamics algorithm (VR-reSGLD) on a Gaussian mixture distribution \citep{Dubey16}. The distribution follows from $x_{i}|\beta\sim0.5\text{N}(\beta,\sigma^{2})+0.5\text{N}(\phi-\beta,\sigma^{2})$,
where $\phi=20$, $\sigma=5$ and $\beta=-5$. We use a training dataset of size $N=10^{5}$ and propose to estimate the posterior distribution over $\beta$.  
We compare the performance of VR-reSGLD against that of the standard stochastic gradient Langevin dynamics (SGLD), and replica exchange SGLD (reSGLD).

\begin{figure}
\center
\subfigure[Trace plot for $\bbeta^{(1)}$\label{fig:Ex1_recovery_svrg_csgld_theta_1}]{\includegraphics[width=0.24\columnwidth]{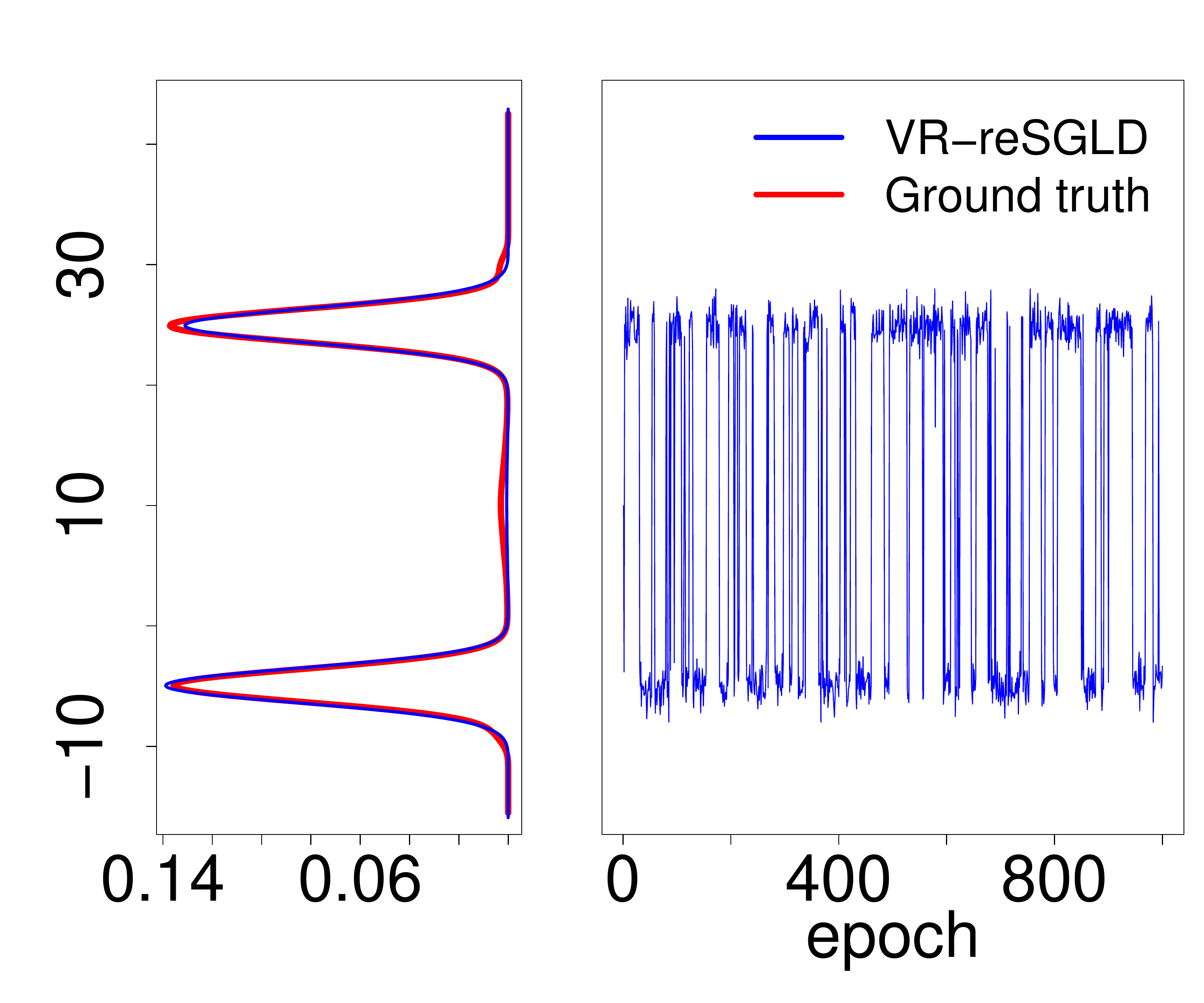}}
\enskip
\subfigure[Trace plot for $\bbeta^{(1)}$\label{fig:Ex1_recovery}]{\includegraphics[width=0.24\columnwidth]{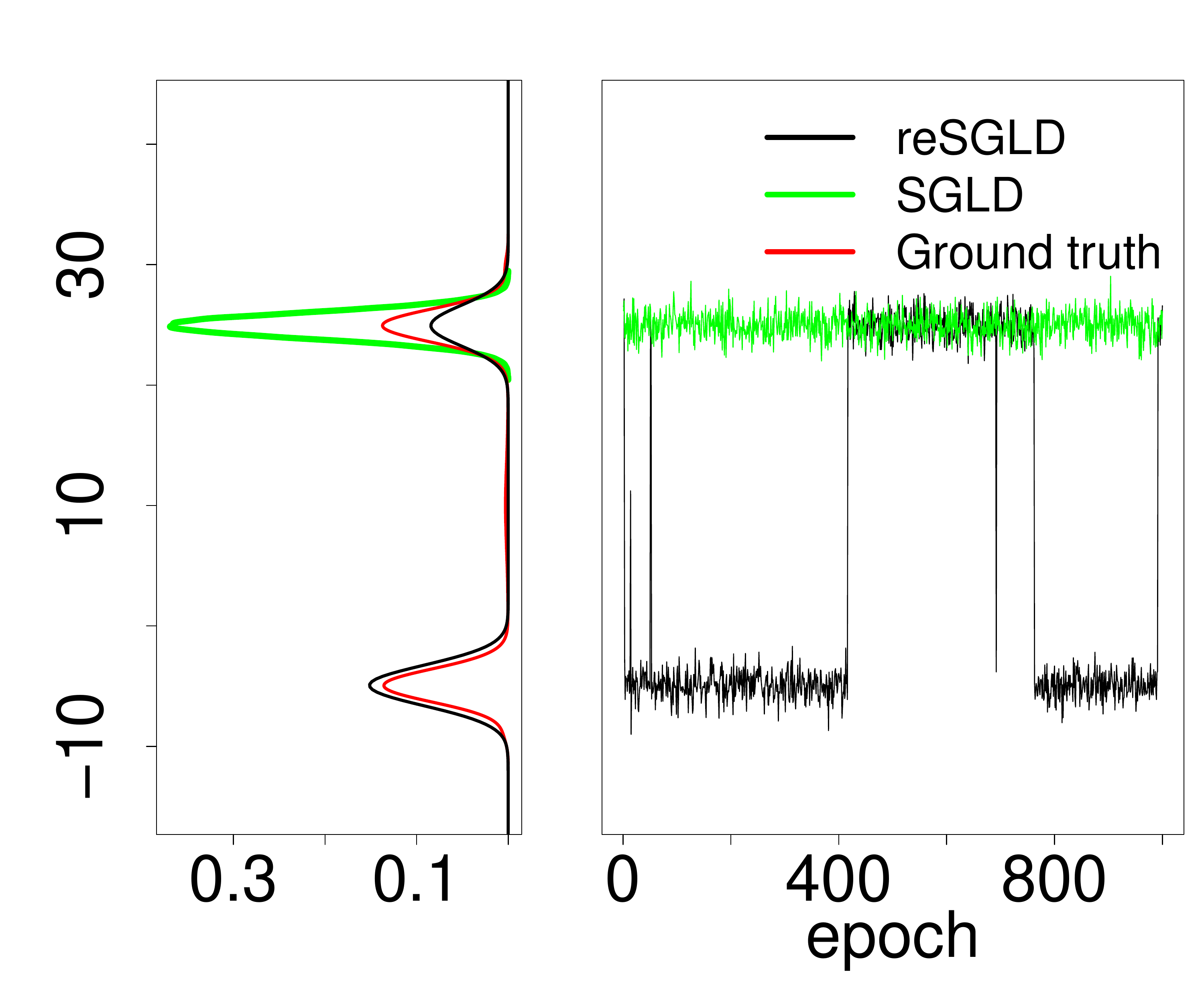}}
\enskip
\subfigure[Paths of $\log_{10}\widetilde{\sigma}^{2}$\label{fig:Ex1_variancetraceplot}]{\includegraphics[width=0.21\columnwidth]{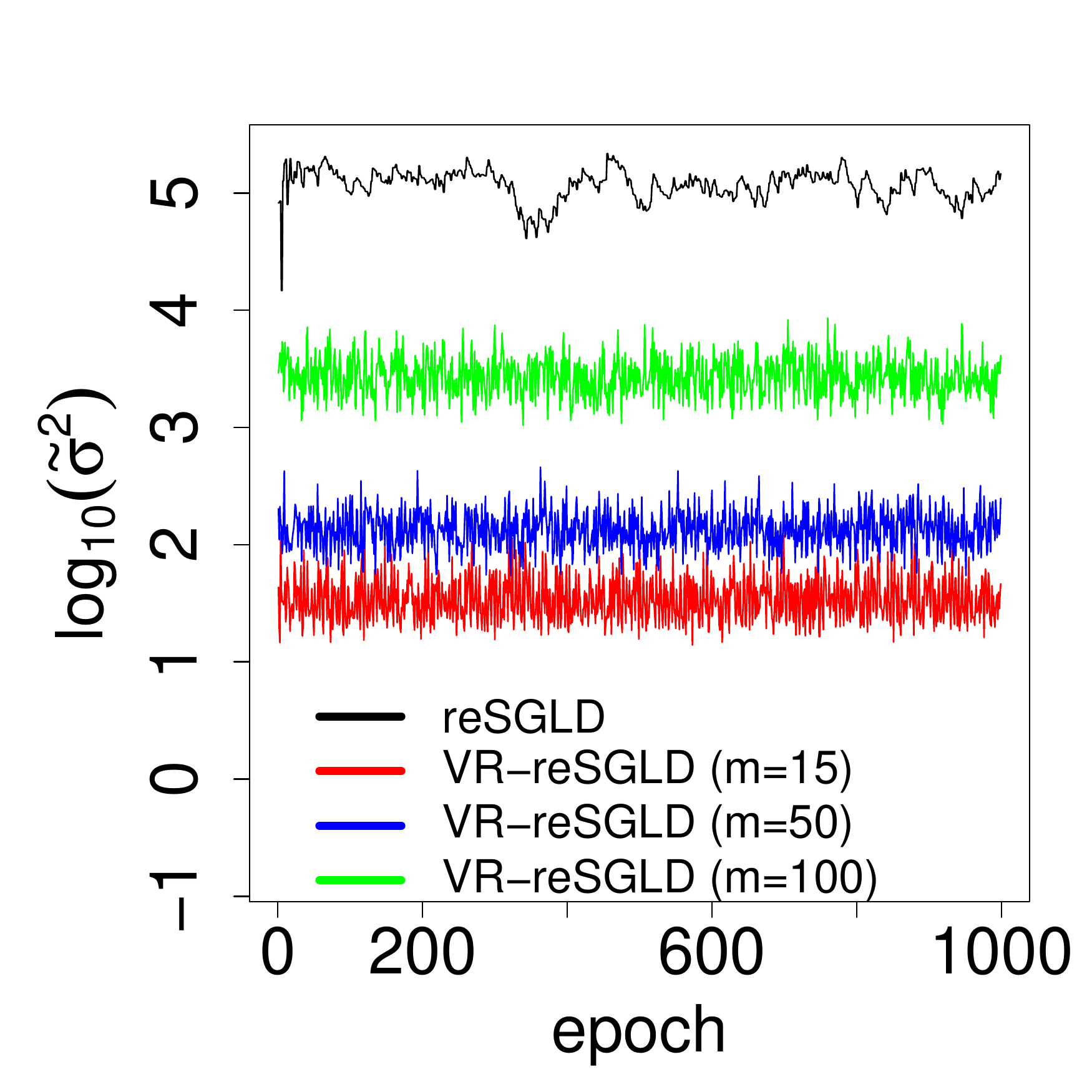}}
\enskip
\subfigure[Contour of $\log_{10}\widetilde{\sigma}^{2}$ \label{fig:Ex1_tau2_m_relation}]{\includegraphics[width=0.22\columnwidth]{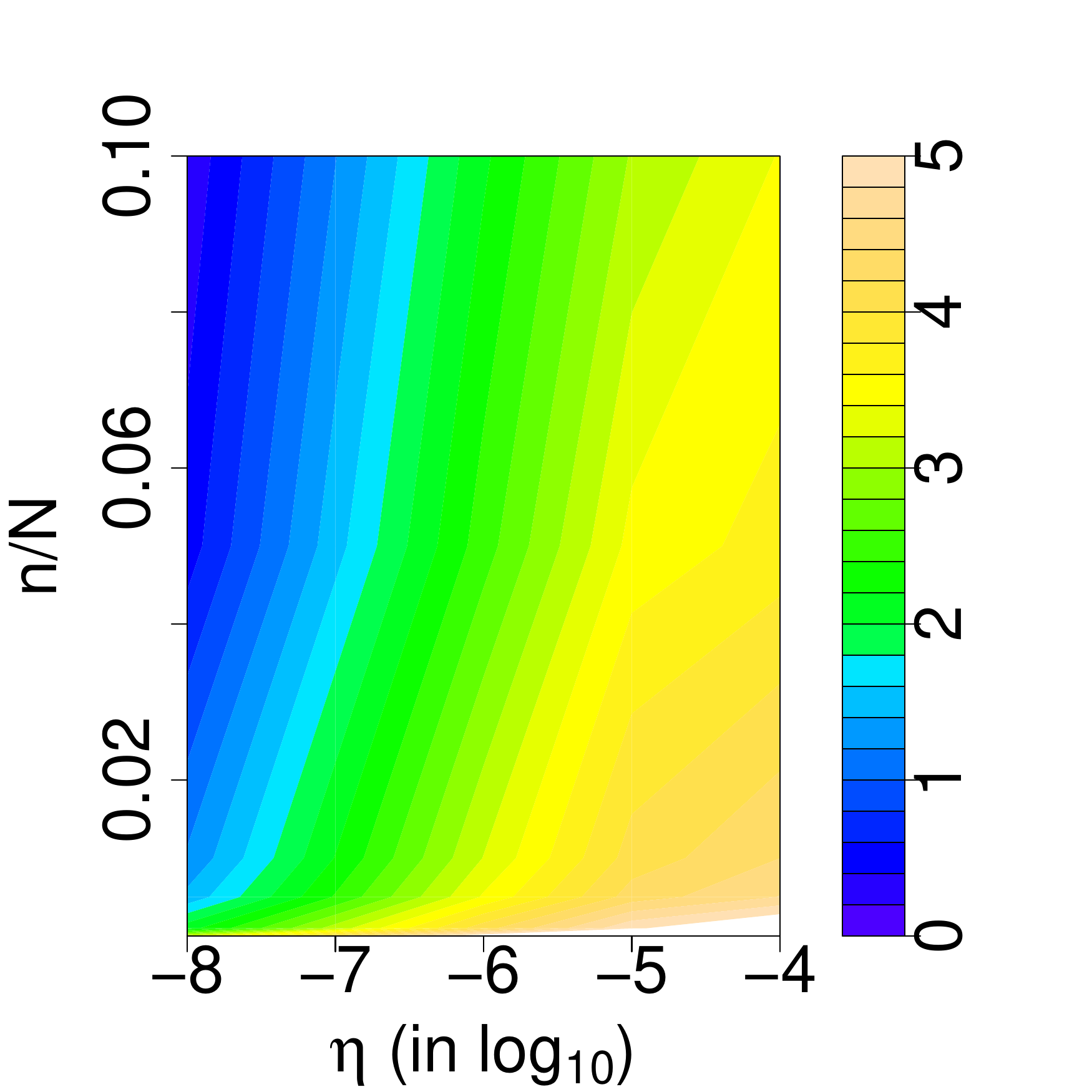}}
\vskip -0.2in
\caption{Trace plots, KDEs of $\bbeta^{(1)}$, and sensitivity study of $\widetilde{\sigma}^2$ with respect to $m, \eta$ and $n$.}
% \vspace{-1em}
\end{figure}

In Figs \ref{fig:Ex1_recovery_svrg_csgld_theta_1} and \ref{fig:Ex1_recovery},
we present trace plots and kernel density estimates (KDE) of samples generated 
from   VR-reSGLD with $m=40$, $\tau^{(1)}=10$ \footnote[2]{We choose $\tau^{(1)}=10$ instead of $1$ to avoid peaky modes for ease of illustration.}, $\tau^{(2)}=1000$,
$\eta=1e-7$, and $F=1$; reSGLD adopt the same hyper-parameters except for $F=100$ because a smaller $F$ may fail to propose any swaps;
SGLD uses $\eta=1e-7$ and $\tau=10$. As the posterior
density is intractable, we consider a ground truth
by running replica exchange Langevin dynamics with long enough iterations. We observe
that VR-reSGLD is able to fully recover the posterior density, and successfully jump between
the two modes passing the energy barrier frequently enough. By contrast, SGLD, initialized at $\beta_{0}=30$,
is attracted to the nearest mode and fails to escape throughout
the run; reSGLD manages to jump between the 
two modes, however,  $F$ is chosen as large as $100$, which induces a large bias and only yields three to five swaps and exhibits the metastability issue. 
In Figure \ref{fig:Ex1_variancetraceplot}, we present the evolution
of the variance for VR-reSGLD over a range of different $m$
and compare it with reSGLD. We see that the variance reduction mechanism has
successfully reduced the variance by hundreds of times. In Fig \ref{fig:Ex1_tau2_m_relation},
we present the sensitivity study of $\tilde{\sigma}^{2}$
as a function of the ratio $n/N$ and the
learning rate $\eta$; for this estimate we average out $10$ realizations
of VR-reSGLD, and our results agree with the theoretical analysis in Lemma \ref{vr-estimator_main}.

\subsection{Non-convex optimization for image data}
\label{nonconvex_optimization}

We further test the proposed algorithm on CIFAR10 and CIFAR100. We choose the 20, 32, 56-layer residual networks as the training models and denote them by ResNet-20, ResNet-32, and ResNet-56, respectively. Considering the wide adoption of M-SGD, stochastic gradient Hamiltonian Monte Carlo (SGHMC) is selected as the baseline. We refer to the standard replica exchange SGHMC algorithm as reSGHMC and the variance-reduced reSGHMC algorithm as VR-reSGHMC. We also include another baseline called cyclical stochastic gradient MCMC (cycSGHMC), which proposes a cyclical learning rate schedule. To make a fair comparison, we test the variance-reduced replica exchange SGHMC algorithm with cyclic learning rates and refer to it as cVR-reSGHMC.

\begin{figure*}[!ht]
  \centering
  \vskip -0.1in
  \subfigure[\footnotesize{CIFAR10: Original v.s. proposed (m=50)} ]{\includegraphics[width=3.2cm, height=3.2cm]{figures/cifar10_batch_256_variance_acceleration_no_2nd_y_V2.pdf}}\label{fig: c2a}\enskip
  \subfigure[\footnotesize{CIFAR100: Original v.s. proposed (m=50)} ]{\includegraphics[width=3.2cm, height=3.2cm]{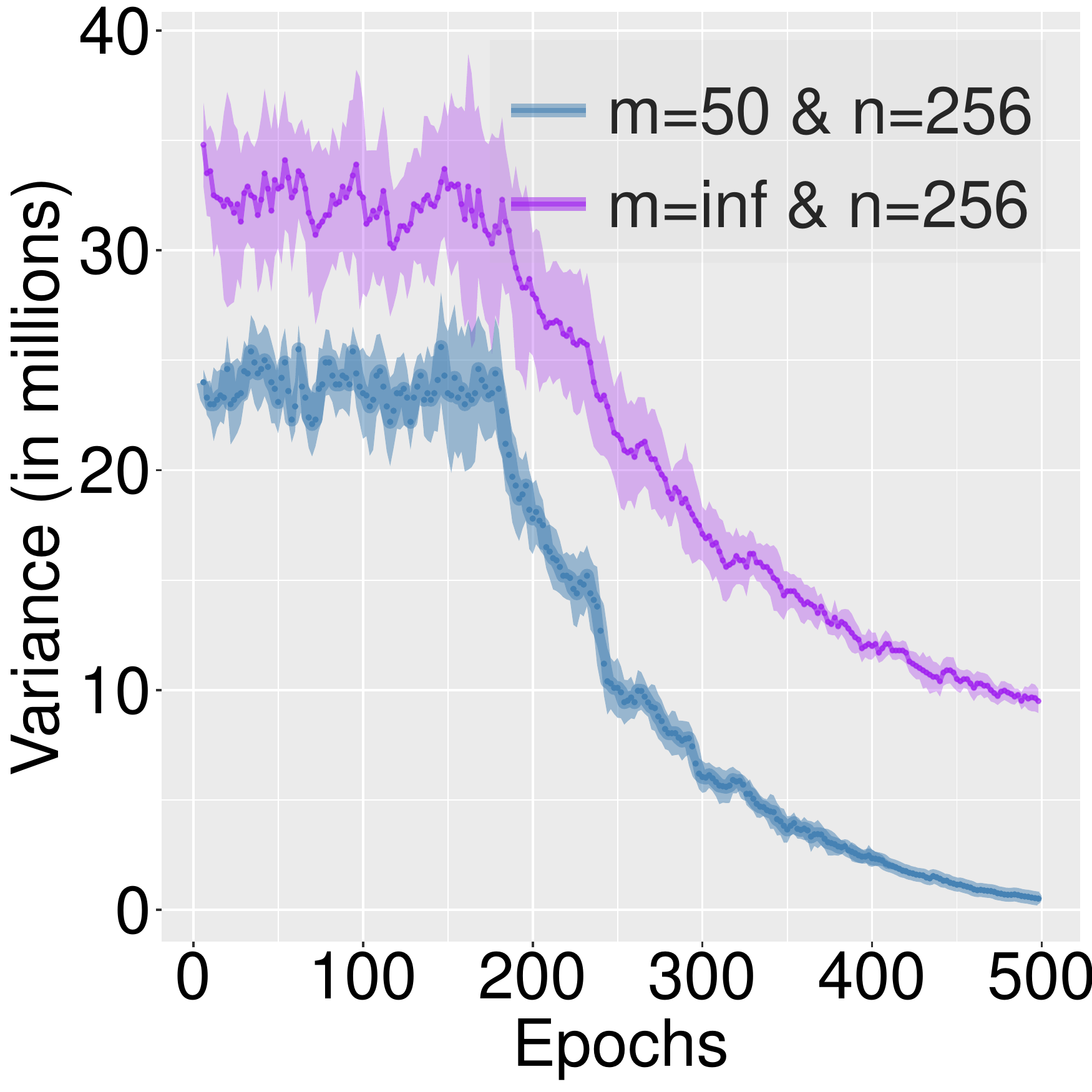}}\label{fig: c2b}\enskip
  \subfigure[Variance reduction setups on CIFAR10]{\includegraphics[width=3.2cm, height=3.2cm]{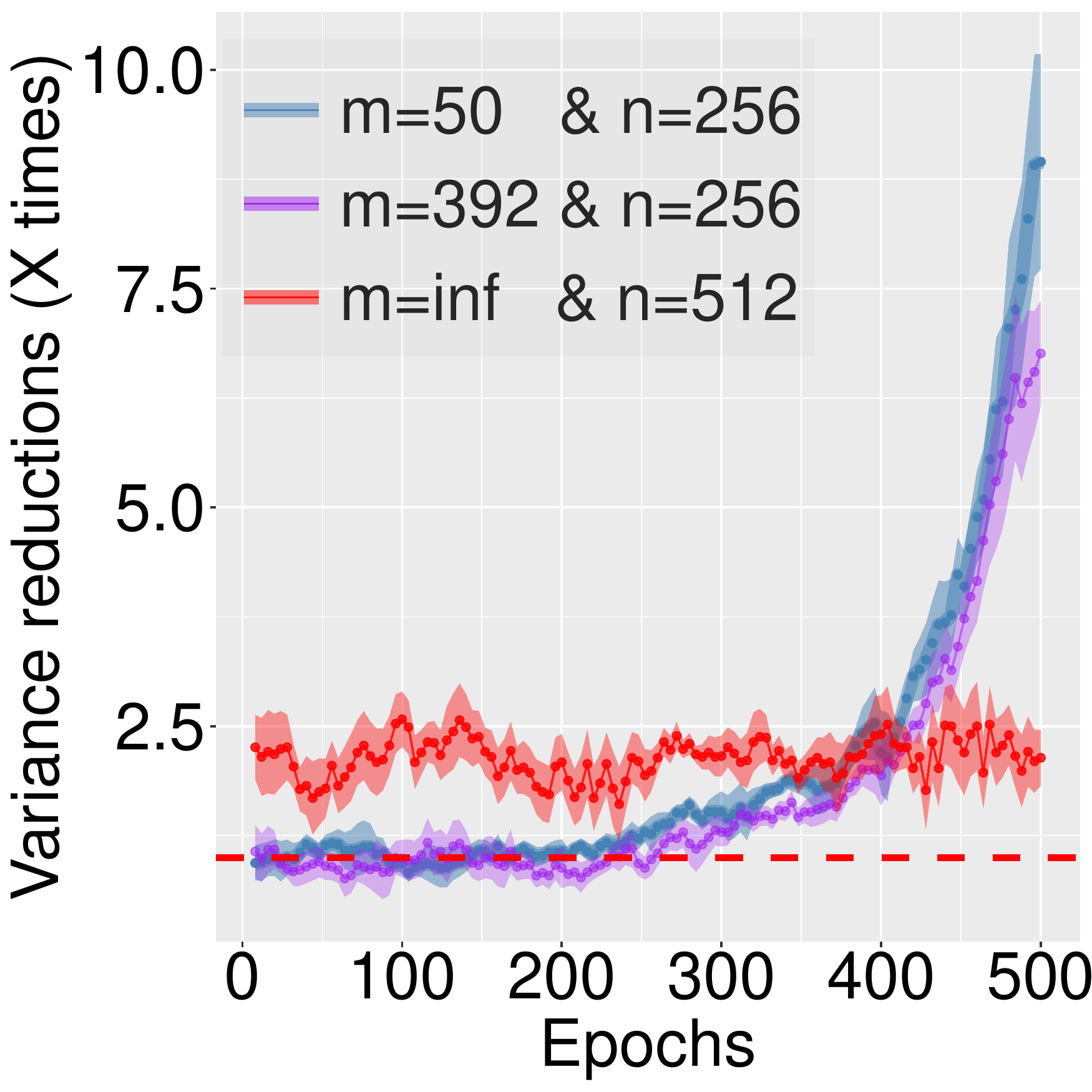}}\label{fig: c2c}\enskip
  \subfigure[Variance reduction setups on CIFAR100]{\includegraphics[width=3.2cm, height=3.2cm]{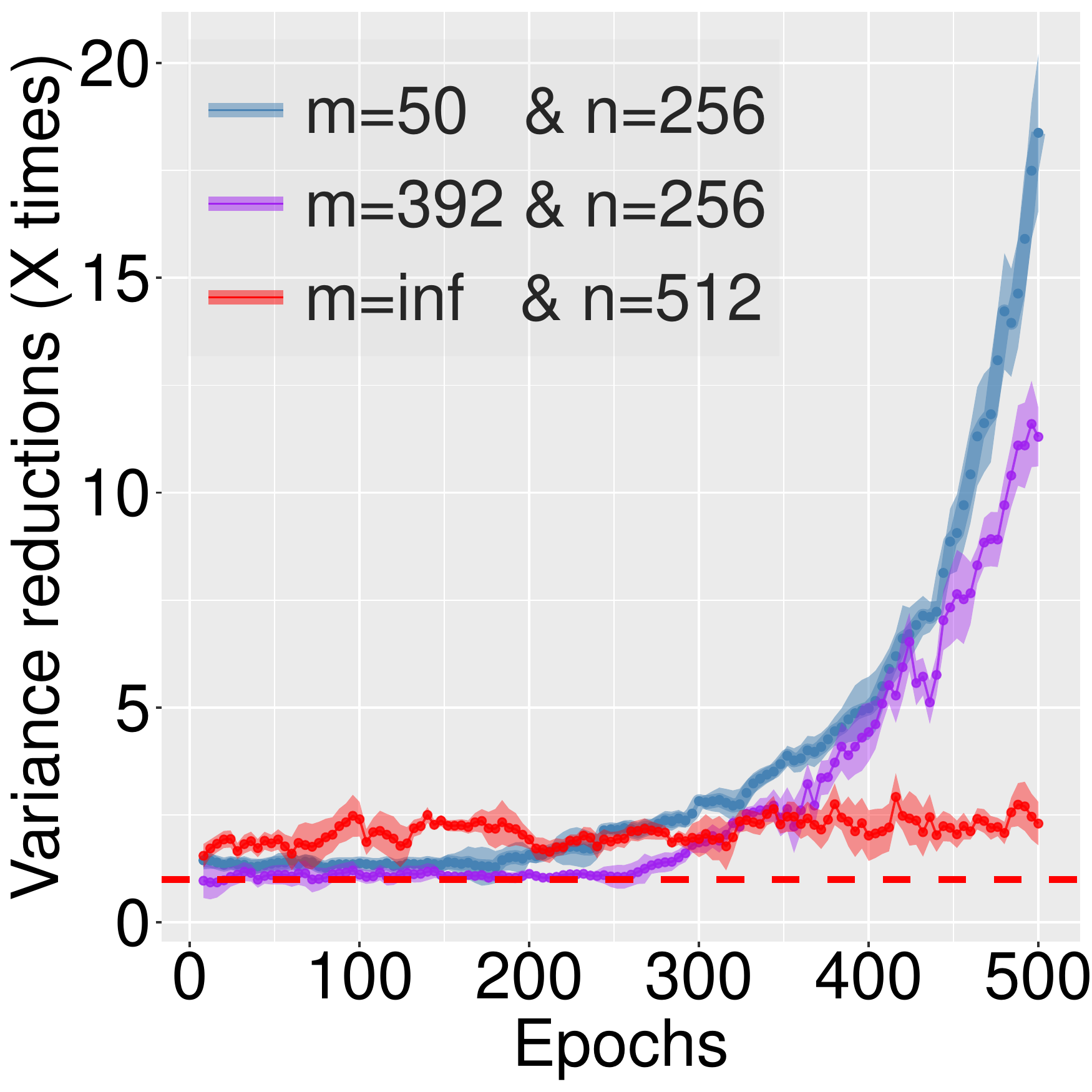}}\label{fig: c2d}
    \vskip -0.1in
  \caption{Variance reduction on the noisy energy estimators on CIFAR10 \& CIFAR100 datasets.}
  \label{cifar_biases_v2}
%   \vspace{-1em}
\end{figure*}

We run M-SGD, SGHMC and (VR-)reSGHMC for 500 epochs. For these algorithms, we follow a setup from \citet{deng2020}. We fix the learning rate $\eta_k^{(1)}=\text{2e-6}$ in the first 200 epochs and decay it by 0.984 afterwards. For SGHMC and the low-temperature processes of (VR-)reSGHMC, we anneal the temperature following $\tau_k^{(1)}=0.01 / 1.02^k$ in the beginning and keep it fixed after the burn-in steps; regarding the high-temperature process, we set $\eta_k^{(2)}=1.5\eta_k^{(1)}$ and $\tau_k^{(2)}=5\tau_k^{(1)}$. The initial correction factor $F_0$ is fixed at $1.5e5$. The thinning factor $\mathbb{T}$ is set to $256$. In particular for cycSGHMC, we run the algorithm for 1000 epochs and choose the cosine learning rate schedule with 5 cycles; $\eta_0$ is set to $\text{1e-5}$; we fix the temperature 0.001 and the threshold $0.7$ for collecting the samples. Similarly, we propose the cosine learning rate for cVR-reSGHMC with 2 cycles and run it for 500 epochs using the same temperature 0.001. We only study the low-temperature process for the replica exchange algorithms. Each experiment is repeated five times to obtain the mean and 2 standard deviations.

We evaluate the performance of variance reduction using VR-reSGHMC and compare it with reSGHMC. We first increase the batch size $n$ from 256 to 512 for reSGHMC and notice that the reduction of variance is around 2 times (see the red curves in Fig.\ref{cifar_biases_v2}(c,d)). Next, we try $m=50$ and $n=256$ for the VR-reSGHMC algorithm, which updates the control variates every 50 iterations. As shown in Fig.\ref{cifar_biases_v2}(a,b), during the first 200 epochs, where the largest learning rate is used, the variance of VR-reSGHMC is slightly reduced by 37\% on CIFAR100 and doesn't make a difference on CIFAR10. However, as the learning rate and the temperature decrease, the reduction of the variance gets more significant. We see from  Fig.\ref{cifar_biases_v2}(c,d) that the reduction of variance can be \emph{up to 10 times on CIFAR10 and 20 times on CIFAR100}. This is consistent with our theory proposed in Lemma \ref{vr-estimator_main}. The reduction of variance based on VR-reSGHMC starts to outperform the baseline with $n=512$ when the epoch is higher than 370 on CIFAR10 and 250 on CIFAR100. We also try $m=392$, which updates the control variates every 2 epochs, and find a similar pattern.

For computational reasons, we choose $m=392$ and $n=256$ for (c)VR-reSGHMC and compare them with the baseline algorithms. With the help of swaps between two SGHMC chains, reSGHMC already obtains remarkable performance \citep{deng2020} and five swaps often lead to an optimal performance. However, VR-reSGHMC still outperforms reSGHMC by around 0.2\% on CIFAR10 and 1\% improvement on CIFAR100  (Table.\ref{cifar_all_results}) and \emph{the number of swaps is increased to around a hundred under the same setting}. We also try cyclic learning rates and compare cVR-reSGHMC with cycSGHMC, we see cVR-reSGHMC outperforms cycSGHMC significantly even if cycSGHMC is running 1000 epochs, which may be more costly than cVR-reSGHMC due to the lack of mechanism in parallelism. Note that cVR-reSGHMC keeps the temperature the same instead of annealing it as in VR-reSGHMC, which is more suitable for uncertainty quantification.

\begin{table*}[ht]
\fontsize{8.5}{11}
\begin{sc}
\vskip -0.15in
\caption[Table caption text]{Prediction accuracies (\%) based on Bayesian model averaging. In particular, M-SGD and SGHMC run 500 epochs using a single chain; cycSGHMC run 1000 epochs using a single chain; replica exchange algorithms run 500 epochs using two chains with different temperatures.}\label{cifar_all_results}
\vskip -0.3in
\begin{center} 
\begin{tabular}{c|ccc|ccc}
\hline
\multirow{2}{*}{Method} & \multicolumn{3}{c|}{CIFAR10} & \multicolumn{3}{c}{CIFAR100} \\
\cline{2-7}
 & ResNet20 & ResNet32  & ResNet56 & ResNet20 & ResNet32 & ResNet56 \\
\hline
\hline
M-SGD & 94.07$\pm$0.11 & 95.11$\pm$0.07 & 96.05$\pm$0.21 & 71.93$\pm$0.13 & 74.65$\pm$0.20  & 78.76$\pm$0.24  \\
SGHMC & 94.16$\pm$0.13 & 95.17$\pm$0.08 & 96.04$\pm$0.18 & 72.09$\pm$0.14 & 74.80$\pm$0.19  & 78.95$\pm$0.22 \\ 
\hline
 
\upshape{re}SGHMC & 94.56$\pm$0.23 & 95.44$\pm$0.16 & 96.15$\pm$0.17 & 73.94$\pm$0.34 & 76.38$\pm$0.23  & 79.86$\pm$0.26 \\ 
\scriptsize{VR-\upshape{re}SGHMC} & \textbf{94.84$\pm$0.11}  & \textbf{95.62$\pm$0.09} &  \textbf{96.32$\pm$0.15}  & \textbf{74.83$\pm$0.18} & \textbf{77.40$\pm$0.27}  & \textbf{80.62$\pm$0.22}   \\
\hline
\scriptsize{\upshape{cyc}SGHMC} &  94.61$\pm$0.15  &  95.56$\pm$0.12  & 96.19$\pm$0.17 & 74.21$\pm$0.22 & 76.60$\pm$0.25  & 80.39$\pm$0.21  \\
\scriptsize{\upshape{c}VR-\upshape{re}SGHMC} & \textbf{94.91$\pm$0.10}  & \textbf{95.64$\pm$0.13} &  \textbf{96.36$\pm$0.16}  & \textbf{75.02$\pm$0.19} & \textbf{77.58$\pm$0.21}  & \textbf{80.50$\pm$0.25}   \\
\hline
\end{tabular}
\end{center} 
\end{sc}
% \vspace{-1em}
\end{table*}

Regarding the training cost and the treatment for improving the performance of variance reduction using adaptive coefficients in the early period, we refer interested readers to Appendix \ref{CV_more}. 

For the detailed implementations, we release the code at \url{https://github.com/WayneDW/Variance_Reduced_Replica_Exchange_Stochastic_Gradient_MCMC}.

\subsection{Uncertainty Quantification for unknown samples} 

A reliable model not only makes the right decision among potential candidates but also casts doubts on irrelevant choices. For the latter, we follow \citet{Balaji17} and evaluate the uncertainty on out-of-distribution samples from unseen classes. To avoid over-confident predictions on unknown classes, the ideal predictions should yield a higher uncertainty on the out-of-distribution samples, while maintaining the accurate uncertainty for the in-distribution samples.

Continuing the setup in Sec.\ref{nonconvex_optimization}, we collect the ResNet20 models trained on CIFAR10 and quantify \begin{wrapfigure}{r}{0.61\textwidth}
  \begin{center}
  \vskip -0.2in
     \includegraphics[width=0.61\textwidth]{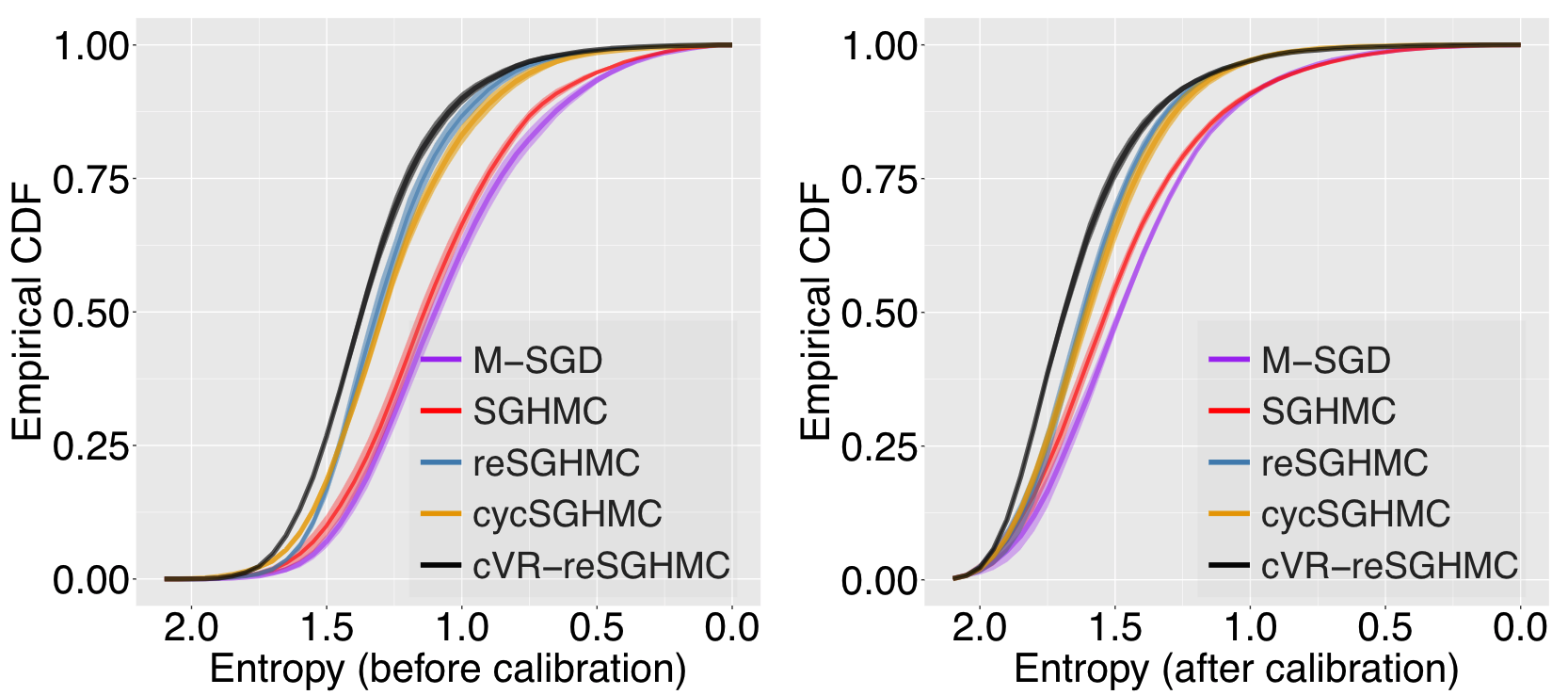}
  \end{center}
  \vskip -0.2in
  \caption{CDF of entropy for predictions on SVHN via CIFAR10 models. A temperature scaling is used in calibrations.}
  \label{UQ}
\end{wrapfigure} the entropy on the Street View House Numbers (SVHN) dataset, which contains 26,032 RGB testing images of digits instead of objects. We compare cVR-reSGHMC with M-SGD, SGHMC, reSGHMC, and cSGHMC. Ideally, the predictive distribution should be the uniform distribution and leads to the highest entropy.  We present the empirical cumulative distribution function (CDF) of the entropy of the predictions on SVHN and report it in Fig.\ref{UQ}. As shown in the left figure, M-SGD shows the smallest probability for high-entropy predictions, implying the weakness of stochastic optimization methods in uncertainty estimates. By contrast, the proposed cVR-reSGHMC yields the highest probability for predictions of high entropy. Admittedly, the standard ResNet models are poorly calibrated in the predictive probabilities and lead to inaccurate confidence. To alleviate this issue, we adopt the temperature-scaling method with a scale of 2 to calibrate the predictive distribution \citep{temperature_scaling} and present the entropy in Fig.\ref{UQ} (right). In particular, we see that 77\% of the predictions from cVR-reSGHMC yields the entropy higher than 1.5, which is 7\% higher than reSGHMC and 10\% higher than cSGHMC and much better than the others.

For more discussions of uncertainty estimates on both datasets, we leave the results in Appendix \ref{UQ_more}.

\section{Conclusion}

We propose the variance-reduced replica exchange stochastic gradient Langevin dynamics algorithm to accelerate the convergence by reducing the variance of the noisy energy estimators. Theoretically, this is \emph{the first variance reduction method that yields the potential of exponential accelerations} instead of solely reducing the discretization error. In addition, we bypass the Gr\"{o}nwall inequality to avoid the crude numerical error and consider a change of Poisson measure in the generalized Girsanov theorem to obtain a much tighter upper bound. Since our variance reduction only conducts on the noisy energy estimators and is not applied to the noisy gradients, the standard hyper-parameter setting can be also naturally imported, which greatly facilitates the training of deep neural works.

\section*{Acknowledgment} 

We would like to thank Maxim Raginsky and the anonymous reviewers for their insightful suggestions. Liang's research was supported in part by the grants DMS-2015498, R01-GM117597 and R01-GM126089. Lin acknowledges the support from NSF (DMS-1555072, DMS-1736364), BNL Subcontract 382247, W911NF-15-1-0562, and DE-SC0021142.

\bibliography{mybib2}
\bibliographystyle{iclr2021_conference}

\newpage

\appendix

\setcounter{lemma}{0}
\renewcommand{\thelemma}{\Alph{section}\arabic{lemma}}
\def\qed{ \ \vrule width.2cm height.2cm depth0cm\smallskip}
\newcommand\myeq{\stackrel{\mathclap{\normalfont\mbox{A}}}{=}}

\newcommand{\la}{\langle}
\def \proof{{\noindent \bf Proof\quad}}

% \maketitle

\section{Preliminaries}
\label{prelim}
\textbf{Notation} We denote the deterministic energy based on the parameter $\bbeta$ by $L(\bbeta)=\sum_{i=1}^N L(\bx_i|\bbeta)$ using the full dataset of size $N$. We denote the unbiased stochastic energy estimator by $\frac{N}{n}\sum_{i\in B} L(\bx_i|\bbeta)$ using the mini-batch of data $B$ of size $n$. %Sometimes, the unbiased stochastic energy estimator $\frac{N}{n}\sum_{i\in B} L(\bx_i|\bbeta)$ is written as $\widetilde L(\bbeta)$ if it doesn't cause confusion. 
The same style of notations is also applicable to the gradient for consistency. We denote the Euclidean $L^2$ norm by $\|\cdot\|$. To prove the desired results, we need the following assumptions:
\begin{assumption}[Smoothness]\label{assump: lip and alpha beta}
The energy function $L(\bx_i|\cdot)$ is $C_N$-smoothness if there exists a constant $C_N>0$ such that $\forall \bbeta_1,\bbeta_2\in\hR^d$, $i\in\{1,2,\cdots, N\}$, we have
\begin{equation}
\label{1st_smooth_condition}
    \|\nabla L(\bx_i|\bbeta_1)-\nabla L(\bx_i|\bbeta_2)\|\le C_N\|\bbeta_1-\bbeta_2\|.
\end{equation}
Note that the above condition further implies 
for a constant $C=NC_N$ and $\forall \bbeta_1,\bbeta_2\in\hR^d$, we have
\begin{equation}
\label{2nd_smooth_condition}
    \|\nabla L(\bbeta_1)-\nabla L(\bbeta_2)\|\le C\|\bbeta_1-\bbeta_2\|.
\end{equation}
\end{assumption}{}

The smoothness conditions (\ref{1st_smooth_condition}) and (\ref{2nd_smooth_condition}) are standard tools in studying the convergence of SGLD in \citep{Xu18} and \citet{Maxim17}, respectively.
\begin{assumption}[Dissipativity]\label{assump: dissipitive}
The energy function $L(\cdot)$ is $(a,b)$-dissipative if there exist constants $a>0$ and $b\ge 0$ such that $\forall \bbeta\in\mathbb R^d$,  $\la \bbeta,\nabla L(\bbeta)\rangle \ge a\|\bbeta\|^2-b.$
\end{assumption}{}

The dissipativity condition implies that the Markov process is able to move inward on average regardless of the starting position. It has been widely used in proving the geometric ergodicity of dynamic systems \citep{mattingly02, Maxim17, Xu18}. 

\begin{assumption}[Gradient oracle]\label{assump: stochastic_noise}
There exists a constant $\delta\in[0, 1)$ such that for any $\bbeta$, we have
\begin{equation}
    \E[\|\nabla \widetilde L(\bbeta)-\nabla L(\bbeta)\|^2]\leq 2\delta (C^2 \|\bbeta\|^2+\Phi^2),
\end{equation}
where $\Phi$ is a positive constant. The same assumption has been used in \citet{Maxim17} to control the stochastic noise from the gradient.
\end{assumption}{}

\section{Exponential accelerations via Variance reduction}
\label{exp_acc}
\setcounter{lemma}{0}
\setcounter{theorem}{0}

We aim to build an efficient estimator to approximate the deterministic swapping rate $S(\bbeta^{(1)}, \bbeta^{(2)})$
\begin{equation}
\label{S_exact}
    S(\bbeta^{(1)}, \bbeta^{(2)})=e^{ \left(\frac{1}{\tau^{(1)}}-\frac{1}{\tau^{(2)}}\right)\left( \sum_{i=1}^N L(\bx_i|\bbeta^{(1)})-\sum_{i=1}^N L(\bx_i|\bbeta^{(2)})\right)}.
\end{equation}

In big data problems and deep learning, it is too expensive to evaluate the energy $\sum_{i=1}^N L(\bx_i|\bbeta)$ for each $\bbeta$ for a large $N$. To handle the computational issues, a popular solution is to use the unbiased stochastic energy $\frac{N}{n}\sum_{i\in B} L(\bx_i|\bbeta)$ for a random mini-batch data $B$ of size $n$. However, a n\"{a}ive replacement of $\sum_{i=1}^N L(\bx_i|\bbeta)$ by $\frac{N}{n}\sum_{i\in B} L(\bx_i|\bbeta)$ leads to a large bias to the swapping rate. To remove such a bias, we follow \citet{deng2020} and consider the corrected swapping rate
\begin{equation}
\begin{split}
    \widehat S(\bbeta^{(1)}, \bbeta^{(2)})&=e^{ \left(\frac{1}{\tau^{(1)}}-\frac{1}{\tau^{(2)}}\right)\left( \frac{N}{n}\sum_{i\in B} L(\bx_i|\bbeta^{(1)})-\frac{N}{n}\sum_{i\in B} L(\bx_i|\bbeta^{(2)})-\left(\frac{1}{\tau^{(1)}}-\frac{1}{\tau^{(2)}}\right)\frac{\widehat \sigma^2}{2}\right)},\\
\end{split}
\end{equation}
where $\widehat\sigma^2$ denotes the variance of $\frac{N}{n}\sum_{i\in B} L(\bx_i|\bbeta^{(1)})-\frac{N}{n}\sum_{i\in B} L(\bx_i|\bbeta^{(2)})$. \footnote{We only consider the case of $F=1$ in the stochastic swapping rate for ease of analysis.} Empirically, $\widehat \sigma^2$ is quite large, resulting in almost no swaps and insignificant accelerations. To propose more effective swaps, we consider the variance-reduced estimator
\begin{equation}
    \widetilde L(B_k|\bbeta_k)=\frac{N}{n}\sum_{i\in B_k}\left( L(\bx_i| \bbeta_k) - L\left(\bx_i\Big| \bbeta_{m\lfloor \frac{k}{m}\rfloor}\right) \right)+\sum_{i=1}^N L\left(\bx_i\Big| \bbeta_{m\lfloor \frac{k}{m}\rfloor}\right),
\end{equation}
where the control variate $\bbeta_{m\lfloor \frac{k}{m}\rfloor}$ is updated every $m$ iterations. Denote the variance of $ \widetilde L(B|\bbeta^{(1)})- \widetilde L(B|\bbeta^{(2)})$ by $\widetilde\sigma^2$. The variance-reduced stochastic swapping rate follows
\begin{equation}
\begin{split}
\label{vr_s}
    \widetilde S_{\eta, m, n}(\bbeta^{(1)}, \bbeta^{(2)})&=e^{ \left(\frac{1}{\tau^{(1)}}-\frac{1}{\tau^{(2)}}\right)\left( \widetilde L(B|\bbeta^{(1)})- \widetilde L(B|\bbeta^{(2)})-\left(\frac{1}{\tau^{(1)}}-\frac{1}{\tau^{(2)}}\right)\frac{\widetilde\sigma^2}{2}\right)}.\\
\end{split}
\end{equation}

Using the strategy of variance reduction, we can lay down the first result, which differs from the existing variance reduction methods in that we only conduct variance reduction in the energy estimator for the class of SGLD algorithms.
\begin{lemma}[Variance-reduced energy estimator]
\label{vr-estimator}
Under the smoothness
and dissipativity assumptions \ref{assump: lip and alpha beta} and \ref{assump: dissipitive}, the variance of the variance-reduced energy estimator $\widetilde L(B_{k}|\bbeta_{k}^{(h)})$, where $h\in\{1,2\}$, is upper bounded by
\begin{equation}
    \Var\left(\widetilde L(B_{k}|\bbeta_{k}^{(h)})\right)\leq \frac{m^2 \eta}{n}D_R^2\left( \frac{2\eta}{n} (2C^2\Psi_{d,\tau^{(2)}, C, a, b} +2Q^2)+4\tau^{(2)} d\right).
\end{equation}
where $D_R=CR+\max_{i\in\{1,2,\cdots, N\}} N \|\nabla L(\bx_i|\bbeta_{\star})\|+\frac{Cb}{a}$ and $R$ is the radius of a sufficiently large ball that contains $\bbeta_k^{(h)}$ for $h\in\{1,2\}$.
% where $\widetilde L(B_{k}|\bbeta_{k}^{(h)})=\frac{N}{n}\sum_{i\in B_{k}}\left( L(\bx_i| \bbeta_{k}^{(h)}) - L\left(\bx_i\Big| \bbeta^{(h)}_{m\lfloor \frac{k}{m}\rfloor}\right) \right)+\sum_{i=1}^N L\left(\bx_i\Big| \bbeta^{(h)}_{m\lfloor \frac{k}{m}\rfloor}\right)$.
\end{lemma}

\begin{proof}
\begin{equation}
\label{var_1st}
    \footnotesize
    \begin{split}
      &\Var\left(\widetilde L(B_{k}|\bbeta_{k}^{(h)})\right)\\
      =&\E\left[\left(\frac{N}{n}\sum_{i\in B_k}\left[ L(\bx_i| \bbeta_{k}^{(h)}) - L\left(\bx_i\Big| \bbeta^{(h)}_{m\lfloor \frac{k}{m}\rfloor}\right) \right]+\sum_{j=1}^N L\left(\bx_j\Big| \bbeta^{(h)}_{m\lfloor \frac{k}{m}\rfloor}\right)-\sum_{j=1}^N L(\bx_j| \bbeta^{(h)}_k)\right)^2\right]\\
      =&\E\left[\left(\frac{N}{n}\sum_{i\in B_k}\left[ L(\bx_i| \bbeta_{k}^{(h)}) - L\left(\bx_i\Big| \bbeta^{(h)}_{m\lfloor \frac{k}{m}\rfloor}\right) +\frac{1}{N}\left(\sum_{j=1}^N L\left(\bx_j\Big| \bbeta^{(h)}_{m\lfloor \frac{k}{m}\rfloor}\right)-\sum_{j=1}^N L(\bx_j| \bbeta^{(h)}_k)\right)\right]\right)^2\right]\\
      =&\frac{N^2}{n^2}\E\left[\left(\sum_{i\in B_k}\left[ L(\bx_i| \bbeta_{k}^{(h)}) - L\left(\bx_i\Big| \bbeta^{(h)}_{m\lfloor \frac{k}{m}\rfloor}\right) +\frac{1}{N}\left(\sum_{j=1}^N L\left(\bx_j\Big| \bbeta^{(h)}_{m\lfloor \frac{k}{m}\rfloor}\right)-\sum_{j=1}^N L(\bx_j| \bbeta^{(h)}_k)\right)\right]\right)^2\right]\\
      =&\frac{N^2}{n^2}\sum_{i\in B_k}\E\left[\left( L(\bx_i| \bbeta_{k}^{(h)}) - L\left(\bx_i\Big| \bbeta^{(h)}_{m\lfloor \frac{k}{m}\rfloor}\right) -\frac{1}{N}\left[\sum_{j=1}^N L(\bx_j| \bbeta_k^{(h)})-\sum_{j=1}^N L\left(\bx_j\Big| \bbeta^{(h)}_{m\lfloor \frac{k}{m}\rfloor}\right)\right]\right)^2\right]\\
      \leq & \frac{N^2}{n^2}\sum_{i\in B_k}\E\left[\left( L(\bx_i| \bbeta_{k}^{(h)}) - L\left(\bx_i\Big| \bbeta^{(h)}_{m\lfloor \frac{k}{m}\rfloor}\right)\right)^2\right]\\
      \leq & \frac{D_R^2}{n}\E\left[\left\|\bbeta_{k}^{(h)}-\bbeta^{(h)}_{m\lfloor \frac{k}{m}\rfloor}\right\|^2\right],
    \end{split}
\end{equation}

where the last equality follows from the fact that $\E[(\sum_{i=1}^n x_i)^2]=\sum_{i=1}^n \E[x_i^2]$ for independent variables $\{x_i\}_{i=1}^n$ with mean 0. The first inequality follows from $\E[(x-\E[x])^2]\leq \E[x^2]$ and the last inequality follows from Lemma \ref{local_smooth}, where $D_R=CR+\max_{i\in\{1,2,\cdots, N\}} N \|\nabla L(\bx_i|\bbeta_{\star})\|+\frac{Cb}{a}$ and $R$ is the radius of a sufficiently large ball that contains $\bbeta_k^{(h)}$ for $h\in\{1,2\}$.

Next, we bound $\E\left[\left\|\bbeta_{k}^{(h)}-\bbeta^{(h)}_{m\lfloor \frac{k}{m}\rfloor}\right\|^2\right]$ as follows %for any $k+1\in \left\{m\lfloor \frac{k}{m}\rfloor+1, m\lfloor \frac{k}{m}\rfloor+2, \cdots, m\lfloor \frac{k}{m}+m\rfloor\right\}$ 
\begin{equation}
\label{var_2nd}
\small
    \E\left[\left\|\bbeta_{k}^{(h)}-\bbeta^{(h)}_{m\lfloor \frac{k}{m}\rfloor}\right\|^2\right]\leq \E\left[\left\|\sum_{j=m\lfloor \frac{k}{m}\rfloor}^{k-1} (\bbeta_{j+1}^{(h)}-\bbeta_{j}^{(h)})\right\|^2\right]\leq m\sum_{j=m\lfloor \frac{k}{m}\rfloor}^{k-1}\E\left[\left\| (\bbeta_{j+1}^{(h)}-\bbeta_{j}^{(h)})\right\|^2\right].
\end{equation}
% where $g=m\lfloor \frac{k}{m}\rfloor+m-1$.
For each term, we have the following bound
\begin{equation}
\label{var_3rd}
\begin{split}
    \E\left[\left\| \bbeta_{j+1}^{(h)}-\bbeta_{j}^{(h)}\right\|^2\right]
    =&\E\left[\left\|\eta \frac{N}{n}\sum_{i\in B_k}\nabla L(\bx_i|\bbeta_{k}^{(h)})+\sqrt{2\eta\tau^{(h)}}\bxi_k\right\|^2\right]\\
    \leq & \frac{2\eta^2 N^2}{n^2}\sum_{i\in B_k} \E\left[\left\|\nabla L(\bx_i|\bbeta_{k}^{(h)})\right\|^2\right]+4\eta\tau^{(2)} d\\
    \leq & \frac{2\eta^2}{n} (2C^2 \E[\|\bbeta_k^{(h)}\|^2]+2Q^2)+4\eta\tau^{(2)} d\\
    \leq & \frac{2\eta^2}{n} (2C^2\Psi_{d,\tau^{(2)}, C, a, b} +2Q^2)+4\eta\tau^{(2)} d,\\
\end{split}
\end{equation}
where the first inequality follows by $\E[\|a+b\|^2]\leq 2\E[\|a\|^2]+2\E[\|b\|^2]$, the i.i.d of the data points and $\tau^{(1)}\leq \tau^{(2)}$ for $h\in\{1,2\}$; the second inequality follows by Lemma \ref{grad_bound}; the last inequality follows from Lemma \ref{Uniform_bound}.

Combining (\ref{var_1st}), (\ref{var_2nd}) and (\ref{var_3rd}), we have
\begin{equation}
    \Var\left(\widetilde L(B_{k}|\bbeta_{k}^{(h)})\right)\leq \frac{m^2 \eta}{n}D_R^2\left( \frac{2\eta}{n} (2C^2\Psi_{d,\tau^{(2)}, C, a, b} +2Q^2)+4\tau^{(2)} d\right).
\end{equation}
\qed
\end{proof}

Since $\Var\left(\widetilde L(B_{k}|\bbeta_{k}^{(h)})\right)\leq \Var\left(\frac{N}{n}\sum_{i\in B}L(\bx_i| \bbeta_k)\right) +\Var\left(\frac{N}{n}\sum_{i\in B} L\left(\bx_i\Big| \bbeta_{m\lfloor \frac{k}{m}\rfloor}\right)\right)$ by definition,  $\Var\left(\widetilde L(B_{k}|\bbeta_{k}^{(h)})\right)$ is upper bounded by $\mathcal{O}\left(\min\{\widehat\sigma^2, \frac{m^2 \eta}{n}\}\right)$, which becomes much smaller using a small learning rate $\eta$, a shorter period $m$ and a large batch size $n$. 

Note that $\widetilde S_{\eta, m, n}(\bbeta^{(1)}, \bbeta^{(2)})$ is defined on the unbounded support $[0, \infty]$ and $\E[\widetilde S_{\eta, m, n}(\bbeta^{(1)}, \bbeta^{(2)})]=S(\bbeta^{(1)}, \bbeta^{(2)})$ regardless of the scale of $\widetilde\sigma^2$. To satisfy the (stochastic) reversibility condition, we consider the truncated swapping rate $\min\{1, \widetilde S_{\eta, m, n}(\bbeta^{(1)}, \bbeta^{(2)})\}$, which still targets the same invariant distribution (see section 3.1 \citep{Matias19} for details). We can show that the swapping rate may even decrease exponentially as the variance increases. %The following result shows a dependence of the truncated swapping rate on the learning rate $\eta$, the period $m$ and the batch size $n$.

\begin{lemma}[Variance reduction for larger swapping rates] \label{exp_S} Given a large enough batch size $n$, the variance-reduced energy estimator $\widetilde L(B_{k}|\bbeta_{k}^{(h)})$ yields a truncated swapping rate that satisfies
\begin{equation}
     \E[\min\{1, \widetilde S_{\eta, m, n}(\bbeta^{(1)}, \bbeta^{(2)})\}]\approx\min\Big\{1, S(\bbeta^{(1)}, \bbeta^{(2)})\left(\mathcal{O}\left(\frac{1}{n^2}\right)+e^{-\mathcal{O}\left(\frac{m^2\eta}{n}+\frac{1}{n^2}\right)}\right)\Big\}.
\end{equation}
% $\widetilde S_{\eta, m, n}=Se^{-\mathcal{O}(\frac{m^2\eta}{n})}$, 
% which is exponentially dependent on the learning rate $\eta$, the shorter period $m$ and the batch size $n$.
\end{lemma}

\begin{proof}

By central limit theorem, the energy estimator $\frac{N}{n}\sum_{i\in B} L(\bx_i|\bbeta_k)$ converges in distribution to a normal distributions as the batch size $n$ goes to infinity. In what follows, the variance-reduced estimator $\widetilde L(B_k|\bbeta_k)$ also converges to a normal distribution, where the corresponding estimator is denoted by $\mathbb{\widetilde L}(B_k|\bbeta_k)$. Now the swapping rate $\mathbb{S}_{\eta, m, n}(\cdot, \cdot)$ based on normal estimators follows
\begin{equation}
\begin{split}
\label{vr_s_normal}
    \mathbb{S}_{\eta, m, n}(\bbeta^{(1)}, \bbeta^{(2)})&=e^{ \left(\frac{1}{\tau^{(1)}}-\frac{1}{\tau^{(2)}}\right)\left( \mathbb{\widetilde L}(B|\bbeta^{(1)})- \mathbb{\widetilde L}(B|\bbeta^{(2)})-\left(\frac{1}{\tau^{(1)}}-\frac{1}{\tau^{(2)}}\right)\frac{\bar \sigma^2}{2}\right)},\\
\end{split}
\end{equation}
where $\bar \sigma^2$ denotes the variance of $\mathbb{\widetilde L}(B|\bbeta^{(1)})- \mathbb{\widetilde L}(B|\bbeta^{(2)})$. Note that $\mathbb{S}_{\eta, m, n}(\bbeta^{(1)}, \bbeta^{(2)})$ follows a log-normal distribution with mean $\log S(\bbeta^{(1)}, \bbeta^{(2)})-\left(\frac{1}{\tau^{(1)}}-\frac{1}{\tau^{(2)}}\right)^2\frac{\bar \sigma^2}{2}$ and variance $\left(\frac{1}{\tau^{(1)}}-\frac{1}{\tau^{(2)}}\right)^2\bar \sigma^2$ on the log-scale, and $S(\bbeta^{(1)}, \bbeta^{(2)})$ is the deterministic swapping rate defined in (\ref{S_exact}). Applying Lemma \ref{exponential_dependence}, we have
\begin{equation}
\begin{split}
    \E[\min\{1, \mathbb{S}_{\eta, m, n}(\bbeta^{(1)}, \bbeta^{(2)})\}]=\mathcal{O}\left(S(\bbeta^{(1)}, \bbeta^{(2)})\exp\left\{-\frac{\left(\frac{1}{\tau^{(1)}}-\frac{1}{\tau^{(2)}}\right)^2\bar \sigma^2}{8}\right\}\right).
\end{split}
\end{equation} 

Moreover, $\bar \sigma^2$ differs from $\widetilde\sigma^2$, the variance of $\widetilde L(B|\bbeta^{(1)})-\widetilde L(B|\bbeta^{(2)})$, by at most a bias of $\mathcal{O}(\frac{1}{n^2})$ according to the estimate of the third term of (S2) in \citet{Matias19} and $\widetilde\sigma^2\leq \Var\left(\widetilde L(B_{k}|\bbeta_{k}^{(1)})\right) +\Var\left(\widetilde L(B_{k}|\bbeta_{k}^{(2)})\right)$, where both $\Var\left(\widetilde L(B_{k}|\bbeta_{k}^{(1)})\right)$ and $\Var\left(\widetilde L(B_{k}|\bbeta_{k}^{(2)})\right)$ are upper bounded by $\frac{m^2 \eta}{n}D_R^2\left( \frac{2\eta}{n} (2C^2\Psi_{d,\tau^{(2)}, C, a, b} +2Q^2)+4\tau d\right)$ by Lemma \ref{vr-estimator}, it follows that
\begin{equation}
% \small
\label{normal_truncate}
\begin{split}
    \E[\min\{1, \mathbb{S}_{\eta, m, n}(\bbeta^{(1)}, \bbeta^{(2)})\}] \leq S(\bbeta^{(1)}, \bbeta^{(2)}) e^{-\mathcal{O}\left(\frac{m^2\eta}{n}+\frac{1}{n^2}\right)}.
\end{split}
\end{equation}

Applying $\min\{1,\mathbb{A}+\mathbb{B}\}\leq \min\{1,\mathbb{A}\}+|\mathbb{B}|$, we have
\begin{equation}
% \small
\label{target_eq1}
\begin{split}
    &\E[\min\{1, \widetilde S_{\eta, m, n}(\bbeta^{(1)}, \bbeta^{(2)})\}]\\
    = & \E\big[\min
    \big\{1, \underbrace{\widetilde S_{\eta, m, n}(\bbeta^{(1)}, \bbeta^{(2)})-\mathbb{S}_{\eta, m, n}(\bbeta^{(1)}, \bbeta^{(2)})}_{\mathbb{B}}+\underbrace{\mathbb{S}_{\eta, m, n}(\bbeta^{(1)}, \bbeta^{(2)})}_{\mathbb{A}}\big\}\big]\\
    \leq & \underbrace{\E\left[\left| \widetilde S_{\eta, m, n}(\bbeta^{(1)}, \bbeta^{(2)})-\mathbb{S}_{\eta, m, n}(\bbeta^{(1)}, \bbeta^{(2)})\right|\right]}_{\mathcal{I}} + \underbrace{\E[\min\{1, \mathbb{S}_{\eta, m, n}(\bbeta^{(1)}, \bbeta^{(2)})\}]}_{\text{see formula\ }  (\ref{normal_truncate})} \\
\end{split}
\end{equation}

By the triangle inequality, we can further upper bound the first term $\mathcal{I}$ 
\begin{equation}
\label{target_eq2}
\begin{split}
    &\ \ \ \ \ \E\left[\left|\widetilde S_{\eta, m, n}(\bbeta^{(1)}, \bbeta^{(2)})-\mathbb{S}_{\eta, m, n}(\bbeta^{(1)}, \bbeta^{(2)}\right|\right]\\
    &\leq \underbrace{\left|\E[\widetilde S_{\eta, m, n}(\bbeta^{(1)}, \bbeta^{(2)})]-S(\bbeta^{(1)}, \bbeta^{(2)})\right|}_{\mathcal{I}_1}+\underbrace{\left|S(\bbeta^{(1)}, \bbeta^{(2)})-\E[\mathbb{S}_{\eta, m, n}(\bbeta^{(1)}, \bbeta^{(2)})]\right|}_{\mathcal{I}_2}\\
    &= S(\bbeta^{(1)}, \bbeta^{(2)}) \mathcal{O}\left(\frac{1}{n^2}\right)+S(\bbeta^{(1)}, \bbeta^{(2)}) \mathcal{O}\left(\frac{1}{n^2}\right),
\end{split}
\end{equation}
where $\mathcal{I}_1$ and $\mathcal{I}_2$ follow from the proof of S1 without and with normality assumptions, respectively \citep{Matias19}.

Combining (\ref{target_eq1}) and (\ref{target_eq2}), we have
\begin{equation}
 \E[\min\{1, \widetilde S_{\eta, m, n}(\bbeta^{(1)}, \bbeta^{(2)})\}]\approx \min\Big\{1, S(\bbeta^{(1)}, \bbeta^{(2)}) \left(\mathcal{O}\left(\frac{1}{n^2}\right)+ e^{-\mathcal{O}\left(\frac{m^2\eta}{n}+\frac{1}{n^2}\right)}\right)\Big\}.
\end{equation}

This means that reducing the update period $m$ (more frequent update the of control variable), the learning rate $\eta$ and the batch size $n$ significantly increases $\min\{1, \widetilde S_{\eta, m, n}\}$ on average.\qed
\end{proof}

% The above lemma shows a potential to exponentially increase the number of effective swaps  via variance reduction. Admittedly, the normality assumption may be violated given a small batch. However, the asymptotics shows that the expectation of $\widetilde S_{\eta, m, n}(\bbeta^{(1)}, \bbeta^{(2)})$ yields a bias at most $\mathcal{O}(\frac{1}{n^2})$ according to (S1) in \citet{Matias19}. Therefore, the truncated stochastic swapping rate $\min\{1, \widetilde S_{\eta, m, n}(\bbeta^{(1)}, \bbeta^{(2)})\}$ is still significant larger than the vanilla corrected swapping rate without variance reduction $\min\{1, \widehat S(\bbeta^{(1)}, \bbeta^{(2)})\}$. 

The above lemma shows a potential to exponentially increase the number of effective swaps  via variance reduction under the same intensity $r$. Next, we show the impact of variance reduction in speeding up the exponential convergence of the corresponding continuous-time replica exchange Langevin diffusion.

\begin{theorem}[Exponential convergence]\label{exponential decay}
Under the smoothness
and dissipativity assumptions \ref{assump: lip and alpha beta} and \ref{assump: dissipitive}, the replica exchange Langevin diffusion associated with the variance-reduced stochastic swapping rates $S_{\eta, m, n}(\cd, \cd)=\min\{1, \widetilde S_{\eta, m, n}(\cd, \cd)\}$ converges exponential fast to the invariant distribution $\pi$ given a smaller learning rate $\eta$, a smaller $m$ or a larger batch size $n$:
\begin{equation}
    \mathcal{W}_2(\nu_t,\pi) \leq  D_0 \exp\left\{-t\left(1+\delta_{ S_{\eta, m, n}}\right)/c_{\text{LS}}\right\},
\end{equation}
where $D_0=\sqrt{2c_{\text{LS}}D(\nu_0||\pi)}$, $\delta_{ S_{\eta, m, n}}:=\inf_{t>0}\frac{\cE_{ S_{\eta, m, n}}(\sqrt{\frac{d\n_t}{d\pi}})}{\cE(\sqrt{\frac{d\n_t}{d\pi}})}-1$ is a non-negative constant depending on the truncated stochastic swapping rate $S_{\eta, m, n}(\cd, \cd)$ and increases with a smaller learning rate $\eta$, a shorter period $m$ and a large batch size $n$. $c_{\text{LS}}$ is the standard constant of the log-Sobolev inequality asscoiated with the Dirichlet form for replica exchange Langevin diffusion without swaps.
\end{theorem}{}

\begin{proof} Given a smooth function $f:\mathbb{R}^d\times \mathbb{R}^d\rightarrow \mathbb{R}$, the infinitesimal generator $\cL_{ S_{\eta, m, n}}$ associated with the replica exchange Langevin diffusion with the swapping rate $ S_{\eta, m, n}=\min\{1, \widetilde S_{\eta, m, n}\}$ follows
\begin{equation}
\label{generator_L}
% \small
\begin{split}
    \cL_{S_{\eta, m, n}}f(\bbeta^{(1)}, \bbeta^{(2)})=&-\langle\nabla_{\bbeta^{(1)}}f(\bbeta^{(1)},\bbeta^{(2)}),\nabla L(\bbeta^{(1)})\rangle-\langle \nabla_{\bbeta^{(2)}}f(\bbeta^{(1)},\bbeta^{(2)}),\nabla L(\bbeta^{(2)})\rangle\\
    &
+\tau^{(1)}\Delta_{\bbeta^{(1)}}f(\bbeta^{(1)},\bbeta^{(2)})+\tau^{(2)}\Delta_{\bbeta^{(2)}}f(\bbeta^{(1)},\bbeta^{(2)})\\
&+ rS_{\eta, m, n}(\bbeta^{(1)},\bbeta^{(2)})\cd (f(\bbeta^{(2)},\bbeta^{(1)})-f(\bbeta^{(1)},\bbeta^{(2)})),
\end{split}
\end{equation}
where $\nabla_{\bbeta^{(h)}}$ and $\Delta_{\bbeta^{(h)}}$ are the gradient and the Laplace operators with respect to $\bbeta^{(h)}$, respectively. Next, we model the exponential decay of $\mathcal{W}_2(\nu_t,\pi)$ using the Dirichlet form
\begin{equation}
\label{dirichlet}
    \cE_{S_{\eta, m, n}}(f)=\int \Gamma_{S_{\eta, m, n}}(f)d\pi,
\end{equation}
where $\Gamma_{S_{\eta, m, n}}(f)=\frac{1}{2}\cdot \cL_{S_{\eta, m, n}}(f^2) -f \cL_{S_{\eta, m, n}}(f)$ is the Carr\'{e} du Champ operator. In particular for the first term $\frac{1}{2}\cL_{S_{\eta, m, n}}(f^2)$, we have
\begin{equation*}
% \small
\label{half_carre}
\begin{split}
    &\ \ \ \frac{1}{2}\cL_{S_{\eta, m, n}}(f(\bbeta^{(1)}, \bbeta^{(2)})^2)\\
    =&-\langle f(\bbeta^{(1)}, \bbeta^{(2)})\nabla_{\bbeta^{(1)}} f(\bbeta^{(1)}, \bbeta^{(2)}) , \nabla_{\bbeta^{(1)}} L(\bbeta^{(1)})\rangle+\tau^{(1)} \|\nabla_{\bbeta^{(1)}}f(\bbeta^{(1)}, \bbeta^{(2)})\|^2 \\
    & \ \ \ \ \ \ \ \ \ \ \ \ \ \ \ \ \ \ \ \ \ \ \ \ \ \ \ \ \ \ \ \ \ \ \ \ \ \ \ \ \ \ \ \ \ \ \ \ \ \ \ \ \ \ \ \  + \tau^{(1)} f(\bbeta^{(1)}, \bbeta^{(2)}) \Delta_{\bbeta^{(1)}}f(\bbeta^{(1)}, \bbeta^{(2)})\\
        &-\langle f(\bbeta^{(1)}, \bbeta^{(2)})\nabla_{\bbeta^{(2)}} f(\bbeta^{(1)}, \bbeta^{(2)}) , \nabla_{\bbeta^{(2)}} L(\bbeta^{(2)})\rangle+\tau^{(2)} \|\nabla_{\bbeta^{(2)}}f(\bbeta^{(1)}, \bbeta^{(2)})\|^2 \\
        & \ \ \ \ \ \ \ \ \ \ \ \ \ \ \ \ \ \ \ \ \ \ \ \ \ \ \ \ \ \ \ \ \ \ \ \ \ \ \ \ \ \ \ \ \ \ \ \ \ \ \ \ \ \ \ \ + \tau^{(2)} f(\bbeta^{(1)}, \bbeta^{(2)}) \Delta_{\bbeta^{(2)}}f(\bbeta^{(1)}, \bbeta^{(2)})\\
        &+\frac{r}{2}S_{\eta, m, n}(\bbeta^{(1)}, \bbeta^{(2)}) (f^2(\bbeta^{(2)},\bbeta^{(1)})-f^2(\bbeta^{(1)},\bbeta^{(2)})).
\end{split}
\end{equation*}

Combining the definition of the Carr\'{e} du Champ operator, (\ref{generator_L}) and (\ref{half_carre}), we have
\begin{equation}
\label{carre_du}
\small
    \begin{split}
        &\Gamma_{S_{\eta, m, n}}(f(\bbeta^{(1)}, \bbeta^{(2)}))\\
        =&\frac{1}{2}\mathcal{L}_{S_{\eta, m, n}}(f^2(\bbeta^{(1)}, \bbeta^{(2)}))-f(\bbeta^{(1)}, \bbeta^{(2)})\mathcal{L}_{S_{\eta, m, n}}(f(\bbeta^{(1)}, \bbeta^{(2)}))\\
        % =&\Big\{-\langle f(\bbeta^{(1)}, \bbeta^{(2)})\nabla_{\bbeta^{(1)}} f(\bbeta^{(1)}, \bbeta^{(2)}) , \nabla_{\bbeta^{(1)}} L(\bbeta^{(1)})\rangle+\tau^{(1)} \|\nabla_{\bbeta^{(1)}}f(\bbeta^{(1)}, \bbeta^{(2)})\|^2 + \tau^{(1)} f(\bbeta^{(1)}, \bbeta^{(2)}) \Delta_{\bbeta^{(1)}}f(\bbeta^{(1)}, \bbeta^{(2)})\\
        % &-\langle f(\bbeta^{(1)}, \bbeta^{(2)})\nabla_{\bbeta^{(2)}} f(\bbeta^{(1)}, \bbeta^{(2)}) , \nabla_{\bbeta^{(2)}} L(\bbeta^{(2)})\rangle+\tau^{(2)} \|\nabla_{\bbeta^{(2)}}f(\bbeta^{(1)}, \bbeta^{(2)})\|^2 + \tau^{(2)} f(\bbeta^{(1)}, \bbeta^{(2)}) \Delta_{\bbeta^{(2)}}f(\bbeta^{(1)}, \bbeta^{(2)})\\
        % &+\frac{1}{2}S_{\eta, m, n}(\bbeta^{(1)}, \bbeta^{(2)}) (f^2(\bbeta^{(2)},\bbeta^{(1)})-f^2(\bbeta^{(1)},\bbeta^{(2)}))\Big\}-f(\bbeta^{(1)}, \bbeta^{(2)})\mathcal{L}_{S_{\eta, m, n}}(f(\bbeta^{(1)}, \bbeta^{(2)}))\\
        =&\tau^{(1)} \|\nabla_{\bbeta^{(1)}}f(\bbeta^{(1)}, \bbeta^{(2)})\|^2+\tau^{(2)} \|\nabla_{\bbeta^{(2)}}f(\bbeta^{(1)}, \bbeta^{(2)})\|^2\\
        &\ \ \ \ \ \ +\frac{r}{2}{S_{\eta, m, n}}(\bbeta^{(1)}, \bbeta^{(2)}) (f(\bbeta^{(2)},\bbeta^{(1)})-f(\bbeta^{(1)},\bbeta^{(2)}))^2.
    \end{split}
\end{equation}

Plugging (\ref{carre_du}) into (\ref{dirichlet}), the Dirichlet form associated with operator $\cL_{S_{\eta, m, n}}$ follows
\begin{equation}\label{dirichlet swap}
\small
\begin{split}
    \cE_{S_{\eta, m, n}}(f)=&\underbrace{\int \Big(\tau^{(1)}\|\nabla_{\bbeta^{(1)}}f(\bbeta^{(1)}, \bbeta^{(2)})\|^2+\tau^{(2)}\|\nabla_{\bbeta^{(2)}}f(\bbeta^{(1)}, \bbeta^{(2)})\|^2 \Big)d\pi(\bbeta^{(1)},\bbeta^{(2)})}_{\text{vanilla term } \cE(f)}\\
    &\ +\underbrace{\frac{r}{2}\int S_{\eta, m, n}(\bbeta^{(1)},\bbeta^{(2)})\cd (f(\bbeta^{(2)},\bbeta^{(1)})-f(\bbeta^{(1)},\bbeta^{(2)}))^2d\pi(\bbeta^{(1)},\bbeta^{(2)})}_{\text{acceleration term}},
\end{split}
\end{equation}
where $f$ corresponds to $\frac{d\nu_t}{d\pi(\bbeta^{(1)}, \bbeta^{(2)})}$. Under the asymmetry conditions of $\frac{\nu_t}{\pi(\bbeta_1, \bbeta^{(2)})}$ and $S_{\eta, m, n}>0$, the acceleration term of the Dirichlet form is strictly positive and linearly dependent on the swapping rate $S_{\eta, m, n}$. Therefore, $\cE_{S_{\eta, m, n}}(f)$ becomes significantly larger as the swapping rate $S_{\eta, m, n}$ increases significantly. According to Lemma 5 \citep{deng2020}, there exists a constant $\delta_{S_{\eta, m, n}}=\inf_{t>0}\frac{\cE_{S_{\eta, m, n}}(\sqrt{\frac{d\n_t}{d\pi}})}{\cE(\sqrt{\frac{d\n_t}{d\pi}})}-1$ depending on $S_{\eta, m, n}$ that satisfies the following log-Sobolev inequality for the unique invariant measure $\pi$  associated with variance-reduced replica exchange Langevin diffusion $\{\bbeta_t\}_{t\ge 0}$
\begin{equation*}
    D(\n_t||\pi)\le 2 \frac{c_{\text{LS}}}{1+\delta_{S_{\eta, m, n}}}\cE_{S_{\eta, m, n}}(\sqrt{\frac{d\n_t}{d\pi}}),
\end{equation*}
where $\delta_{S_{\eta, m, n}}$ increases rapidly with the swapping rate $S_{\eta, m, n}$.
By virtue of the exponential decay of entropy \citep{Bakry2014}, we have
\begin{equation*}
    D(\nu_t||\pi)\leq D(\nu_0||\pi) e^{-2t(1+\delta_{S_{\eta, m, n}})/c_{\text{LS}}},%\leq D(\m_0||\pi) e^{-2k\eta(1+\delta_{S_{\eta, m, n}})/c_{\text{LS}}},
\end{equation*}
%where $t\in[k\eta,(k+1)\eta)$ for some $k$ and 
where $c_{\text{LS}}$ is the standard constant of the log-Sobolev inequality asscoiated with the Dirichlet form for replica exchange Langevin diffusion without swaps (Lemma 4 as in \citet{deng2020}). Next, we upper bound $\mathcal{W}_2(\nu_t,\pi)$ by the
Otto-Villani theorem \citep{Bakry2014}
\begin{equation*}
    \mathcal{W}_2(\nu_t,\pi) \leq \sqrt{2 c_{\text{LS}} D(\nu_t||\pi)}\leq \sqrt{2c_{\text{LS}}D(\m_0||\pi)} e^{-t\left(1+\delta_{S_{\eta, m, n}}\right)/c_{\text{LS}}},
\end{equation*}
where $\delta_{S_{\eta, m, n}}>0$ depends on the learning rate $\eta$, the period $m$ and the batch size $n$. \qed

% Thus, we obtain the following log-Sobolev inequality  and its corresponding Dirichlet form $\cE_{S_{\eta, m, n}}(\cd)$. In particular, the LSI constant $ \frac{c_{\text{LS}}}{1+\delta_{S_{\eta, m, n}}}$ in replica exchange Langevin diffusion with swapping rate $S_{\eta, m, n}(\cd, \cd)>0$ is strictly smaller than the LSI constant $c_{\text{LS}}$ in the replica exchange Langevin diffusion with swapping rate $S_{\eta, m, n}(\cd, \cd)=0$. 

\end{proof}

In the above analysis, we have established the relation that $\delta_{S_{\eta, m, n}}=\inf_{t>0}\frac{\cE_{S_{\eta, m, n}}(\sqrt{\frac{d\n_t}{d\pi}})}{\cE(\sqrt{\frac{d\n_t}{d\pi}})}-1$ depending on $S_{\eta, m, n}$ may increase significantly with a smaller learning rate $\eta$, a shorter period $m$ and a large batch size $n$. For more quantitative study on how large $\delta_{S_{\eta, m, n}}$ is on related problems, we refer interested readers to the study of spectral gaps in \citet{Holden18, jingdong, Futoshi2020}.

\section{Discretization error}
\label{discre_error}
\setcounter{lemma}{0}
%We shall keep the convention as below, let us fix 
Consider a complete filtered probability space $ (\Omega, \mathcal F,\mathbb F =(\mathcal F_t)_{t \in [0,T]}, \mathbb P)$ which supports all the random subjects considered in the sequel. With a little abuse usage of notation, the probability measure $\mathbb P$ (component wise if $\mathbb P$ is joint probability measure with mutually independent components) would always denote the Wiener measure under which the process $(\bW_t)_{0\le t\le T}$ is a $\mathbb P$-Brownian motion. To be precise, in what follows, we shall denote $\mathbb P:=\mathbb P^{\bW}\times \mathbf N$, where $\mathbb P^{\bW}$ is the infinite dimensional Wiener measure and $\mathbf N$ is the Poisson measure independent of $\mathbb P^{\bW}$ and has some constant jump intensity. In our general framework below, the jump process $\alpha$ is introduced by swapping the diffusion matrix of the two Langevin dynamics and the jump intensity is defined through the swapping probability in the following sense, which ensures the independence of $\mathbb P^{\bW}$ and $\mathbf N^S$ in each time interval $[i\eta,(i+1)\eta]$, for $i\in\mathbb N^+$. The precise definition of the \textbf{Replica exchange Langevin diffusion (reLD)} is given as below. For any fixed learning rate $\eta>0$, we define
\begin{equation}  
\small
\label{reLD}
\left\{  
             \begin{array}{lr}  
             d\bbeta_t=-\nabla G(\bbeta_t)dt+\Si(\alpha_t)d\bW_t,  \\  
              & \\
              \mathbb{P}\left(\alpha(t)=j|\alpha(t-dt)=l,  \bbeta(\lfloor t/\eta \rfloor \eta)=\bbeta\right)=rS(\bbeta) \eta \mathbf{1}_{\{t=\lfloor t/\eta \rfloor \eta\}} +o(dt),~~\text{for}~~  l\neq j,
            %  \mathbb{P}\left(\alpha(\lfloor (t+\eta)/\eta\rfloor\eta)=j|\alpha(\lfloor t/\eta \rfloor \eta)=i, \bbeta(\lfloor t/\eta \rfloor \eta)=\bbeta \right)=S(\bbeta)\eta+o(\eta),~~\text{for}~~  i\neq j,
             \end{array}  
\right.  
\end{equation} 
% when $t=\eta$, $\mathbb{P}\left(\alpha(2\eta)=j|\alpha(\eta)=i, \bbeta(\eta)=\bbeta \right)=S(\bbeta)\eta+o(\eta)$
% when $t=0.5\eta$, $\mathbb{P}\left(\alpha(\eta)=j|\alpha(0)=i, \bbeta(0)=\bbeta \right)=S(\bbeta)\eta+o(\eta)$
% \textcolor{green}{$\mathbb{P}\left(\alpha(\lfloor (0.5+dt)\rfloor)=j|\alpha(\lfloor 0.5 \rfloor)=i, \bbeta(\lfloor t/\eta \rfloor \eta)=\bbeta \right)=S(\bbeta)dt+o(dt)$}
% \textcolor{red}{$\mathbb{P}\left(\alpha(t+dt)=j|\alpha(t)=i, , \bbeta(\lfloor t/\eta \rfloor \eta)=\bbeta \right)=S(\bbeta) \eta \delta_{N\eta}(t) +o(dt)$}
% \textcolor{blue}{\textbf{Poisson process in continuous time}}
where $\nabla G(\bbeta):=\begin{pmatrix}{}
\nabla L(\bbeta^{(1)})\\
\nabla L(\bbeta^{(2)})
\end{pmatrix}$, and $\mathbf 1_{t=\lfloor t/\eta\rfloor\eta}$ is the indicator function, i.e. for every $t=i\eta$ with $i\in \mathbb N^+$, given $\bbeta(i\eta)=\bbeta$, we have  $\mathbb{P}\left(\alpha(t)=j|\alpha(t-dt)=l\right)=rS(\bbeta)\eta$, where $S(\bbeta)$ is defined as $\min\{1, S(\bbeta^{(1)}, \bbeta^{(2)})\}$ and $S(\bbeta^{(1)}, \bbeta^{(2)})$ is defined in (\ref{S_exact}). In this case, 
the Markov Chain $\alpha(t)$ is a constant on the time interval $[\lfloor t/\eta\rfloor\eta,\lfloor t/\eta\rfloor\eta+\eta)$ with some state in the finite-state space $\{0,1\}$ and the generator matrix $Q$ follows
\begin{equation*}
    Q=\begin{pmatrix}
-rS(\bbeta) \eta \delta(t-\lfloor t/\eta \rfloor \eta)&rS(\bbeta) \eta \delta(t-\lfloor t/\eta \rfloor \eta)\\
rS(\bbeta) \eta \delta(t-\lfloor t/\eta \rfloor \eta)&-rS(\bbeta) \eta \delta(t-\lfloor t/\eta \rfloor \eta)
\end{pmatrix},
\end{equation*}
where $\delta(\cdot)$ is a Dirac delta function. The diffusion matrix $\Si(\alpha_t)$
is thus defined as $(\Si(0), \Si(1) ):=\left\{\begin{pmatrix}{}
\sqrt{2\tau^{(1)}}\mathbf I_d&0\\
0&\sqrt{2\tau^{(2)}}\mathbf I_d
\end{pmatrix}, \begin{pmatrix}{}
\sqrt{2\tau^{(2)}}\mathbf I_d&0\\
0&\sqrt{2\tau^{(1)}}\mathbf I_d
\end{pmatrix}\right\}$ . From our definition and  following \cite{yin_zhu_10}[Section 2.7], the generator matrix $Q$ will depend on the initial value at each time interval $[i\eta,(i+1)\eta)$. The distribution of process $(\bbeta_t)_{0\le t\le T}$ is denoted as $\nu_T:=\mathbb P^{G}\times \mathbf N^{S}$ which is absolutely continuous with respect to the reference measure $\mathbb P:=\mathbb P^{\bW}\times \mathbf N$, under which $\bW$ is Brownian motion and $\alpha(\cdot)$ is a Poisson process with some constant jump intensity. This fact follows from the result in \cite{Gikhman}[VII, Section 6, Theorem 2] and \cite{yin_zhu_10}[Section 2.5, formula (2.40)]. The motivation of only considering the positive swapping rate in $i\eta$, for $i\in\mathbb N^+$, and zero elsewhere is due to our construction of the discretized process $\widetilde\bbeta$ as shown below (see \eqref{reSGLD}). A simple illustration of the idea can be seen from the auxiliary process construction in \cite{yin_zhu_10}[Section 2.5], following which we want to make sure the stopping time of $\bbeta$ and $\widetilde\bbeta$ happening at the same time. Otherwise, it is unlikely (and also unreasonable) to derive the Radon-Nikodym derivative  of the two process $\bbeta$ and $\widetilde\bbeta$.
% \textcolor{red}{we need to explain the reason why we only consider positive swapping rate in k$\eta$. The main reason can be that if there is positive swapping rate in 1.5$\eta$ (continuous case) but 0 swapping rate in (discrete case), then the Radon Nikodim derivative doesn't exist.}
Thus, we should think of the process is concatenated on the time interval $[i\eta,(i+1)\eta)$ up to time horizon $T$. Similarly, we consider the following \textbf{Replica exchange stochastic gradient Langevin diffusion}, for the same learning rate $\eta>0$ as above, we have
\begin{equation}  
\small
\label{reSGLD}
\left\{  
             \begin{array}{lr}  
             d\widetilde \bbeta_t^{\eta}=-\nabla\widetilde G(\widetilde \bbeta^{\eta}_{\lfloor t/\eta \rfloor \eta})dt+\Si(\widetilde \alpha_{\lfloor t/\eta \rfloor \eta})d\bW_t,  \\ 
              & \\
             \mathbb{P}\left(\widetilde \alpha(t)=j|\widetilde\alpha(t-dt)=l, \widetilde\bbeta(\lfloor t/\eta \rfloor \eta)=\widetilde\bbeta \right)=r\widetilde S(\widetilde \bbeta)\eta\mathbf 1_{\{t=\lfloor t/\eta\rfloor\eta\}}+o(dt),~~\text{for}~~  l\neq j,
             \end{array}  
\right.  
\end{equation} 
where $\nabla \widetilde G(\bbeta):=\begin{pmatrix}{}
\nabla \widetilde L(\bbeta^{(1)})\\
\nabla \widetilde L(\bbeta^{(2)})
\end{pmatrix}$ and $\widetilde S(\widetilde\bbeta) = \min\{1, \widetilde S_{\eta, m, n}(\widetilde\bbeta^{(1)},\widetilde\bbeta^{(2)})\}$ and $\widetilde S_{\eta, m, n}(\widetilde\bbeta^{(1)},\widetilde\bbeta^{(2)})$ is shown in (\ref{vr_s}). The distribution of process $(\widetilde \bbeta_t)_{0\le t\le T}$ is denoted as $\mu_T:=\mathbb P^{\widetilde G}\times \mathbf N^{\widetilde S}$, where $\widetilde \alpha$ is a Poisson process with jump intensity $r\widetilde S(\widetilde \bbeta)\eta\delta(t-\lfloor t/\eta\rfloor\eta)$ on the time interval $[\lfloor t/\eta\rfloor\eta,\lfloor t/\eta\rfloor\eta+\eta)$. Note that $\bbeta$ and $\widetilde\bbeta$ are defined by using the same $\mathbb P$-Brownian motion $\bW$, but with two different jump intensity on the time interval $[\lfloor t/\eta\rfloor\eta,\lfloor t/\eta\rfloor\eta+\eta)$. Notice that, if there is no jump, the construction of $\widetilde\bbeta$ based on $\bbeta$ follows from the fact that they share the same marginal distributions as shown in \cite{gyongy1986mimicking}, where one can find the details in \cite{Maxim17}. Given the jump process $\alpha$ and $\widetilde\alpha$ introduced into the dynamics of $\bbeta$ and $\widetilde\bbeta$, the construction is more complicated. Thanks to \cite{bentata2009mimicking}, we can carry on the similar construction in our current setting. We then introduce the following Radon-Nikodym density for $d\nu_T/d\mu_T$. In the current setting, the change of measure can be seen as the combination of two drift-diffusion process and two jump process simultaneously. We first introduce some notation. For each vector $A\in\mathbb R^n$, we denote $\|A\|^2:=A^*A$. Furthermore, we introduce a sequence of stopping time based on our definition of process $\bbeta$ and $\widetilde\bbeta$. For $j\in \mathbb N^+$,  we denote $\zeta_j's$ as a  stopping times defined by $\zeta_{j+1}:=\inf\{t>\zeta_j:\alpha(t)\neq \alpha(\zeta_j)\}$ and $N(T)=\max\{n\in \mathbb N:\zeta_n\leq T \}$. It is easy to see that for any  stopping time $\zeta_i$, there exists $l\in \mathbb N^+$ such that $\zeta_j=l\eta$. Similarly, we have the stopping time for the process $\widetilde\bbeta$ by  $\widetilde\zeta_{j+1}:=\inf\{t>\widetilde\zeta_j:\widetilde\alpha(t)\neq\widetilde \alpha(\zeta_j)\}$ and $\tilde\alpha(t)$ follows the same trajectory of $\alpha(t)$. To serve the purpose of our analysis, one should think of the process $\bbeta$ as the auxiliary process to the process $\widetilde\bbeta$, see similar constructions in \cite{yin_zhu_10}[Section 2.5, formula (2.39)]. The difference is that both of our process $\bbeta$ and $\widetilde\bbeta$ are associated with jump process jumping at time $i\eta$, for some integer $i\in\mathbb N^+$, instead of jumping at any continuous time. We combine approximation method from \cite{yin_zhu_10}[Section 2.7] for non-constant generator matrix $Q$ and the density representation for Markov process in \cite{Gikhman}[VII, Section 6, Teorem 2] to get the following
\begin{lemma}\label{lemma: RN density}
Let $\{\zeta_j|j\in\{0,1,\cdots, N(T)\} \}$ be a sequence of stopping time defined by $\alpha$. Let $k\in\mathbb N^+$ be an fixed integer such that $k\eta\le T\le (k+1)\eta$.
For each fixed learning rate $\eta>0$ and for any $\varepsilon>0$, the Radon-Nikodym derivative of $ {d\mu_T}/{d\nu_T}$ is given as below, 
\begin{align*}
    \frac{d\mu_T}{d\nu_T}=& \exp\Big(\sum_{j=0}^{N(T)}\int_{\zeta_{j}}^{\zeta_{j+1}\wedge T} \Big[ \Si^{-1}(\widetilde \alpha(\zeta_j)) \nabla \widetilde G(\bbeta_{t})-\Si^{-1}(\alpha(\zeta_j))\nabla G(\bbeta_t)\Big] d\bW_t^G\Big. \\
    \Big.
    &\quad\quad\quad-\frac{1}{2}\sum_{j=0}^{N(T)}\int_{\zeta_{j}}^{\zeta_{j+1}\wedge T}\Big\| \Si^{-1}(\widetilde \alpha(\zeta_j)) \nabla \widetilde G(\bbeta_{t})-\Si^{-1}(\alpha(\zeta_j))\nabla G(\bbeta_t)\Big\|^2dt\Big) \\
    &\times  \exp\left\{ -\sum_{j=0}^{N(T))}\int_{\zeta_j}^{\zeta_{j+1}\wedge T-\varepsilon}r\delta(t-\lfloor t/\eta\rfloor \eta) [\widetilde S(\widetilde\bbeta_{\lfloor t/\eta\rfloor\eta})-S(\bbeta _{\lfloor t/\eta\rfloor\eta})]\eta dt \right\}\times \Pi_{j=0}^{N(T)}\frac{\widetilde S(\widetilde\bbeta_{ \zeta_j})}{S(\bbeta_{\zeta_j})}.
\end{align*}
\end{lemma}
\begin{proof} Recall that $\zeta_j$ is stopping time defined by $\alpha$ (same as defined by $\widetilde \alpha$), i.e. $\zeta_{j+1}:=\inf\{t>\zeta_j:\alpha(t)\neq \alpha(\zeta_j)\}$, for $j=0,1,\cdots,N(T)$, and for each $\zeta_j$, there exists $l\in \{0,1,\cdots,k\}$ such that $\zeta_j=l\eta$.
% For any interval $[\zeta_j,\zeta_{j+1})$, we can chop it into the sum of sub-intervals of length $\eta$.
We now follow \cite{Gikhman}[VII, Section 6, Theorem 2] to derive the Radon-Nikodym density for $d\mu_T/d\nu_T$. In this case, if the generator matrix $Q$ is constant, i.e. the jump intensity is constant, we can follow the similar construction from \cite{yin_zhu_10}[Formula (2.40)], see also \cite{Eizenberg90}[Formula(3.13)]. Next, we adjust our setting to the case that we can treat our generator matrix as constant on each time interval $[\zeta_j,\zeta_{j+1})$, then the existing results apply to our case for the density with respect to the Poisson measure (jump process $\alpha$ and $\widetilde\alpha$), i.e. $d\mathbf N^{S}/d\mathbf N^{\widetilde S}$. Furthermore, once the generator matrix $Q$ is constant, then the measure $\mathbb P^G$ ( or $\mathbb P^{\widetilde G}$) is independent to $\mathbf N^S$ (or $\mathbf N^{\widetilde S}$). We show the following steps to give a clear outline of our proof.

\noindent\textbf{Step 1:}
For each stopping time interval $[\zeta_j,\zeta_{j+1})$, no jump would occur after the initial point at time $\zeta_j$ and the diffusion matrix $\Si$ and $\widetilde\Si$ keep the same, thus we can apply the generalized Girsanov theorem to get the Randon-Nikodym derivative for $d\mathbb P^G/d\mathbb P^{\widetilde G}$.

\noindent\textbf{Step 2:}
In order to combine the the two density of $d\mathbf N^{S}/d\mathbf N^{\widetilde S}$ and $d\mathbb P^G/d\mathbb P^{\widetilde G}$, we need the independent property of the two measures on the same time interval, then we directly get the density following from \cite{Gikhman}[VII, Section 6, Theorem 2].
Different from the work mentioned above, we will first write all the density on each time interval $[i\eta,(i+1)\eta)$ to incorporate the independent requirement mentioned above. Notice that the relative change of density for $d\mathbf N^{S}/d\mathbf N^{\widetilde S}$ would only depends on the left end point, since the jump intensity would change its values at the  initial value of interval $[i\eta,(i+1)\eta)$, which is a standard idea to deal with generator matrix $Q$ depending on the initial value instead of a constant matrix case. (See \cite{yin_zhu_10}[Section2.7] for similar treatments). 

\noindent\textbf{Step 3:} In general, the stopping time interval could contain several time interval with length $\eta$, however the jump intensity should only depend on the left end point for each time interval $[i\eta,(i+1)\eta)$. 
% In particular, for all the interval $[i\eta,(i+1)\eta)$, for $i=0,\cdots, k$, we should always keep in mind that the jump intensity of $\alpha$ and $\widetilde \alpha$ only depend on the initial value $(\bbeta_i,\widetilde\bbeta_i)$ on each interval and vanish at any other time.
Based on the above set up, we now derive the Radon-Nikodym derivative. First notice that, on each period $[\zeta_j,\zeta_{j+1})$, the matrix $\Si$ is fixed and is evaluated at $\Si(\alpha(\zeta_j))$, which is the same for $\Si(\widetilde\alpha(\zeta_j))$. In particular, $\Si(\alpha(\zeta_j))=\Si(\widetilde\alpha(\zeta_j))$ is a constant diagonal matrix. According to our definition  $d\nu_T = d\mathbb P^{G}\times d \mathbf N^S$ and $d\mu_T=d\mathbb P^{\widetilde G}\times d\mathbf N^{\widetilde S}$, we write the Radon-Nikodym derivative on each of the time interval $[i\eta,(i+1)\eta)$ and concatenate them together. We consider the swapping of the diffusion matrix first where a similar construction can be found in \cite{yin_zhu_10}[Formula (2.40)], we get the following Radon-Nikodym derivative, for any $\varepsilon>0$, 
\begin{align}
\small
 \frac{d\mathbf N^{\widetilde S}}{d\mathbf N^{S}}&=\exp\Big\{ -\sum_{j=0}^{N(T)}\int_{j\eta}^{(j+1)\eta\wedge T-\varepsilon} r\delta(t-\lfloor t/\eta\rfloor\eta)(\widetilde S(\widetilde\bbeta_{\bbeta_{\lfloor t/\eta\rfloor\eta}})-S(\bbeta _{\bbeta_{\lfloor t/\eta\rfloor\eta}}))\eta dt \Big\}\\
 &\times \Pi_{j=0}^{N(T)}\frac{\widetilde S(\widetilde\bbeta_{\zeta_j})}{S(\bbeta_{\zeta_j})}.\nonumber
% &=\exp\left\{ -\sum_{i=0}^{k}(\widetilde S(\widetilde\bbeta_{i\eta})-S(\bbeta _{i\eta}))\eta \right\}\times \Pi_{i=0}^{k}\left(\frac{\widetilde S(\widetilde\bbeta_{i\eta})}{S(\bbeta_{i\eta})}\right).
\end{align}
Next, we show the density for $d\mathbb P^G/d\mathbb P^{\widetilde G}$ as below. On each interval $[\zeta_j,\zeta_{j+1})$, given initial value $(\bbeta_j,\widetilde \bbeta_j)$, the matrix $\Si(\alpha(\zeta_j))$ and $\Si(\widetilde\alpha(\zeta_j))$ are always the same, since no jump would happen. In particular, in this continuous case the integral on $[\zeta_j,\zeta_{j+1})$ and $[\zeta_j,\zeta_{j+1}]$ are the same.  Thus we have the following Radon-Nikodym derivative    
\begin{align}%\label{density G and G tilde}
\frac{d\mathbb P^{\widetilde G}}{d\mathbb P^G}=& \exp\Big(\sum_{j=0}^{N(T)}\int_{\zeta_{j}}^{\zeta_{j+1}\wedge T} \Big[ \Si^{-1}(\widetilde \alpha(\zeta_j)) \nabla \widetilde G(\bbeta_{t})-\Si^{-1}(\alpha(\zeta_j))\nabla G(\bbeta_t)\Big] d\bW_t^G\Big. \nonumber\\
    \Big. &\quad\quad\quad-\frac{1}{2}\sum_{j=0}^{N(T)}\int_{\zeta_{j}}^{\zeta_{j+1}\wedge T}\Big\| \Si^{-1}(\widetilde \alpha(\zeta_j)) \nabla \widetilde G(\bbeta_{t})-\Si^{-1}(\alpha(\zeta_j))\nabla G(\bbeta_t)\Big\|^2dt\Big).
\end{align}
Notice that matrix $\Si$ is diagonal square matrix, thus we have $\Si=\Si^*$. Recall that $\bW$ is a $\mathbb P$-Brownian motion, assuming there is no jump in the dynamic for $\bbeta$, then according to the Girsanov theorem (see an example in Theorem 8.6.6 and Example 8.6.9 \citep{oksendal2003stochastic}) with Radon-Nikodym derivative ${d\mathbb P^{G}}/{d\mathbb P}$, we have the  $\mathbb P^G$-Brownian motion, denoted as $\bW^G$, which follows
\begin{align}\label{G brownian motion}
\bW_t^{G}:=\bW_t+\int_0^t\Si^{-1}(\alpha_s)(\nabla G(\bbeta_s))ds.
\end{align}
This fact holds true on each of the time interval $[\zeta_j,\zeta_{j+1}]$. 
Multiplying the two density $d\mathbb P^{G}/\mathbb P^{\widetilde G}$ and $d\mathbf N^S/d\mathbf N^{\widetilde S}$, we complete the proof.
\end{proof} 
\begin{remark}
Notice that, if we keep the constant diffusion matrix without jump, then the Randon-Nikodym derivative $d\mu_T/d\nu_T$ has been used in the stochastic gradient descent setting, for example \cite{Maxim17}. However, the notation of the Brownian motion has been used freely, we try to make it consistent in the current setting. Namely, for constant diffusion matrix $\Si$, we have 
\begin{align}
\frac{d\mathbb P^{\widetilde G}}{d\mathbb P^G}=& \exp\Big(\int_0^T \Big[ \Si^{-1} \nabla \widetilde G(\widetilde\bbeta_s)-\Si^{-1}\nabla G(\bbeta_s)\Big] d\bW_s^G\Big. \nonumber\\
    \Big. &-\frac{1}{2}\int_{0}^T\Big\| \Si^{-1} \nabla \widetilde G(\widetilde\bbeta_s)-\Si^{-1}\nabla G(\bbeta_s)\Big\|^2ds\Big),
\end{align}
where $\bW^G$ is a $\mathbb P^G$-Brownian motion as shown in \eqref{G brownian motion}, not a $\mathbb P$-Brownian motion $\bW$.
\end{remark}
\begin{remark}
The density $\frac{d\mu_T}{d\nu_T}$ that we derived above is so far the best we can do. If one would like to use the continuous time control $\alpha(t)$ with continuous jump intensity $S(\bbeta(t))$ instead of jumping at the initial point with a fixed rate, %then the Poisson measure does not seem to be independent of the the filtration $\mathcal F_t^{\bW}$ introduced by the $\mathbb P$-Brownain motion. 
then we can not even write the Randon-Nikodym derivative anymore, since $\alpha(t)$ and $\widetilde\alpha(t)$ will define different stopping time, i.e. jump at different time and $\mu_T$ is not absolutely continuous with respect to $\nu_T$. % we can not separate the noise from the diffusion and the jump completely. To the best of our knowledge, we can not deal with the case when the generator matrix $Q$ depends on the whole process of $\{\bbeta(t)\}_{0\le t\le T}$ instead of the initial value or the constant case.
\end{remark}
Based on the above lemma, we further get the following estimates.
\begin{lemma}
\label{numerical_error_KL}
Given a large enough batch size $n$ or a small enough $m$ and $\eta$, we have the bound of the KL divergence of $D_{KL}(\mu_{T}|\nu_{T})$ as below,
\[
D_{KL}(\mu_{T}|\nu_{T})\le (\Phi_0+\Phi_1\eta)k\eta +N(T)\Phi_2,
\]
with 
\begin{align*}
    % \Phi_0&=r\Big(\min\left\{1, \left|\frac{1}{\tau^{(1)}}-\frac{1}{\tau^{(2)}}\right|  \widetilde \sigma+\mathcal{O}(\widetilde \sigma^2) \right\}+\frac{\delta\Phi^2}{4\tau^{(1)}} \Big),\\
    \Phi_0&=\mathcal{O}\left(\frac{m}{\sqrt{n}}\sqrt{\eta}d\right)+\frac{r\delta\Phi^2}{4\tau^{(1)}},\\
    \Phi_1&= \Big(C^2d\frac{\tau^{(2)}}{\tau^{(1)}}+\frac{C^2\delta kd}{2\tau^{(1)}}[\tau^{(1)}+\tau^{(2)}] \Big),\\
%   \Phi_2&=   \left(\frac{1}{\tau^{(1)}}-\frac{1}{\tau^{(2)}}\right)^2 \frac{ \widetilde \sigma^2}{2}+  \left|\frac{1}{\tau^{(1)}}-\frac{1}{\tau^{(2)}}\right|  \widetilde \sigma.
    \Phi_2&=   \mathcal{O}\left(\frac{m}{\sqrt{n}}\sqrt{\eta} d\right).
\end{align*}
\end{lemma}
\begin{proof}
By the very definition of the KL-divergence, we have 
% and the fact that $\bW^G$ is a $\mathbb P^G$-Brownian motion, i.e. applying martingale property, and the independence of $\mathbb P^G$ and $\mathbf N^S$, we have, $N(T)$ as a realization, denoted as $N(T)=n$, 
\begin{align*}
 D_{KL}(\mu_{T}|\nu_{T})=&-\int d \nu_T\log \frac{d\mu_T}{d\nu_T}  \\
 =&-\mathbb E_{\nu_T} \Big[\log (d\mu_T/d\nu_T)\Big|   (\bbeta,\widetilde\bbeta)=(\beta,\widetilde\beta) \Big].
\end{align*}
We shall keep the convention below and denote $\mathbb E_{\nu_T,\beta}=\mathbb E_{\nu_T}[\cdot| (\bbeta,\widetilde\bbeta)=(\beta,\widetilde\beta)]$, where $\beta=(\beta^{(1)},\beta^{(2)})\in \mathbb R^{2d}$ and $\widetilde\beta=(\widetilde\beta^{(1)},\widetilde\beta^{(2)})\in \mathbb R^{2d}$ denotes the values at each time $i\eta$, $i=0,1,\cdots,k$. Plugging Lemma \ref{lemma: RN density} in the above equation and we unify the notation by using time intervals of the type $[i\eta,(i+1)\eta]$. To be precise, we get 
\begin{align}%\label{density G and G tilde}
\frac{d\mathbb P^{\widetilde G}}{d\mathbb P^G}=& \exp\Big(\sum_{i=0}^{k-1}\int_{i\eta}^{(i+1)\eta} \Big[ \Si^{-1}(\widetilde \alpha(i\eta)) \nabla \widetilde G(\bbeta_{t})-\Si^{-1}(\alpha(i\eta))\nabla G(\bbeta_t)\Big] d\bW_t^G\Big. \nonumber\\
&\quad\quad\quad +\int_{k\eta}^{T} \Big[ \Si^{-1}(\widetilde \alpha(k\eta)) \nabla \widetilde G(\bbeta_{t})-\Si^{-1}(\alpha(k\eta))\nabla G(\bbeta_t)\Big] d\bW_t^G\nonumber \\
    \Big.
&\quad\quad\quad-\frac{1}{2}\sum_{i=0}^{k-1}\int_{i\eta}^{(k+1)\eta}\Big\| \Si^{-1}(\widetilde \alpha(i\eta)) \nabla \widetilde G(\bbeta_{t})-\Si^{-1}(\alpha(i\eta))\nabla G(\bbeta_t)\Big\|^2dt\nonumber\\
&\quad\quad\quad-\frac{1}{2}\int_{k\eta}^{T}\Big\| \Si^{-1}(\widetilde \alpha(k\eta)) \nabla \widetilde G(\bbeta_{t})-\Si^{-1}(\alpha(k\eta))\nabla G(\bbeta_t)\Big\|^2dt\Big).
\end{align}
The above equality follows from the fact that each time interval $[\zeta_j,\zeta_{j+1}]$ always contain exactly some sub-interval $[i\eta,(i+1)\eta]$. Namely, we have $[\zeta_j,\zeta_{j+1}]=[i\eta,(i+1)\eta]\cup [(j+1)\eta,(j+2)\eta]\cup\cdots \cup[l\eta, (l+1)\eta]$, for some $i,l\in\{0,1,\cdots, k\}$. In particular, the matrix $\Si$ keep the same on each interval $[i\eta,(i+1)\eta]$, for some $i\in\{0,1,\cdots, k\}$. Similarly, we expand the Radon-Nokodym derivative for $\frac{d\mathbf N^{\widetilde S}}{d\mathbf N^{S}}$ on the time interval of length $\eta$. Based on our definition of jump intensity, we get
\begin{align}
 \frac{d\mathbf N^{\widetilde S}}{d\mathbf N^{S}}&=\exp\Big\{ -\sum_{j=0}^{N(T)}\int_{j\eta}^{(j+1)\eta\wedge T-\varepsilon} r\delta(t-\lfloor t/\eta\rfloor\eta)(\widetilde S(\widetilde\bbeta_{\lfloor t/\eta\rfloor \eta})-S(\bbeta _{\lfloor t/\eta\rfloor \eta}))\eta dt\Big.\nonumber \\
 &\quad\quad\quad\quad  -\int_{k\eta}^T r\delta(s-\lfloor s/\eta\rfloor\eta)(\widetilde S(\widetilde\bbeta_{k\eta})-S(\bbeta _{k\eta}))\eta ds \Big\}\times \Pi_{j=0}^{N(T)}\left(\frac{\widetilde S(\widetilde\bbeta_{\zeta_j})}{S(\bbeta_{\zeta_j})}\right)\nonumber \\
&=\exp\left\{ -\sum_{i=0}^{k}r(\widetilde S(\widetilde\bbeta_{i\eta})-S(\bbeta _{i\eta}))\eta \right\}\times \Pi_{j=0}^{N(T)}\left(\frac{\widetilde S(\widetilde\bbeta_{\zeta_j})}{S(\bbeta_{\zeta_j})}\right).
\end{align}
Without loss of generality, we shall only consider the sum $\sum_{i=0}^{k-1}$ and skip the interval $[k\eta,T]$. Notice that on each time interval $[i\eta,(i+1)\eta)$, the control $\alpha(i\eta)$ and $\widetilde\alpha(i\eta)$ are fixed, thus the two component of the measure $d\nu_{T,\bbeta}$ are independent. Taking into account the fact that $\bW^G$ is $\mathbb P^G$-Brownian motion, thus we apply the martingale property and arrive at 

% Next, we show the density for $d\mathbb P^G/d\mathbb P^{\widetilde G}$ as below. On each interval $[i\eta,(i+1)\eta)$, given initial value $(\bbeta_i,\widetilde \bbeta_i)$, the matrix $\Si(\alpha(i\eta))$ and $\Si(\widetilde\alpha(i\eta))$ are always the same, since the interval $[i\eta,(i+1)\eta)$ would always fall into one of the time interval $[\zeta_j,\zeta_{j+1})$ for some $j\in\{0,1,\cdots, N(T)\}$, where no jump would happen. In particular, in this case the integral on $[i\eta,(i+1)\eta)$ and $[i\eta,(i+1)\eta]$ are the same.  Thus we have the following Radon-Nikodym derivative    

\begin{align*}
&D_{KL}(\mu_{T}|\nu_{T})\\
=&\mathbb E_{\nu_T,\bbeta} \Big[
    \frac{1}{2}\sum_{i=0}^{k-1}\int_{i\eta}^{(i+1)\eta}\Big\| \Si^{-1}(\widetilde \alpha(i\eta)) \nabla \widetilde G(\bbeta_{t})-\Si^{-1}(\alpha(i\eta))\nabla G(\bbeta_t)\Big\|^2dt\Big] \nonumber \\
    &+\mathbb E_{\nu_T,\bbeta} \Big[ \sum_{i=0}^{k-1} [\widetilde S(\widetilde\bbeta_{i\eta})-S(\bbeta _{i\eta})]\eta- \sum_{j=0}^{N(T)}\Big( \log{\widetilde S(\widetilde\bbeta_{ \zeta_j})}-\log{S(\bbeta_{\zeta_j})}\Big)\Big]\nonumber\\
\le  &\underbrace{\frac{1}{2}\sum_{i=0}^{k-1}\mathbb E_{\nu_T,\bbeta} \Big[
    \int_{i\eta}^{(i+1)\eta}\Big\| \Si^{-1}(\widetilde \alpha(i\eta)) \nabla \widetilde G(\bbeta_{t})-\Si^{-1}(\alpha(i\eta))\nabla G(\bbeta_t)\Big\|^2dt\Big]}_{\mathcal I}\nonumber  \\
    &+\underbrace{\sum_{i=0}^{k-1}\mathbb E_{\nu_T,\bbeta} \Big[ r|\widetilde S(\widetilde\bbeta_{i\eta})-S(\bbeta _{i\eta})|\eta\Big] }_{\mathcal J} + \underbrace{\sum_{j=0}^{N(T)}\mathbb E_{\nu_T,\bbeta} \Big[ | \log{\widetilde S(\widetilde\bbeta_{ \zeta_j})}-\log{S(\bbeta_{\zeta_j})}|\Big]}_{\mathcal K}.\nonumber
\end{align*}
We then estimates the three terms $\mathcal I,\mathcal J, \mathcal K$ in order as below.

\noindent\textbf{Estimate of $\mathcal I$:} Due to the fact that every interval $[i\eta,(i+1)\eta)\subset [\zeta_j,\zeta_{j+1})$ for some $j\in \{0,1,\cdots,N(T)\}$, we know that the control $\alpha$ and $\widetilde\alpha$ are the same in the interval $[i\eta,(i+1)\eta]$ and the diffusion matrix $\Si$ is just constant matrix. Thus, we know that matrix $\Si^{-1}(\widetilde\alpha(i\eta))=\Si^{-1}(\alpha(i\eta))$, which takes one of the form from $(\Si^{-1}(0), \Si^{-1}(1) ):=\left\{\begin{pmatrix}{}
\frac{1}{\sqrt{2\tau^{(1)}}}\mathbf I_d&0\\
0&\frac{1}{\sqrt{2\tau^{(2)}}}\mathbf I_d
\end{pmatrix}, \begin{pmatrix}{}
\frac{1}{\sqrt{2\tau^{(2)}}}\mathbf I_d&0\\
0&\frac{1}{\sqrt{2\tau^{(1)}}}\mathbf I_d
\end{pmatrix}\right\}$. If $\Si^{-1}(\alpha(i\eta))=\Si^{-1}(0)$, we get 
\begin{align*}
 &\Big\| \Si^{-1}(\widetilde \alpha(i\eta)) \nabla \widetilde G(\bbeta_{t})-\Si^{-1}(\alpha(i\eta))\nabla G(\bbeta_t)\Big\|^2 \\
 =&\sum_{j=1}^d \frac{1}{2\tau^{(1)}}|\nabla_j \widetilde G(\bbeta_{t})-\nabla_j G(\bbeta_t) |^2+\sum_{j=d+1}^{2d} \frac{1}{2\tau^{(2)}}|\nabla_j \widetilde G(\bbeta_{t})-\nabla_j G(\bbeta_t) |^2 \\
 \le &\frac{1}{2\tau^{(1)}}\sum_{j=1}^{2d} |\nabla_j \widetilde G(\bbeta_{t})-\nabla_j G(\bbeta_t) |^2\leq \frac{1}{2\tau^{(1)}}\|\nabla \widetilde G(\bbeta_{t})-\nabla_j  G(\bbeta_t) \|^2.
\end{align*}
Here $\nabla G(\bbeta):=\begin{pmatrix}{}
\nabla L(\bbeta^{(1)})\\
\nabla L(\bbeta^{(2)})
\end{pmatrix}$ and $\nabla \widetilde G(\bbeta):=\begin{pmatrix}{}
\nabla \widetilde L(\bbeta^{(1)})\\
\nabla \widetilde L(\bbeta^{(2)})
\end{pmatrix}$.
The other matrix form of $\Si^{-1}(1)$ will result in the same estimates. We thus get 
\begin{align*}
    \mathcal I\le& \frac{1}{4\tau^{(1)}}\sum_{i=0}^{k-1}\mathbb E_{\nu_T,\bbeta} \Big[
    \int_{i\eta}^{(i+1)\eta}\Big\|  \nabla \widetilde G(\bbeta_{t})-\nabla G(\bbeta_t)\Big\|^2dt\Big]
%   = &\color{red} \frac{1}{4\tau^{(2)}}\sum_{i=0}^{k-1}\mathbb E_{\nu_T,\bbeta} \Big[
%     \int_{k\eta}^{(k+1)\eta}\Big\|  \nabla \widetilde G(\widetilde\bbeta_{t})-\nabla G(\widetilde\bbeta_t)\Big\|^2dt\Big]
\end{align*}
On each fixed interval, for $t\in[k\eta,(k+1)\eta)$ , we have $\mathbb P^G$-Brownian motion and $\mathbb P^{\widetilde G}$-Brownian motion (see examples in Theorem 8.6.6 and Example 8.6.9 \citep{oksendal2003stochastic}),
\begin{align*}
   d \bW_{t}^{G}=&d\bW_{t}+\Si^{-1}(\alpha_t)(\nabla G(\bbeta_t))dt.\\
     d \bW_{t}^{\widetilde G}=&d\bW_{t}+\Si^{-1}(\alpha_t)(\nabla \widetilde G(\widetilde\bbeta_t))dt.
\end{align*}
Plugging the $\mathbb P^G$ (and $\mathbb P^{\widetilde G}$)-Brownian motions to the original dynamics (\ref{reLD}) and (\ref{reSGLD}), we have
\begin{align*}
    d\bbeta_t =\Si(\alpha_t)d\bW_t^G,\quad \text{and}\quad d\widetilde\bbeta_t=\Si(\alpha_t)d\bW^{\widetilde G}_t.
\end{align*}
On each interval $[i\eta,(i+1)\eta)$, $\Si(\alpha_t)$ is a constant matrix, thus we know that the probability distribution of $\{\bbeta_t\}_{t\in[k\eta,(k+1)\eta)}$ and $\{\widetilde\bbeta_t\}_{t\in[k\eta,(k+1)\eta)}$ are the same and we denote as $\mathcal L(\bbeta_t)=\mathcal{L}(\widetilde\bbeta_t)$. The difference is that $\bbeta_t$ is driven by $\mathbb P^G$-Brownian motion and $\widetilde\bbeta_t$ is driven by $\mathbb P^{\widetilde G}$-Brownian motion, which implies that, for $t\in [i\eta,(i+1)\eta)$, we have 
\begin{align}
    \mathbb E_{\nu_T,\bbeta}\Big[\Big\|  \nabla \widetilde G(\bbeta_{t})-\nabla G(\bbeta_t)\Big\|^2 \Big] = \mathbb E_{\mu_T,\widetilde\bbeta}\Big[\Big\|  \nabla \widetilde G(\widetilde\bbeta_{t})-\nabla G(\widetilde\bbeta_t)\Big\|^2 \Big].
\end{align}
Thus, we have the following estimates, 
\begin{align*}
 \mathcal I    \le &\frac{1}{4\tau^{(1)}}\sum_{i=0}^{k-1}\mathbb E_{\mu_T,\widetilde\bbeta} \Big[
    \int_{i\eta}^{(i+1)\eta}\Big\|  \nabla  G(\widetilde\bbeta_t)-\nabla G(\widetilde\bbeta_{\lfloor t/\eta\rfloor\eta})\Big\|^2dt\Big]\\
    &+\frac{1}{4\tau^{(1)}}\sum_{i=0}^{k-1}\mathbb E_{\mu_T,\widetilde\bbeta} \Big[
    \int_{i\eta}^{(i+1)\eta}\Big\|  \nabla G(\widetilde\bbeta_{\lfloor t/\eta\rfloor\eta})-\nabla\widetilde G(\widetilde\bbeta_{\lfloor t/\eta\rfloor\eta})\Big\|^2dt\Big]\\
\le&\frac{C^2}{4\tau^{(1)}}\sum_{i=0}^{k-1}\mathbb E_{\mu_T,\widetilde\bbeta} \Big[
    \int_{i\eta}^{(i+1)\eta}\Big\|  \widetilde\bbeta_t-\widetilde\bbeta_{i\eta})\Big\|^2dt\Big]\cdots \mathcal I_1\\
    &+\frac{1}{4\tau^{(1)}}\sum_{i=0}^{k-1}\mathbb E_{\mu_T,\widetilde\bbeta} \Big[
    \int_{i\eta}^{(i+1)\eta}\Big\|  \nabla G(\widetilde\bbeta_{\lfloor t/\eta\rfloor\eta})-\nabla\widetilde G(\widetilde\bbeta_{\lfloor t/\eta\rfloor\eta})\Big\|^2dt\Big]\cdots \mathcal I_2.
\end{align*}
We now estimate the two terms $\mathcal I_1$ and $\mathcal I_2$ separately. Notice that, following our notation of $\mathbb P^{\widetilde G}$-Brownian motion, for $t\in [i\eta,(i+1)\eta)$, we have 
\begin{align*}
\widetilde \bbeta_t-\widetilde\bbeta_{i\eta}=\Si(\alpha_t) (\bW^{\widetilde G}_t-\bW^{\widetilde G}_{i\eta} )=\Si(\alpha_t) (\bW^{\widetilde G}_t-\bW^{\widetilde G}_{i\eta} ),   
\end{align*}
which implies that (recall that $d\mu_T=d\mathbb P^{\widetilde G}\times \mathbf N^{\widetilde S}$ and $\Si\in \mathbb R^{2d\times 2d}$),  
\begin{align*}
    \mathbb E_{\mu_{T,\widetilde\bbeta}} [\|\widetilde  \bbeta_t-\widetilde\bbeta_{i\eta}\|^2 ]\le 2\tau^{(1)}d\eta + 2\tau^{(2)}d\eta\leq 4\tau^{(2)}d\eta.
\end{align*}
We thus conclude that,
\begin{align*}
   \mathcal I_1\le C^2\frac{\tau^{(2)}}{\tau^{(1)}
    }kd\eta^2. 
\end{align*}
As for the term $\mathcal I_2$, according to Assumption \ref{assump: stochastic_noise}, we obtain that
\begin{align*}
  \mathcal I_2    \le & \frac{\eta\delta}{4\tau^{(1)}}\sum_{i=0}^{k-1}\mathbb E_{\mu_T,\widetilde\bbeta} \Big[C^2\|\widetilde\bbeta_{i\eta}\|^2+\Phi^2\Big].
\end{align*}
Now, we just need to estimate $E_{\mu_T,\widetilde\bbeta} [\|\widetilde\bbeta_{k\eta}\|^2]$ \footnote[2]{In principle, the Wiener measure $\bm{W}$ under $\mathbb P^{\widetilde G}$ is not a Brownian motion, thus the uniform $L^2$ bound used in Lemma.3 may not be appropriate. Instead, we estimate the upper bound using a slightly weaker result.}. On each interval $[i\eta,(i+1)\eta]$, under the measure $d\mu_{T,\widetilde \bbeta}$, we have 
\begin{align*}
    \widetilde\bbeta_{(i+1)\eta}=\widetilde\bbeta_{i\eta}+\Si(\alpha(i\eta))(\bW^{\widetilde G}_{(i+1)\eta}-\bW^{\widetilde G}_{i\eta}),
\end{align*}
which implies that 
\begin{align*}
   & \mathbb E_{\mu_T,\widetilde \bbeta} [\|\widetilde\bbeta_{(i+1)\eta}\|^2 ]\\
   =& \mathbb E_{\mu_T,\widetilde \bbeta} [\|\widetilde\bbeta_{i\eta}\|^2 ]+\mathbb E_{\mu_T,\widetilde \bbeta} [\langle \widetilde\bbeta_{i\eta},\bW^{\widetilde G}_{(i+1)\eta}-\bW^{\widetilde G}_{i\eta} \rangle ]+\mathbb E_{\mu_T,\widetilde \bbeta} [\|\bW^{\widetilde G}_{(i+1)\eta}-\bW^{\widetilde G}_{i\eta} \|^2 ] \\
    =&  \mathbb E_{\mu_T,\widetilde \bbeta} [\|\widetilde\bbeta_{i\eta}\|^2 ]+[2\tau^{(1)}+2\tau^{(2)}]d\eta
    % \le \mathbb E_{\mu_T,\widetilde \bbeta} [\|\widetilde\bbeta_{k\eta}\|^2 ]+4\tau^{(2)}d\eta
\end{align*}
The last equality follows from the independence of $\widetilde\bbeta_{k\eta}$ and $\bW^{\widetilde G}_{(k+1)\eta}-\bW^{\widetilde G}_{k\eta}$, and $\bW^{\widetilde G}$ is a $\mathbb P^{\widetilde G}$-Brownian motion.
By induction, we get 
\begin{align*}
 \mathbb E_{\mu_T,\widetilde \bbeta} [\|\widetilde\bbeta_{i\eta}\|^2 ]   \le 2id[\tau^{(1)}+\tau^{(2)}]\eta\le 2kd[\tau^{(1)}+\tau^{(2)}] . 
\end{align*}
We conclude that, 
\begin{align*}
     \mathcal I_2    \le  \frac{k\eta\delta}{4\tau^{(1)}}\Big(2C^2 [{\tau^{(1)}}{}+{\tau^{(2)}}{} ]kd \eta+\Phi^2\Big),
\end{align*}
which implies that
\begin{align*}
    \mathcal I\le \frac{k\eta}{4\tau^{(1)}}\Big(2\delta C^2 [{\tau^{(1)}}{}+{\tau^{(2)}}{} ]kd \eta+\delta\Phi^2\Big)+C^2\frac{\tau^{(2)}}{\tau^{(1)}
    }kd\eta^2.
\end{align*}

\noindent\textbf{Estimate $\mathcal J$:} According to our definition of the swapping probability, we have, for each $i$, 
\begin{align*}
    \widetilde S(\widetilde\bbeta_{i\eta}) = \min\{1, \widetilde S_{\eta, m, n}(\widetilde\bbeta_{i\eta}^{(1)},\widetilde\bbeta_{i\eta}^{(2)})\},\quad     S(\bbeta_{i\eta}) = \min\{1, S(\bbeta_{i\eta}^{(1)},\bbeta_{i\eta}^{(2)})\},
\end{align*}
which means $|\widetilde S(\widetilde\bbeta_{i\eta})-S(\bbeta_{i\eta})|\le 1$. Denote $C_{\tau}=|\frac{1}{\tau^{(1)}}-\frac{1}{\tau^{(2)}}|$, we have  
 \begin{align*}
 \widetilde S_{\eta, m, n}(\widetilde\bbeta_{i\eta}^{(1)},\widetilde\bbeta_{i\eta}^{(2)})=&\exp\Big(C_{\tau}( \widetilde L(B_{i\eta}|\bbeta_{i\eta}^{(1)})-\widetilde L(B_{i\eta}|\bbeta_{i\eta}^{(2)}))-C_{\tau}^2\frac{\widetilde\sigma^2}{2}\Big)\\
 S(\bbeta_{i\eta}^{(1)},\bbeta_{i\eta}^{(2)})=& \exp\Big(C_{\tau}  (L( \bbeta_{i\eta}^{(1)})- L( \bbeta_{i\eta}^{(2)}))\Big).
 \end{align*}

 Applying Taylor expansion for the exponential function at $C_{\tau}  (L( \bbeta_{k\eta}^{(1)})- L( \bbeta_{k\eta}^{(2)}))$, we have 
 
%  \textcolor{red}{1. It looks weird if there is no expectation, maybe you can move the folliwing inequality inside (52)-(55)?; 2. I feel we should change the following first inequality to $\lesssim$.}
 \begin{align*}
&\E_{\nu_T,\bbeta} \Big[  | \widetilde S_{\eta, m, n}(\widetilde\bbeta_{i\eta}^{(1)},\widetilde\bbeta_{i\eta}^{(2)})- S(\bbeta_{i\eta}^{(1)},\bbeta_{i\eta}^{(2)})|\Big]  \\
\lesssim & \E_{\nu_T,\bbeta} \Big[ S(\bbeta_{i\eta}^{(1)},\bbeta_{i\eta}^{(2)})\Big|C_{\tau}( \widetilde L(B_{k\eta}|\bbeta_{i\eta}^{(1)})-\widetilde L(B_{k\eta}|\bbeta_{i\eta}^{(2)}))-C_{\tau}^2\frac{\widetilde\sigma^2}{2}-C_{\tau}  (L( \bbeta_{i\eta}^{(1)})- L( \bbeta_{i\eta}^{(2)}))\Big|+\text{higher order term}\Big]\\
\le &\E_{\nu_T,\bbeta} \Big[  \Big|C_{\tau}( \widetilde L(B_{i\eta}|\bbeta_{i\eta}^{(1)})-\widetilde L(B_{i\eta}|\bbeta_{i\eta}^{(2)}))-C_{\tau}^2\frac{\widetilde\sigma^2}{2}-C_{\tau}  (L( \bbeta_{i\eta}^{(1)})- L( \bbeta_{i\eta}^{(2)}))\Big|+\mathcal{O}(\widetilde \sigma^2)\Big]
 \end{align*}
 where the last inequality follows from %$\Big|C_{\tau}( \widetilde L(B_{i\eta}|\bbeta_{i\eta}^{(1)})-\widetilde L(B_{i\eta}|\bbeta_{i\eta}^{(2)})-C_{\tau}^2\frac{\widetilde\sigma^2}{2}-C_{\tau}  (L( \bbeta_{i\eta}^{(1)})- L( \bbeta_{i\eta}^{(2)}))\Big| \approx \mathcal O( \widetilde \sigma^2)$ and
 $S(\bbeta_{i\eta}^{(1)},\bbeta_{i\eta}^{(2)})\le 1$. Combining Lemma \ref{vr-estimator}, we thus get the following estimates,
 \begin{align*}
    \mathcal J&= \sum_{i=0}^{k-1}\mathbb E_{\nu_T,\bbeta} \Big[  r|\widetilde S_{\eta, m, n}(\widetilde\bbeta_{i\eta})-S(\bbeta _{i\eta})|\eta\Big] \\
    &\le   r\eta \sum_{i=0}^{k-1}\mathbb E_{\nu_T,\bbeta}\Big[\Big|C_{\tau}( \widetilde L(B_{i\eta}|\bbeta_{i\eta}^{(1)})-\widetilde L(B_{i\eta}|\bbeta_{i\eta}^{(2)}))-C_{\tau}^2\frac{\widetilde\sigma^2}{2}-C_{\tau}  (L( \bbeta_{i\eta}^{(1)})- L( \bbeta_{i\eta}^{(2)}))\Big|+\mathcal{O}(\widetilde \sigma^2) \Big] \\
    % &\le kr\eta \min\left\{1, \left|\frac{1}{\tau^{(1)}}-\frac{1}{\tau^{(2)}}\right|  \widetilde \sigma+\mathcal{O}(\widetilde \sigma^2) \right\},\\
    &\le rk\eta \mathcal{O}(C_{\tau} \widetilde \sigma + \widetilde \sigma^2)= rk\eta \mathcal{O}\left(\left(\frac{m^2}{n}\eta\right)^{1/2} d\right) 
\end{align*}
where the last inequality follows from the Jensen's inequality and the last order holds given a large enough batch size $n$ or a small enough $m$ and $\eta$. % the similar estimates as shown below for the term $\mathcal K$.

\noindent\textbf{Estimate $\mathcal K$:} We now estimate the last term $\mathcal K$, we have 
\begin{align*}\small
    \mathcal K= & \sum_{j=0}^{N(T)}\mathbb E_{\nu_T,\bbeta} \Big[ | \log{\widetilde S_{\eta, m, n}(\widetilde\bbeta_{ \zeta_j})}-\log{S(\bbeta_{\zeta_j})}|\Big]\\
    \le & C_{\tau} \sum_{j=0}^{N(T)}\mathbb E_{\nu_T,\bbeta} \Big[ \Big|[\widetilde L(B_{\zeta_j}|\bbeta_{\zeta_j}^{(1)})-\widetilde L(B_{\zeta_j}|\bbeta_{\zeta_j}^{(2)})-C_{\tau}\frac{\widetilde\sigma^2}{2}]-[L( \bbeta_{\zeta_j}^{(1)})- L( \bbeta_{\zeta_j}^{(2)})]\Big|\Big]\\
    \le &N(T)C_{\tau}^2\mathbb E_{\nu_T,\bbeta}[\widetilde \sigma^2/2]+C_{\tau}\sum_{j=1}^{N(T)} \Var[\widetilde L(B_{\zeta_j}|\bbeta_{\zeta_j}^{(1)})-\widetilde L(B_{\zeta_j}|\bbeta_{\zeta_j}^{(2)}) ]^{1/2}\\
    \le &\frac{N(T) C_{\tau}^2  \widetilde \sigma^2}{2}+N(T)C_{\tau}\widetilde \sigma
\end{align*}
Combining Lemma \ref{vr-estimator} again, we conclude with
\begin{align*}
    \mathcal K\le   C_{\tau}^2 \frac{N(T) \widetilde \sigma^2}{2}+ N(T) C_{\tau}  \widetilde \sigma=N(T)\mathcal{O}\left(\left(\frac{m^2}{n}\eta\right)^{1/2}d\right).
\end{align*}
Combining the estimates of $\mathcal I$, $\mathcal J$, and $\mathcal K$, we complete the proof.
\end{proof}
\begin{remark}
\label{argument_raginsky} After the change of measure, the expectation is under the new measure $\mathbb P^G$ (or $\mathbb P^{\widetilde G}$) instead of the Wiener measure $\mathbb P$. In the estimate of term $\mathcal I$, similar $L^2$ estimates of the term $\mathbb E_{\mu_T,\widetilde \bbeta}[\|\widetilde \bbeta_{(i+1)\eta} \|^2]$ has been obtained in \cite{Maxim17}[Proof of Lemma 7] when there is no swap. The difference is we write the dynamic of $\widetilde\bbeta_{(i+1)\eta}$ with respect to the $\mathbb P^{\widetilde G}$-Brownian motion $\bW^{\widetilde G}$ instead of the $\mathbb P$-Brownian motion $W$. In principle, $W$ under $\mathbb P^{\widetilde G}$ is not a Brownian motion. 
\end{remark}
%%%%%%%%%%%%%%%%%%%%%%%%%%%%%%%%%%%%%%%

We then extend the distance of relative entropy $D_{KL}(\mu_{T}|\nu_{T})$ to the Wasserstein distance $\mathcal{W}_2(\mu_{T},\nu_{T})$ via a weighted transportation-cost inequality of \citet{bolley05}.

\begin{theorem}
\label{numerical_error_W2}
Given a large enough batch size $n$ or a small enough $m$ and $\eta$, we have
\begin{equation}
\label{second_last_W2}
\begin{split}
    \mathcal{W}_2(\mu_{T},\nu_{T})&\leq  \mathcal{O}\left(d k^{3/2}\eta \left(\eta^{1/4}+\delta^{1/4}+\left(\frac{m^2}{n}\eta\right)^{1/8}\right)\right).\\
\end{split}
\end{equation}
\end{theorem}

\begin{proof} Before we proceed, we first show in Lemma.\ref{L2_bound} that $\nu_{T}$ has a bounded second moment; the $L_2$ upper bound of $\mu_{T}$ is majorly proved in Lemma.C2 \citep{chen2018accelerating} except that the slight difference is that the constant in the RHS of (C.38) \citet{chen2018accelerating} is changed to account for the stochastic noise. Then applying Corollary 2.3 in \citet{bolley05}, we can upper bound the two Borel probability measures $\mu_{T}$ and $\nu_{T}$ with finite second moments as follows
\begin{equation}
    \mathcal{W}_2(\mu_{T},\nu_{T})\leq C_{\nu} \left[\sqrt{D_{KL}(\mu_{T}|\nu_{T})}+\left(\frac{D_{KL}(\mu_{T}|\nu_{T})}{2}\right)^{1/4}\right],
\end{equation}
where $C_{\nu}=2\inf_{\lambda>0}\left(\frac{1}{\lambda}\left(\frac{3}{2}+\log\int_{\mathbb{R}^d} e^{\lambda \|w\|^2}\nu(dw)\right)\right)^{1/2}$. Applying Lemma \ref{exp_int}, we have
\begin{equation*}
\begin{split}
    \mathcal{W}_2^2(\mu_{T},\nu_{T})&\leq \left(12+8\left(\kappa_0+2b+4d \tau^{(2)}\right)k\eta\right)\left(D_{KL}(\mu_{T}|\nu_{T})+\sqrt{D_{KL}(\mu_{T}|\nu_{T})}\right).\\
    % &\leq \left(12+8\left(\kappa_0+2b+4d \tau^{(2)}\right)\right)k\eta \left((\widetilde \Phi_0 + \widetilde \Phi_1 \sqrt{\eta})k\eta + N(T)\widetilde \Phi_2\right),\\
    % &\lesssim \left(\right)(k\eta)^2 + N(T)k\eta.
\end{split}
\end{equation*}

Combining Lemma.\ref{numerical_error_KL} and $\sqrt{N(T)}\leq N(T)$ and taking $\eta\leq 1$, $k\eta>1$, and $\lambda=1$, we have
\begin{equation*}
\begin{split}
    \mathcal{W}_2^2(\mu_{T,\widetilde\bbeta},\nu_{T, \bbeta})&\leq \left(12+8\left(\kappa_0+2b+4d \tau^{(2)}\right)\right)k\eta \left((\widetilde \Phi_0 + \widetilde \Phi_1 \sqrt{\eta})k\eta + N(T)\widetilde \Phi_2\right),\\
    % &\lesssim \left(\right)(k\eta)^2 + N(T)k\eta.
\end{split}
\end{equation*}
where $\widetilde \Phi_i=\Phi_i+\sqrt{\Phi_i}$ for $i\in \{0, 1,2\}$. In what follows, we have
\begin{equation*}
\begin{split}
    \mathcal{W}_2^2(\mu_{T,\widetilde\bbeta},\nu_{T, \bbeta})&\leq  \left(\Psi_0 + \Psi_1\sqrt{\eta}\right)(k\eta)^2+\Psi_2 k\eta N(T) ,\\
\end{split}
\end{equation*}
where $\Psi_i=\left(12+8\left(\kappa_0+2b+4d \tau^{(2)}\right)\right)\widetilde \Phi_i$ for $i\in \{0, 1,2\}$.

By the orders of $\Phi_0$, $\Phi_1$ and $\Phi_2$ defined in Lemma.\ref{numerical_error_KL}, we have
\begin{equation*}
\begin{split}
    \mathcal{W}_2^2(\mu_{T,\widetilde\bbeta},\nu_{T, \bbeta})&\leq  \mathcal{O}\left(d^2 k^3\eta^2 \left(\eta^{1/2}+\delta^{1/2}+\left(\frac{m^2}{n}\eta\right)^{1/4}+\frac{N(T)}{k\eta}\left(\frac{m^2}{n}\eta\right)^{1/4}\right)\right)\\
    &\leq \mathcal{O}\left(d^2 k^3\eta^2 \left(\eta^{1/2}+\delta^{1/2}+\left(\frac{m^2}{n}\eta\right)^{1/4}\right)\right),\\
\end{split}
\end{equation*}
where $\frac{N(T)}{k\eta}$ can be interpreted as the average swapping rate from time $0$ to $T$ and is of order $\mathcal{O}(1)$. Taking square root to both sides of the above inequality lead to the desired result (\ref{second_last_W2}).

\end{proof}

\section{Proof of Technical Lemmas}
\setcounter{lemma}{0}

\begin{lemma}[Local Lipschitz continuity]\label{local_smooth}
Given a $d$-dimensional centered ball $U$ of radius $R$, $L(\cdot)$ is $D_R$-Lipschitz continuous in that $|L(\bx_i|\bbeta_1)- L(\bx_i|\bbeta_2)|\le \frac{D_R}{N}\|\bbeta_1-\bbeta_2\|$ for $\forall \bbeta_1, \bbeta_2\in U$ and any $i\in\{1,2,\cdots, N\}$,
where $D_R=CR+\max_{i\in\{1,2,\cdots, N\}} N \|\nabla L(\bx_i|\bbeta_{\star})\|+\frac{Cb}{a}$.
\end{lemma}{}  
\begin{proof}

For any $\bbeta_1, \bbeta_2\in U$, there exists $\bbeta_3\in U$ that satisfies the mean-value theorem such that
\begin{equation*}
    |L(\bx_i|\bbeta_1) - L(\bx_i|\bbeta_2)|=\langle\nabla L(\bx_i|\bbeta_3), \bbeta_1-\bbeta_2\rangle \leq \|\nabla L(\bx_i|\bbeta_3)\|\cdot \|\bbeta_1-\bbeta_2\|,
\end{equation*}

Moreover, by Lemma \ref{grad_bound}, we have
\begin{equation*}
    |L(\bx_i|\bbeta_1) - L(\bx_i|\bbeta_2)|\leq \|\nabla L(\bx_i|\bbeta_3)\|\cdot \|\bbeta_1-\bbeta_2\|\leq \frac{CR+Q}{N}\|\bbeta_1-\bbeta_2\|.\qed
\end{equation*}

\end{proof}

\begin{lemma}
\label{grad_bound}
Under the smoothness
and dissipativity assumptions \ref{assump: lip and alpha beta}, \ref{assump: dissipitive}, for any $\bbeta\in\mathbb{R}^d$, it follows that
\begin{equation}
   \|\nabla L(\bx_i|\bbeta)\|\leq \frac{C}{N}\|\bbeta\|+\frac{Q}{N}.
\end{equation}
where $Q=\max_{i\in\{1,2,\cdots, N\}} N \|\nabla L(\bx_i|\bbeta_{\star})\|+\frac{bC}{a}$.
\end{lemma}

\begin{proof}
According to the dissipativity assumption, we have 
\begin{equation}
    \la \bbeta_{\star},\nabla L(\bbeta_{\star})\rangle \ge a\|\bbeta^{\star}\|^2-b,
\end{equation}
where $\bbeta_{\star}$ is a minimizer of $\nabla L(\cdot)$ such that $\nabla L(\bbeta_{\star})=0$. In what follows, we have $\|\bbeta_{\star}\|\leq \frac{b}{a}$.

Combining the triangle inequality and the smoothness assumption \ref{assump: lip and alpha beta}, we have
\begin{equation}
\begin{split}
    \|\nabla L(\bx_i|\bbeta)\|\leq & C_N\|\bbeta-\bbeta_{\star}\| + \|\nabla L(\bx_i|\bbeta_{\star})\|\leq C_N\|\bbeta\| + \frac{C_N b}{a} + \|\nabla L(\bx_i|\bbeta_{\star})\|.
\end{split}
\end{equation}

Setting $C_N=\frac{C}{N}$ as in (\ref{2nd_smooth_condition}) and $Q=\max_{i\in\{1,2,\cdots, N\}} \|\nabla L(\bx_i|\bbeta_{\star})\|+\frac{b C}{a}$ completes the proof. 
\qed
\end{proof}

The following lemma is majorly adapted from Lemma C.2 of \citet{chen2018accelerating}, except that the corresponding constant in the RHS of (C.38) is slightly changed to account for the stochastic noise. A similar technique has been established in Lemma 3 of \citet{Maxim17}.
\begin{lemma}[Uniform $L^2$ bounds on replica exchange SGLD]
\label{Uniform_bound}
Under the smoothness
and dissipativity assumptions \ref{assump: lip and alpha beta}, \ref{assump: dissipitive}.  Given a small enough learning rate $\eta\in (0, 1\vee \frac{a}{C^2})$, there exists a positive constant $\Psi_{d,\tau^{(2)}, C, a, b}<\infty$ such that $\sup_{k\geq 1} \E[\|\bbeta_{k}\|^2] < \Psi_{d,\tau^{(2)}, C, a, b}$.
\end{lemma}

% The following lemma is a restatement of Lemma C.5 in \citet{Xu18}.
% \begin{lemma}
% \label{lemma:1}
% Under Assumptions smoothness / dissipitive, for any $\bbeta\in\mathbb{R}^d$, it follows that
% \begin{equation}
%     \E\left[\left\|\nabla F(\bx)-\frac{1}{n}\sum_{i\in B_k}\nabla f_i(\bbeta)\right\|^2\right]\leq \frac{4(n-B)(M\|\bbeta\|+G)^2}{B(n-1)}, 
% \end{equation}
% where $G=\max_{i\in\{1,2,\cdots, n\}}\{\|\nabla f_i(\bbeta^{\star})\}+\frac{bM}{m}$.
% \end{lemma}

\begin{lemma}[Exponential dependence on the variance]
\label{exponential_dependence}
Assume $S$ is a log-normal distribution with mean $u-\frac{1}{2}\sigma^2$ and variance $\sigma^2$ on the log scale. Then $\E[\min(1, S)]=\mathcal{O}(e^{u-\frac{\sigma^2}{8}})$, which is {exponentially smaller} given a large variance $\sigma^2$.
\end{lemma}

\begin{proof} For a log-normal distribution $S$ with mean $u-\frac{1}{2}\sigma^2$ and variance $\sigma^2$ on the log scale, the probability density $f_S(S)$ follows that $\frac{1}{S\sqrt{2\pi \sigma^2}}\exp\left\{-\frac{(\log S-u+\frac{1}{2}\sigma^2)^2}{2\sigma^2}\right\}$. In what follows, we have
\begin{equation*}
\small
\begin{split}
    \E[\min(1, S)]=&\int_{0}^{\infty} \min(1, S) f_S(S)dS=\int_{0}^{\infty} \min(1, S) \frac{1}{S\sqrt{2\pi \sigma^2}}\exp\left\{-\frac{(\log S-u+\frac{1}{2}\sigma^2)^2}{2\sigma^2} \right\}dS\\
\end{split}
\end{equation*}

By change of variable $y=\frac{\log S-u+\frac{1}{2}\sigma^2}{\sigma}$ where $S=e^{\sigma y+u-\frac{1}{2}\sigma^2}$ and $y=-\frac{u}{\sigma}+\frac{\sigma}{2}$ given $S=1$, it follows that
\begin{equation*}
\begin{split}
\small
    &\E[\min(1, S)]\\
    =&\int_{0}^{1} S \frac{1}{S\sqrt{2\pi \sigma^2}}\exp\left\{-\frac{(\log S-u+\frac{1}{2}\sigma^2)^2}{2\sigma^2} \right\}dS+\int_{1}^{\infty} \frac{1}{S\sqrt{2\pi \sigma^2}}\exp\left\{-\frac{(\log S-u+\frac{1}{2}\sigma^2)^2}{2\sigma^2} \right\}dS\\
    =&\int_{-\infty}^{-\frac{u}{\sigma}+\frac{\sigma}{2}} \frac{1}{\sqrt{2\pi\sigma^2}} e^{-\frac{y^2}{2}}\sigma e^{u-\frac{1}{2}\sigma^2+\sigma y}dy+\int_{-\frac{u}{\sigma}+\frac{\sigma}{2}}^{\infty}\frac{1}{\sqrt{2\pi\sigma^2}} e^{-\sigma y-u+\frac{1}{2}\sigma^2}e^{-\frac{y^2}{2}}\sigma e^{u-\frac{1}{2}\sigma^2+\sigma y}dy \\
    =&e^u\int_{-\infty}^{-\frac{u}{\sigma}+\frac{\sigma}{2}} \frac{1}{\sqrt{2\pi}} e^{-\frac{(y-\sigma)^2}{2}}dy+\frac{1}{\sigma}\int_{-\frac{u}{\sigma}+\frac{\sigma}{2}}^{\infty} \frac{1}{\sqrt{2\pi}} e^{-\frac{y^2}{2}}dy\\
    =&e^u\int^{\infty}_{\frac{u}{\sigma}+\frac{\sigma}{2}} \frac{1}{\sqrt{2\pi}} e^{-\frac{z^2}{2}}dz+\frac{1}{\sigma}\int_{-\frac{u}{\sigma}+\frac{\sigma}{2}}^{\infty} \frac{1}{\sqrt{2\pi}} e^{-\frac{y^2}{2}}dy\\
    \leq & e^u\int^{\infty}_{-\frac{u}{\sigma}+\frac{\sigma}{2}} \frac{1}{\sqrt{2\pi}} e^{-\frac{z^2}{2}}dz+\frac{1}{\sigma}\int_{-\frac{u}{\sigma}+\frac{\sigma}{2}}^{\infty} \frac{1}{\sqrt{2\pi}} e^{-\frac{y^2}{2}}dy\\
    \leq &\left(e^u+\frac{1}{\sigma}\right)e^{-\frac{(-\frac{u}{\sigma}+\frac{\sigma}{2})^2}{2}}\lesssim e^{u-\frac{\sigma^2}{8}},\\
\end{split}
\end{equation*}
where the last equality follows from the change of variable $z=\sigma-y$ and the second last inequality follows from the exponential tail bound of the standard Gaussian variable $\mathbb{P}(y>\epsilon)\leq e^{\frac{-\epsilon^2}{2}}$.
\qed
\end{proof}

\begin{lemma}[Uniform $L^2$ bound on replica exchange Langevin diffusion]
\label{L2_bound}
For all $\eta\in (0, 1\wedge\frac{a}{4C^2})$, we have that $$\E[\|(\bbeta^{(1)}_t, \bbeta^{(2)}_t)\|^2]\leq \E[e^{\|\bbeta_0^{(1)}, \bbeta_0^{(2)}\|^2}]+\frac{b+2d\tau^{(2)}}{a}.$$
\end{lemma}

\begin{proof}
Consider $L_t(\bbeta_t)=\|\bbeta_t\|^2$, where $\bbeta_t=(\bbeta_t^{(1)}, \bbeta_t^{(2)})\in\mathbb{R}^{2d}$. The proof is marjorly adapted from Lemma 3 in \citet{Maxim17}, except that the generalized It\^{o} formula (formula 2.7 in page 29 of \citet{yin_zhu_10}) is used to handle the jump operator, which follows that
\begin{equation*}
\begin{split}
    dL_t=&-2\langle \bbeta_t, \nabla G(\bbeta_t) \rangle + 2d(\tau^{(1)}+\tau^{(2)}) dt+2\bbeta_t^{T}\Sigma(\alpha_t) dW(t)\\
    &\ \ \ +\underbrace{rS_{\eta, m, n}(\bbeta_t^{(1)},\bbeta_t^{(2)})\cd (L_t(\bbeta^{(2)},\bbeta^{(1)})-L_t(\bbeta_t^{(1)},\bbeta_t^{(2)}))}_{\text{Jump-inducing drift}} + M_1(t)+M_2(t),
\end{split}
\end{equation*}
where $\nabla G(\bbeta):=\begin{pmatrix}{}
\nabla L(\bbeta^{(1)})\\
\nabla L(\bbeta^{(2)})
\end{pmatrix}$ and $M_1(t)$ and $M_2(t)$ are two martingales defined in formula 2.7 in \citet{yin_zhu_10}). Due to the definition of $L_t(\bbeta_t)$, we have $L_t(\bbeta_t^{(1)}, \bbeta_t^{(2)})=L_t(\bbeta_t^{(2)}, \bbeta_t^{(1)})$, which implies that the Jump-inducing drift actually disappears. Taking expectations and applying the margingale property of the It\^{o} integral, we have the almost the same upper bound as Lemma 3 in \citet{Maxim17}. Combining $\E[\|\bbeta_0\|^2]\leq \log \E[e^{{\|\bbeta_0\|^2}}]$ completes the proof.
\end{proof}

\begin{lemma}[Exponential integrability of replica exchange Langevin diffusion]
\label{exp_int}
For all $\tau\leq \frac{2}{a}$, it follows that
\begin{equation*}
    \log \E[e^{\|(\bbeta_t^{(1)}, \bbeta_t^{(2)})\|^2}]\leq \underbrace{\log \E[e^{\|(\bbeta_0^{(1)}, \bbeta_0^{(2)})\|^2}]}_{\kappa_0} + 2(b+2d\tau^{(2)})t.
\end{equation*}
\end{lemma}

\begin{proof}
The proof is marjorly adapted from Lemma 4 in \citet{Maxim17}. The only difference is that the generalized It\^{o} formula (formula 2.7 in \citet{yin_zhu_10}) is used again as in Lemma \ref{L2_bound}. Consider $L(t, \bbeta_t)=e^{\|\bbeta_t\|^2}$, where $\bbeta=(\bbeta_t^{(1)}, \bbeta_t^{(2)})\in\mathbb{R}^{2d}$. Due to the special structure that $L(t, \bbeta_t)$ is invariant under the swaps of $(\bbeta_t^{(1)}, \bbeta_t^{(2)})$, the generator of $L(t, \bbeta_t)$ with swaps is the same as the one without swaps. Therefore, the desired result follows directly by repeating the steps from Lemma 4 in \citet{Maxim17}.
\end{proof}

% \section{More Empirical Study on on Simulations}
% \label{subsec:Toy-multi-modal-example_main_body}

% % \ref{Ex_1_additional_traceplots}
% \begin{figure}
% \subfigure[$\theta^{(1)}$ of reSGLD]{\includegraphics[width=0.24\columnwidth]{Figures-Ex1-1Dmixture/trace_are_sgld_theta_1_ID=1_ETA1=0\lyxdot 00001_TAU1=1\lyxdot 0_ETA2=0\lyxdot 00001_TAU2=100000\lyxdot 0}}
% \subfigure[$\theta^{(2)}$ of reSGLD]{\includegraphics[width=0.24\columnwidth]{Figures-Ex1-1Dmixture/trace_are_sgld_theta_2_ID=1_ETA1=0\lyxdot 00001_TAU1=1\lyxdot 0_ETA2=0\lyxdot 00001_TAU2=100000\lyxdot 0}}
% \subfigure[$\theta^{(1)}$ of VR-reSGLD]{\includegraphics[width=0.24\columnwidth]{Figures-Ex1-1Dmixture/trace_svrg_are_sgld_theta_1_ID=1_ETA1=0\lyxdot 00001_TAU1=1\lyxdot 0_ETA2=0\lyxdot 00001_TAU2=500\lyxdot 0_RATECV=15_RATESIG2=17}}
% \subfigure[$\theta^{(2)}$ of VR-reSGLD]{\includegraphics[width=0.24\columnwidth]{Figures-Ex1-1Dmixture/trace_svrg_are_sgld_theta_2_ID=1_ETA1=0\lyxdot 00001_TAU1=1\lyxdot 0_ETA2=0\lyxdot 00001_TAU2=500\lyxdot 0_RATECV=15_RATESIG2=17}}\caption{Traceplots of reSGLD and VR-reSGLD\label{Ex_1_additional_traceplots}}
% \end{figure}

\newpage

\section{More Empirical Study on Image Classification}
\label{CV_more}

% \subsection{Hyperparameters}
% \label{hyper-par}
% For SGD, SGHMC, reSGHMC and VR-reSGHMC, we fix the learning rate $\eta_k=\text{2e-6}$ in the first 200 epochs and decay it by 0.984 every epoch afterwards. For SGHMC and the low-temperature processes of reSGHMC and VR-reSGHMC, we set an annealing temperature following $\tau_k^{(1)}={0.01}/{1.02^k}$ to conduct simulated annealing and accelerate the optimization; with respect to the high temperature process, we adopt a larger learning rate $\eta_k^{(2)}=1.5\eta_k^{(1)}$ and temperature $\tau_k^{(2)}=5\tau_k^{(1)}$. The correction factor is set to $F_k=F_0 1.02^k$ to counteract annealing temperatures and $F_0$ is fixed at $\text{1e5}$.

\subsection{Training cost}
\label{cost}

The batch size of $n=512$ almost doubles the training time and memory, which becomes too costly in larger experiments. A frequent update of control variates using $m=50$ is even more time-consuming and is not acceptable in practice. The choice of $m$ gives rise to a tradeoff between computational cost and variance reduction. As such, we choose $m=392$, which still obtains significant reductions of the variance at the cost of 40\% increase on the training time. Note that when we set $m=2000$, the training cost is only increased by 8\% while the variance reduction can be still at most 6 times on CIFAR10 and 10 times on CIFAR100.

\subsection{Adaptive coefficient}
\label{adaptive_c}
We study the correlation coefficient of the noise from the current parameter $\bbeta_k^{(h)}$, where $h\in\{1,2\}$, and the control variate $\bbeta^{(h)}_{m\lfloor \frac{k}{m}\rfloor}$. As shown in Fig.\ref{cifar_adaptive_c_figs}, the correlation coefficients are only around -0.5 due to the large learning rate in the early period. This implies that VR-reSGHMC may overuse the noise from the control variates and thus fails to fully exploit the potential in variance reduction. In spirit to the adaptive variance, we try the adaptive correlation coefficients to capture the pattern of the time-varying correlation coefficients and present it in Algorithm \ref{adaptive_alg}. 
\begin{algorithm}[tb]
  \caption{Adaptive variance-reduced replica exchange SGLD. The learning rate and temperature can be set to dynamic to speed up the computations. A larger smoothing factor $\gamma$ captures the trend better but becomes less robust.}
  \label{adaptive_alg}
\begin{algorithmic}
\footnotesize
% \STATE{\textbf{Input } Control variate $ \bbeta^{(h)}$, learning rate $\eta^{(h)}$ and temperature $\tau^{(h)}$ for $h\in\{1,2\}$, correction factor $F$.}
\STATE{\textbf{Input } Initial parameters $\bbeta_0^{(1)}$ and $\bbeta_0^{(2)}$, learning rate $\eta$ and temperatures $\tau^{(1)}$ and $\tau^{(2)}$, correction factor $F$.}
\REPEAT
  \STATE{\textbf{Parallel sampling} \text{Randomly pick a mini-batch set $B_{k}$ of size $n$.}}
  \begin{equation*}
      \bbeta^{(h)}_{k} =  \bbeta^{(h)}_{k-1}- \eta \frac{N}{n}\sum_{i\in B_{k}}\nabla  L( \bx_i|\bbeta^{(h)}_{k-1})+\sqrt{2\eta\tau^{(h)}} \bxi_{k}^{(h)}, \text{ for } h\in\{1,2\}.
  \end{equation*}
%   \vskip -0.5in
  \STATE{\textbf{Variance-reduced energy estimators} Update $\tiny{\widehat L^{(h)}=\sum_{i=1}^N L\left(\bx_i\Big| \bbeta^{(h)}_{m\lfloor \frac{k}{m}\rfloor}\right)}$ every $m$ iterations.}
  \begin{equation*}
  \small
%   \label{__vr_loss}
      \widetilde L(B_{k}|\bbeta_{k}^{(h)})=\frac{N}{n}\sum_{i\in B_{k}} L(\bx_i| \bbeta_{k}^{(h)}) + \widetilde {c}_k\cdot\left[ \frac{N}{n}\sum_{i\in B_{k}} L\left(\bx_i\Big| \bbeta^{(h)}_{m\lfloor \frac{k}{m}\rfloor}\right) -\widehat L^{(h)}\right] , \text{ for } h\in\{1,2\}. 
  \end{equation*}
  \IF{$k\ \text{mod}\ m=0$} 
%   \STATE{Compute sample variance $\tilde \sigma_k^2$ based on $\widetilde L(\widetilde B_{i}|\bbeta_{k}^{(1)})$ and $\widetilde L(\widetilde B_{i}|\bbeta_{k}^{(2)})$, where $i\in\{1,\cdots, \tilde m\}$.}
%   \STATE{Obtain an estimator $\tilde \sigma_k^2$ for $\Var(\widetilde L( B_k|\bbeta_{k}^{(1)})-\widetilde L( B_{k}|\bbeta_{k}^{(2)}))$.}
  \STATE{Update $\widetilde \sigma^2_{k} = (1-\gamma)\widetilde \sigma^2_{k-m}+\gamma \sigma^2_{k}$, where $\sigma_k^2$ is an estimate for $\Var\left(\widetilde L( B_k|\bbeta_{k}^{(1)})-\widetilde L( B_{k}|\bbeta_{k}^{(2)})\right)$.}
  \STATE{Update $\widetilde {c}_k=(1-\gamma)\widetilde {c}_{k-m}+\gamma c_k$, where $c_k$ is an estimate for $-\frac{\text{Cov}\big( L(B|\bbeta_k^{(h)}),  L\big(B|\bbeta_{m\lfloor \frac{k}{m}\rfloor}^{(h)}\big)\big)}{ \Var\big(L\big(B|\bbeta_{m\lfloor \frac{k}{m}\rfloor}^{(h)}\big)\big)}$.}
  \ENDIF
   
%   \STATE{Obtain an unbiased estimate $\tilde \sigma_{m+1}^2$ for $\sigma^2$.}
   
  \STATE{\textbf{Bias-reduced swaps} Swap $ \bbeta_{k+1}^{(1)}$ and $ \bbeta_{k+1}^{(2)}$ if $u<\widetilde S_{\eta,m,n}$, where $u\sim \text{Unif }[0,1]$, and $\widetilde S_{\eta,m,n}$ follows}
  \begin{equation*}
  \footnotesize
      \textstyle \widetilde S_{\eta,m,n}=\exp\left\{ \left(\frac{1}{\tau^{(1)}}-\frac{1}{\tau^{(2)}}\right)\left(  \widetilde L( B_{k+1}|\bbeta_{k+1}^{(1)})-  \widetilde L( B_{k+1}|\bbeta_{k+1}^{(2)})-\frac{1}{F}\left(\frac{1}{\tau^{(1)}}-\frac{1}{\tau^{(2)}}\right)\widetilde \sigma^2_{m\lfloor \frac{k}{m}\rfloor}\right)\right\}.
  \end{equation*}
%   \IF{$u<\widehat S$}
%   \STATE Swap $ \bbeta_{k+1}^{(1)}$ and $ \bbeta_{k+1}^{(2)}$.
%   \ENDIF
    
  \UNTIL{$k=k_{\max}$.}
    \vskip -1 in
\STATE{\textbf{Output:}  $\{\bbeta_{i\mathbb{T}}^{(1)}\}_{i=1}^{\lfloor k_{\max}/\mathbb{T}\rfloor}$, where $\mathbb{T}$ is the thinning factor.}
\end{algorithmic}
\end{algorithm}

As a result, we can further improve the performance of variance reduction by as much as 40\% on CIFAR10 and 30\% on CIFAR100 in the first 200 epochs. As the training continues and the learning rate decreases, the correlation coefficient is becoming closer to -1. In the late period, there is still 10\% improvement compared to the standard VR-reSGHMC.

In a nut shell, we can try adaptive coefficients in the early period when the absolute value of the correlation is lower than 0.5 or just use the vanilla replica exchange stochastic gradient Monte Carlo to avoid the computations of variance reduction.

\begin{figure*}[!ht]
  \centering
  \subfigure[\footnotesize{CIFAR10 \& m=50} ]{\includegraphics[width=3.2cm, height=2.6cm]{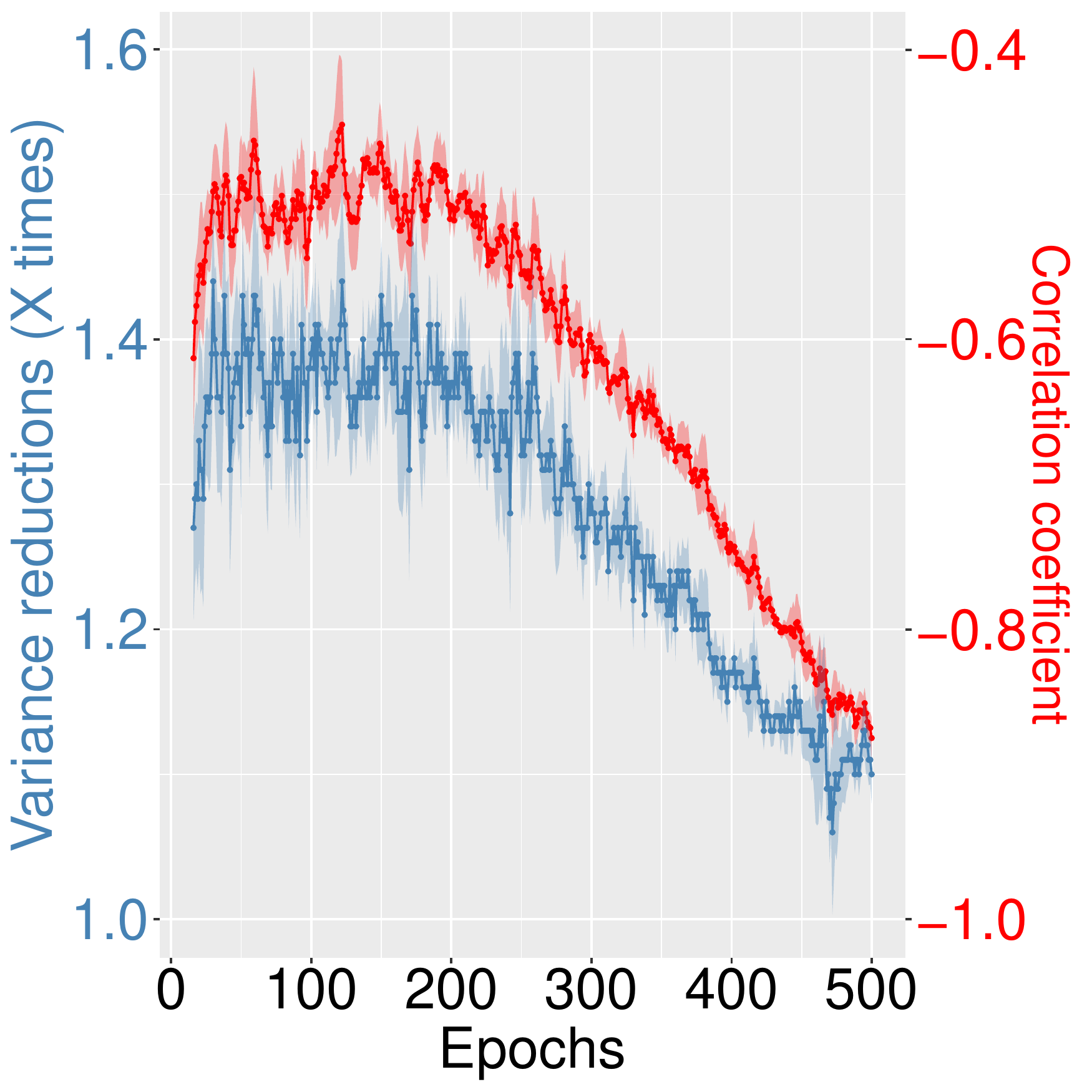}}\label{fig: 2a}\enskip
  \subfigure[\footnotesize{CIFAR100 \& m=50} ]{\includegraphics[width=3.2cm, height=2.6cm]{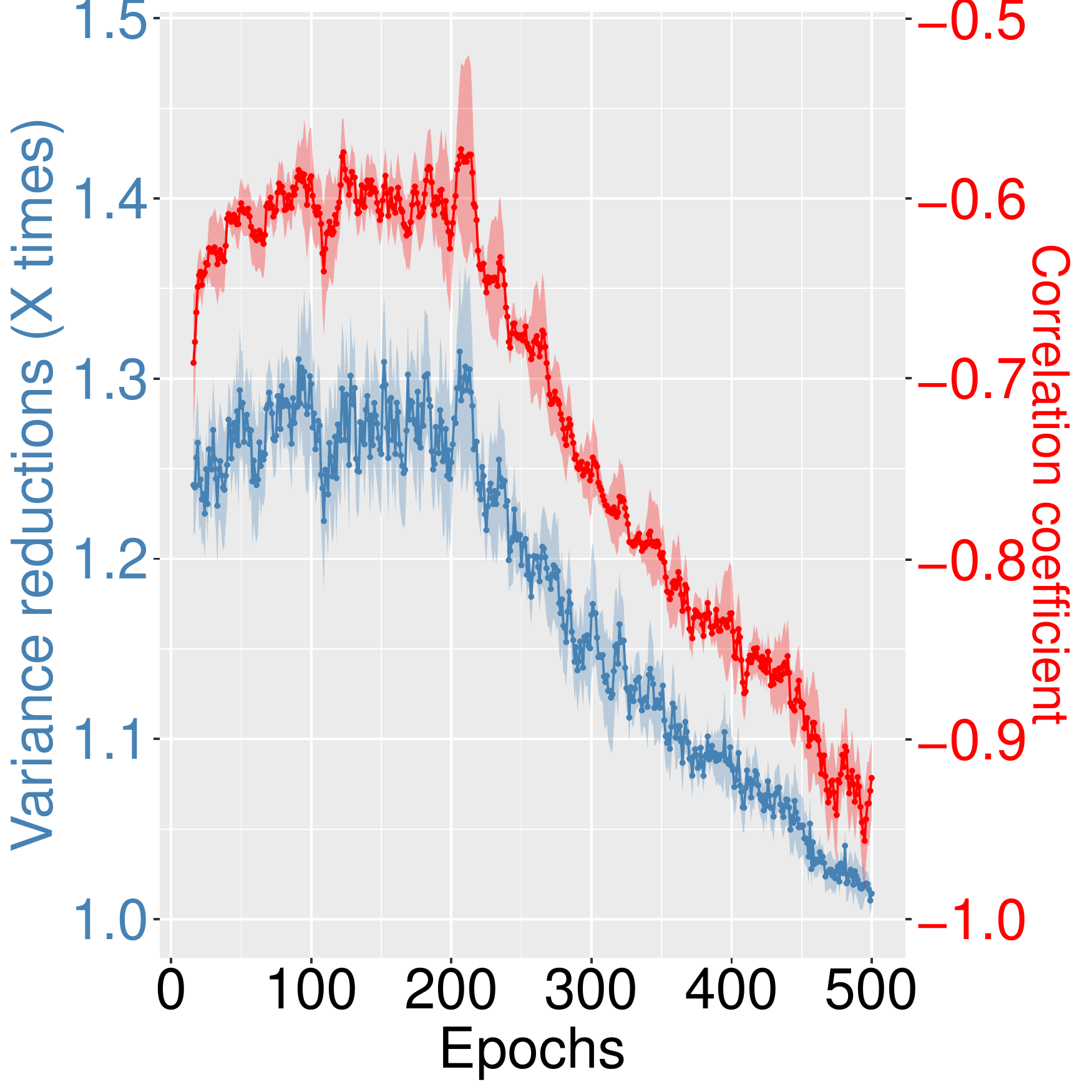}}\label{fig: 2b}\enskip
  \subfigure[CIFAR10 \& m=392]{\includegraphics[width=3.2cm, height=2.6cm]{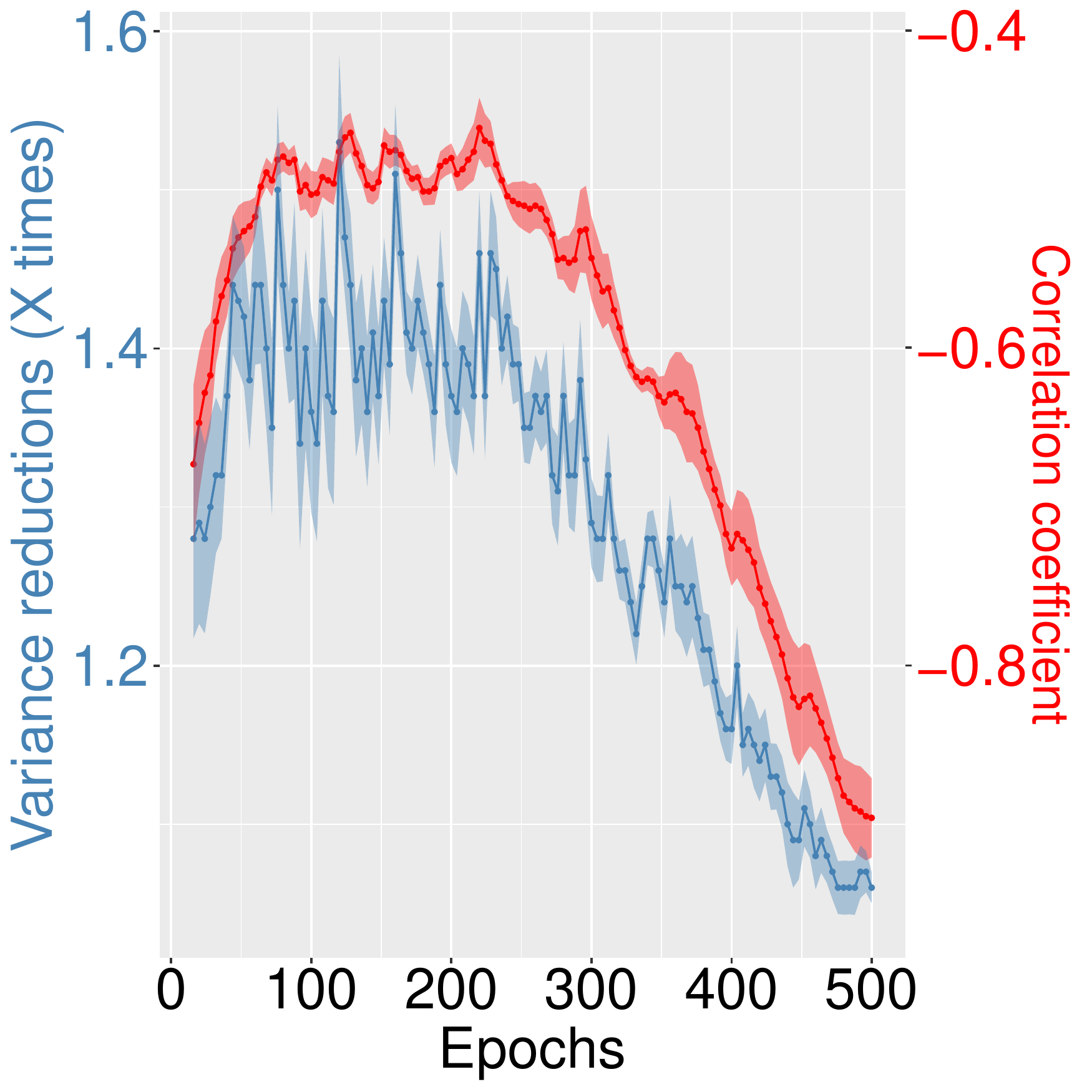}}\label{fig: 2c}\enskip
  \subfigure[CIFAR100 \& m=392]{\includegraphics[width=3.2cm, height=2.6cm]{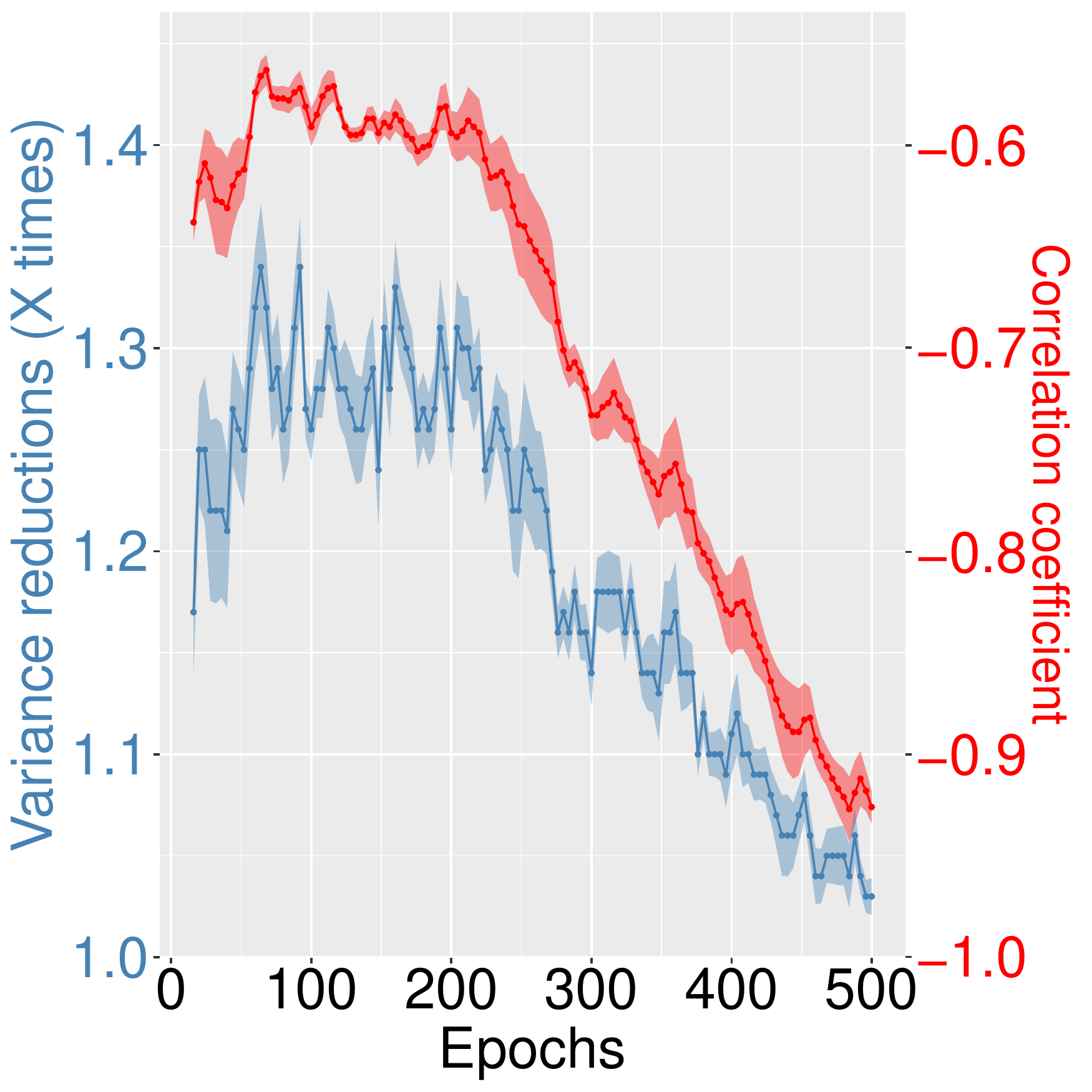}}\label{fig: 2d}
    \vskip -0.15in
  \caption{A study of variance reduction techniques using adaptive coefficient and non-adaptive coefficient on CIFAR10 \& CIFAR100 datasets.}
  \label{cifar_adaptive_c_figs}
%   \vspace{-1.3em}
\end{figure*}

\section{More Empirical Study on Uncertainty Quantification}
\label{UQ_more}

To avoid sacrificing the prediction power for the known classes, we also include the uncertainty estimate on CIFAR10 using the Brier score  (BS) \footnote{$\text{BS}=\frac{1}{N}\sum_{i=1}^N \sum_{j=1}^R (f_{ij}-o_{ij})^2$, where $f_i$ is the predictive probability and $o_i$ is actual output of the event which is 1 if it happens and 0 otherwise; $N$ is the number of instances and $R$ is the number of classes.} 
and compare it  with  the  estimates on  SVHN. The optimal BS scores on the seen CIFAR10 dataset and the unseen SVHN dataset are 0 and 0.1, respectively. As shown in Table.\ref{BS_score}, we see that the scores before calibration in the seen CIFAR10 is much lower than the ones in the unseen SVHN. This implies that all the models perform quite well in terms of what it knows, although cSGHMC are slightly better than the alternatives. To alleviate this issue, we propose to calibrate the predictive probability through the temperature scaling \citep{temperature_scaling} and obtain much better results.
Regarding the BS score on the unseen dataset, we see that M-SGD still performs the worst for frequently making over-confident predictions; SGHMC performs better but is far away from satisfying. reSGHMC obtains much better performance by allowing interactions between different chains. However, the large correction term affects the efficiency of the swaps significantly. In the end, our proposed algorithm increases the efficiency of the swaps via variance reduction and further improves the highly-optimized BS score based on reSGHMC from 0.29 to 0.27, which is much closer to the ideal 0.1. Note that the accurate uncertainty estimates of cVR-reSGHMC on the seen dataset is still maintained. Together with the lowest BS score in the unseen SVHN dataset, cVR-reSGHMC shows its strength in uncertainty quantification.

\begin{table*}[ht]
\begin{sc}
% \footnotesize
\caption[Table caption text]{Uncertainty estimates on SVHN using CIFAR10 models.}\label{BS_score}
\begin{center} 
\begin{tabular}{c|cc|cc}
\hline
\multirow{2}{*}{Method} & \multicolumn{2}{c|}{Brier Score (\upshape{before calibration})} & \multicolumn{2}{c}{Brier Score (\upshape{after calibration})}  \\
\cline{2-5}
 & \footnotesize{\upshape{CIFAR10 (seen)}} & \footnotesize{\upshape{SVHN (unseen)}} & \footnotesize{\upshape{CIFAR10 (seen)}} & \footnotesize{\upshape{SVHN (unseen)}} \\
\hline
\hline
M-SGD &   0.090$\pm$0.001   &    0.48$\pm$0.02 &   0.098$\pm$0.001   &    0.33$\pm$0.02    \\
SGHMC &  0.089$\pm$0.001  &    0.47$\pm$0.02  &  0.099$\pm$0.001  &    0.31$\pm$0.02  \\ 
\hline
\upshape{re}SGHMC & 0.086$\pm$0.002  &  0.41$\pm$0.03 & 0.097$\pm$0.001  &  0.29$\pm$0.02   \\ 
% \hline
\footnotesize{\upshape{c}SGHMC} &  0.084$\pm$0.001   & 0.43$\pm$0.02 &  0.092$\pm$0.001   & 0.30$\pm$0.02   \\
\footnotesize{\upshape{c}VR-\upshape{re}SGHMC} &   0.085$\pm$0.001  &   {0.38$\pm$0.02}  &   0.094$\pm$0.001  &   {0.27$\pm$0.02}  \\
\hline
\end{tabular}
\end{center} 
\end{sc}
\vspace{-1.5em}
\end{table*}

\section{Modified Example \ref{subsec:Toy-multi-mo}}
\label{subsec:Toy-multi-mo-modif}

We revisit Example \ref{subsec:Toy-multi-mo}, and re-run the procedures %under consideration targeting the same posterior of interest but 
with temperature
$\tau^{(1)}=1.0$. In Fig. \ref{fig:Trace-plots-and}, we present
trace plots and kernel density estimates (KDE) of samples generated
from VR-reSGLD, reSGLD, and SGLD. In particular, we run VR-reSGLD
with $m=40$, $\tau^{(1)}=1$, $\tau^{(2)}=500$, $\eta=1e-5$, and
$F=1$; reSGLD with the same hyper-parameters as VR-reSGLD except
for $F=500$; and SGLD with $\eta=1e-5$ and $\tau=1$. Note that
here, we run reSGLD with a greater $F$ than in Example \ref{subsec:Toy-multi-mo} in order to 
prohibit the drastic reduction of the swapping rate which is caused by the pickier target density. As in Example \ref{subsec:Toy-multi-mo}, for
the ground truth, we run replica exchange Langevin dynamics with long
enough iterations. In Figs \ref{fig:Trace-plot-forwetwert} and \ref{fig:Trace-plot-forsdfgdaf},
we observe that, even though the distribution of interest has a
pickier density, our proposed algorithm VR-reSGLD was able to detect both modes and acceptably jump between them. On the other hand, the competitor algorithm
SGLD was trapped in the first mode visited and never escaped. reSGLD
was able to jump some times between modes only after considering a
substantial factor $F=500$ which, according to the theory, introduces
bias. 

\begin{figure}
\center
\subfigure[Trace plot for VR-reSGLD and ground truth\label{fig:Trace-plot-forwetwert}]{\includegraphics[width=0.35\columnwidth]{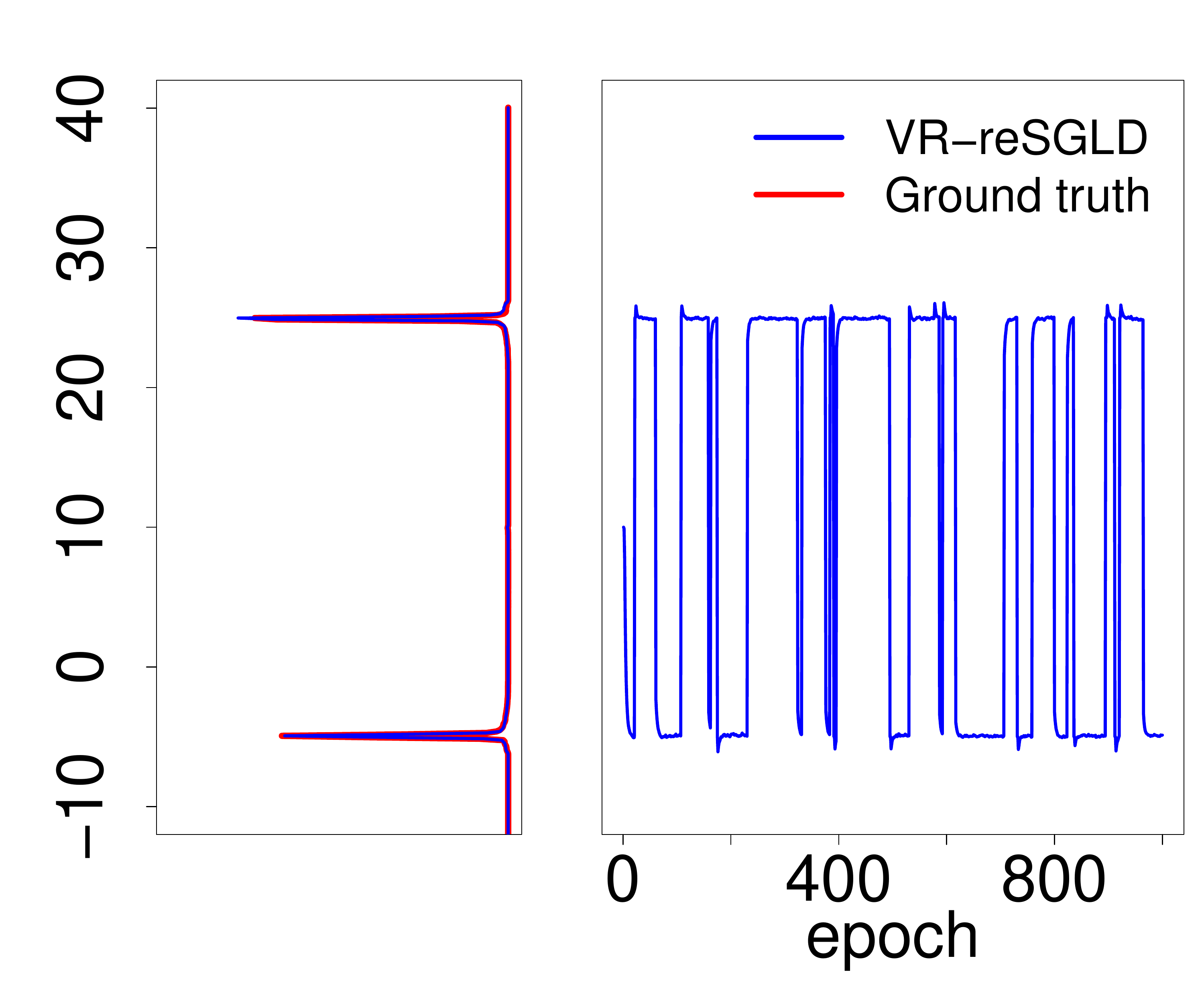}
}
\qquad{}
\subfigure[Trace plot for SGLD, reSGLD and ground truth\label{fig:Trace-plot-forsdfgdaf}]{\includegraphics[width=0.35\columnwidth]{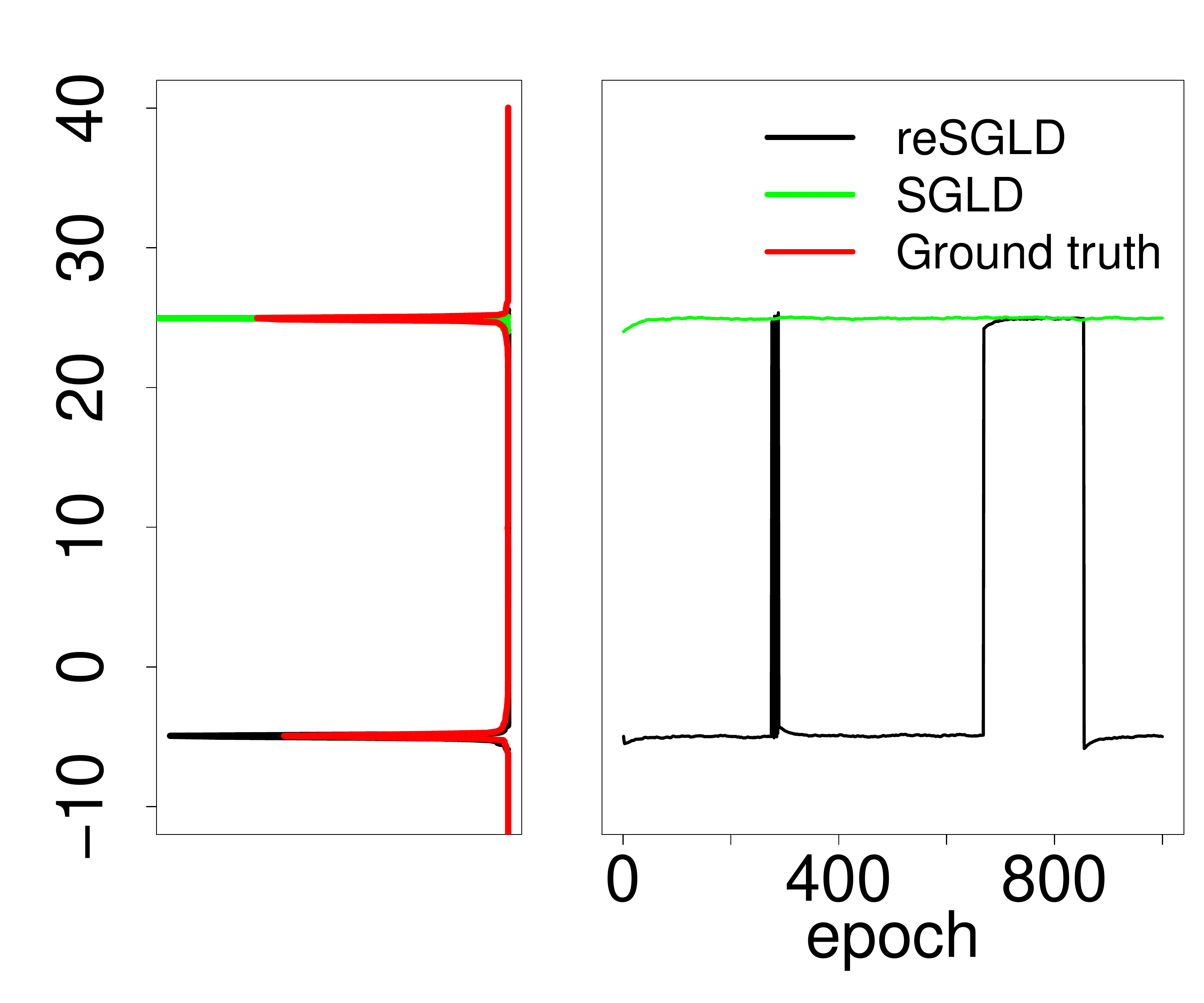}}\caption{Trace plots and KDEs of $\beta^{(1)}$\label{fig:Trace-plots-and}}

\end{figure}

% \bibliography{iclr2021_conference}
% \bibliographystyle{iclr2021_conference}

% \appendix
% \section{Appendix}

% \end{document}

\end{document}